\documentclass{article}

\usepackage{times,color}
\usepackage{microtype}
\usepackage{graphicx} 
\usepackage{subfigure} 
\usepackage{booktabs} 

\usepackage{amsmath}
\usepackage{amssymb}
\usepackage{amsthm}
\usepackage[acronym]{glossaries}
\usepackage{multirow}

\usepackage{natbib}

\usepackage{algorithm}
\usepackage{algorithmic}

\usepackage{grffile}
\usepackage{romannum}

\usepackage[]{hyperref}

\usepackage[accepted]{icml2019}

\icmltitlerunning{Revisiting the Softmax Bellman Operator:
New Benefits and New Perspective}

\newcommand{\x}[1]{\ensuremath{\mathbf{x}}}

\newcommand{\eqnref}[1]{Eq.~(\ref{#1})}

\newcommand{\thmref}[1]{Theorem~\ref{#1}}

\newcommand{\bx}{{\bf x}}
\newcommand{\btheta}{{\boldsymbol{\theta}}}

\newtheorem{theorem}{Theorem}
\newtheorem{definition}[theorem]{Definition}
\newtheorem{lemma}[theorem]{Lemma}

\newtheorem{corollary}[theorem]{Corollary}

\def\cA{{ \cal A }}
\def\cS{{ \cal S }}

\def\cT{{ \cal T }}
\def\cL{{ \cal L }}

\def\cTs{ {\cal T}_{\text{soft}} }

\newacronym{MDP}{\textsc{MDP}}{Markov decision process}
\newacronym{LSPI}{\textsc{lspi}}{least-squares policy iteration}
\newacronym{CNN}{\textsc{CNN}}{convolutional neural network}
\newacronym{DQN}{\textsc{DQN}}{deep Q-network}
\newacronym{DDQN}{\textsc{DDQN}}{double DQN}
\newacronym{RL}{\textsc{RL}}{reinforcement learning}
\newacronym{DL}{\textsc{DL}}{deep learning}
\newacronym{AI}{\textsc{AI}}{artificial intelligence}
\newacronym{SGD}{\textsc{SGD}}{stochastic gradient descent}
\newacronym{PMF}{\textsc{PMF}}{probability mass function}
\newacronym{ALE}{\textsc{ALE}}{Arcade Learning Environment}
\newacronym{KL}{\textsc{KL}}{Kullback-Leibler}

\begin{document} 

\twocolumn[
\icmltitle{Revisiting the Softmax Bellman Operator:\\
New Benefits and New Perspective}



\icmlsetsymbol{equal}{*}
\icmlsetsymbol{work}{*}

\begin{icmlauthorlist}
\icmlauthor{Zhao Song}{duke,work}
\icmlauthor{Ronald E. Parr}{duke}
\icmlauthor{Lawrence Carin}{duke}
\end{icmlauthorlist}

\icmlaffiliation{duke}{Duke University $^{*}$Work performed as a graduate student at Duke University; now at Baidu Research}

\icmlcorrespondingauthor{Zhao Song}{zsong@alumni.duke.edu}
\icmlcorrespondingauthor{Ronald E. Parr}{parr@cs.duke.edu}
\icmlcorrespondingauthor{Lawrence Carin}{lcarin@duke.edu}

\icmlkeywords{Machine Learning, ICML}

\vskip 0.3in

]



\printAffiliationsAndNotice{}  

\begin{abstract}
	The impact of softmax on the value function itself in reinforcement learning (RL) is often viewed as problematic because it leads to sub-optimal value (or Q) functions and interferes with the contraction properties of the Bellman operator. Surprisingly, despite these concerns, and {\em independent of its effect on exploration}, the softmax Bellman operator when combined with Deep Q-learning, leads to Q-functions with superior policies in practice, even outperforming its double Q-learning counterpart. 
	To better understand how and why this occurs, we revisit theoretical properties of the softmax Bellman operator, and prove that $(i)$ it converges to the standard Bellman operator exponentially fast in the inverse temperature parameter, and $(ii)$ the distance of its Q function from the optimal one can be bounded. These alone do not explain its superior performance, so we also show that the softmax operator can reduce the overestimation error, which may give some insight into why a sub-optimal operator leads to better performance in the presence of value function approximation. A comparison among different Bellman operators is then presented, showing the trade-offs when selecting them. 
\end{abstract}
\vspace{-0.25in}

\section{Introduction}
\label{sec:intro}

\glsreset{DQN}
\glsreset{DDQN}
\glsreset{RL}

The Bellman equation~\citep{bellman57} has been a fundamental tool in~\gls{RL}, as it provides a sufficient condition for the optimal policy in dynamic
programming. The use of the max function in the Bellman equation further suggests that the optimal policy should be greedy w.r.t. the Q-values. On the other hand, 
the trade-off between exploration and exploitation~\citep{Thrun92} motivates the use of exploratory and potentially sub-optimal actions during learning, and one commonly-used strategy is to add randomness by 
replacing the max function with the softmax function, as in Boltzmann exploration~\citep{sutton98}. Furthermore, the softmax function is a differentiable
approximation to the max function, and hence can facilitate analysis~\citep{ReverdyLeonard}.
\vspace{-0.02in}

The beneficial properties of the softmax Bellman operator are in contrast to its potentially negative effect on the accuracy of the resulting value or Q-functions.  For example, it has been demonstrated that the softmax Bellman operator is not a contraction, for certain temperature parameters~\citep[Page 205]{littman}. Given this, one might expect that the convenient properties of the softmax Bellman operator would come at the expense of the accuracy of the resulting value or Q-functions, or the quality of the resulting policies. In this paper, we demonstrate that, in the case of deep Q-learning, this expectation is surprisingly incorrect.  We combine the softmax Bellman operator with the~\gls{DQN}~\citep{Mnih+al:2015} and~\gls{DDQN}~\citep{HasseltGuezSilver} algorithms, by replacing the max function
therein with the softmax function, in the target network. We then test the variants on several games in the \gls{ALE}~\citep{BellemareNaddafVenessBowling}, a standard large-scale deep~\gls{RL} testbed. The results show that the variants using the softmax Bellman operator can achieve higher
test scores, and reduce the Q-value overestimation as well as the gradient noise
on most of them. {\em This effect is independent of exploration and is entirely attributable to the change in the Bellman operator.}
\vspace{-0.02in}

This surprising result suggests that a deeper understanding of the softmax Bellman operator is warranted. To this end, we prove that starting from the same initial Q-values, we can upper and lower bound how far the Q-functions computed with the softmax operator can deviate from those computed with the regular Bellman operator.  
We further show that the softmax Bellman operator converges to the optimal Bellman operator in an exponential rate w.r.t. the inverse temperature parameter. This gives insight into why the negative convergence results may not be as discouraging as they initially seem, but it does not explain the superior performance observed in practice. Motivated by  recent work~\citep{HasseltGuezSilver,AnschelBaramShimkin} targeting bias and instability of the original~\gls{DQN}~\citep{Mnih+al:2015}, we further investigate whether the softmax Bellman operator can alleviate these issues. 
As discussed in~\citet{HasseltGuezSilver}, one possible explanation for the poor performance of
the vanilla~\gls{DQN} on some Atari games was the overestimation bias when computing 
the target network, due to the max operator therein. We prove that given the same assumptions as~\citet{HasseltGuezSilver}, the softmax Bellman operator can reduce the overestimation bias, for any inverse temperature parameters. We also quantify the overestimation reduction by providing its lower and upper bounds.

Our results are complementary to and add new motivations for existing work that explores various ways of softening the Bellman operator. For example, entropy regularizers have been used to smooth policies. The motivations for such approaches include computational convenience, exploration, or robustness~\citep{fox2016taming,haarnoja2017reinforcement,SchulmanAbbeelChen,neu2017unified}. With respect to value functions, ~\citet{AsadiLittman} proposed an alternative mellowmax
operator and proved that it is a contraction. Their experimental results  suggested that it can improve  exploration, but the possibility that the sub-optimal Bellman operator could, independent of exploration, lead to superior policies was not considered. Our results, therefore, provide additional motivation for further study of operators such as mellowmax. Although very similar to softmax, the mellowmax operator needs some extra computation to represent a policy, as noted in~\citet{AsadiLittman}. Our paper discusses this and other trade-offs among different Bellman operators.


The rest of this paper is organized as follows: We provide the necessary background and notation in Section~\ref{sec:background}. The softmax Bellman operator is introduced in Section~\ref{sec:softmax}, where its convergence properties and performance bound are provided. Despite being sub-optimal on Q-functions, the softmax operator is shown in Section~\ref{sec:experiments} to consistently outperform its max counterpart on several Atari games. Such surprising result further motivates us to investigate why this happens in Section~\ref{sec:why}. A thorough comparison among different Bellman operators is presented in Section~\ref{sec:comparison}. Section~\ref{sec:related} discusses the related work and Section~\ref{sec:conclusion} concludes this paper.

\section{Background and Notation}
\label{sec:background}

\glsreset{MDP}
\glsreset{DDQN}

A~\gls{MDP} can be represented as a 5-tuple $\langle \cS, \cA, P, R, \gamma \rangle$, where 
$\cS$ is the state space, $\cA$ is the action space, $P$ is the 
transition kernel whose element $P (s' | s, a)$ denotes the transition probability
from state $s$ to state $s'$ under action $a$, $R$ is a reward function whose element
$R(s, a)$
denotes the expected reward for executing action $a$ in state $s$, and 
$\gamma \in (0, 1)$ is the discount factor. The policy $\pi$ in an~\gls{MDP} can be represented
in terms of a~\gls{PMF}, where $\pi (s, a) \in [0, 1]$ denotes the probability
of selecting action $a$ in state $s$, and $\sum_{a \in \cA} \pi (s, a) = 1$.

For a given policy $\pi$, its state-action value function $Q^{\pi} (s, a)$ is defined as
the accumulated, expected, discount reward, when taking action $a$ in state $s$, and
following policy $\pi$ afterwards, i.e., 
$Q^{\pi} (s, a) = \mathbb{E}_{a_t \sim \pi} \left[ \sum_{t=0}^{\infty} 
\gamma^t r_t | s_0=s, a_0=a \right].$
For the optimal policy $\pi^{*}$, its corresponding Q-function satisfies the following
Bellman equation:
\begin{equation*}
	Q^{*} (s, a) = R(s, a) + \gamma \sum_{s'} P (s' | s, a) \, \max_{a'} Q^{*} (s', a'). 
	\vspace{-0.1in}
	\label{eq:bellman}
\end{equation*}

In~\gls{DQN}~\citep{Mnih+al:2015}, the Q-function is parameterized with a neural network as $Q_{\btheta} (s, a)$, 
which takes the state $s$ as input and outputs the corresponding Q-value in the final
fully-connected linear layer, for every action $a$. The training objective for the~\gls{DQN} can be 
represented as 
\begin{equation}
	\vspace{-0.0in}
	\min_{\btheta} \frac{1}{2} \left\| Q_{\btheta} (s, a) - [ R (s, a) + \gamma \max_{a'} 
	Q_{\btheta^{-}} (s', a') ] \right\|^2,
	\label{eq:dqn_obj}
\end{equation}
where $\btheta^{-}$ corresponds to the frozen weights in the target network, and is updated at fixed intervals. 
The optimization of~\eqnref{eq:dqn_obj} is performed via RMSProp~\citep{TielemanHinton}, 
with mini-batches sampled from a replay buffer.

To reduce the overestimation bias, the~\gls{DDQN} algorithm modified the target that 
$Q_{\btheta} (s, a)$ aims to fit in~\eqnref{eq:dqn_obj} as
\begin{equation*}
	R (s, a) + \gamma \, Q_{\btheta^{-}} \big(s', \arg \max_a Q_{\btheta_t} (s', a) \big).
	\label{eq:ddqn_obj}
\end{equation*}
Note that a separate network based on the latest estimate $\btheta_t$ is employed for action selection,
and the evaluation of this policy is due to the frozen network.

\paragraph{Notation} The softmax function is defined as
\begin{equation*}
	f_{\tau} (\bx) = \frac{ [\exp(\tau x_1), \exp(\tau x_2), \ldots, \exp(\tau x_m)]^T }
	{ \sum_{i=1}^m \exp (\tau x_i) },
\end{equation*}
where the superscript $T$ denotes the vector transpose. Subsequently, the softmax-weighted
function is represented as $g_{\bx} (\tau) = f_{\tau}^T (\bx) \, \bx$, as a function of $\tau$.
Also, we define the vector $Q(s, ) = [Q(s, a_1), Q(s, a_2), \ldots, Q(s, a_m)]^T$.
We further set $m$ to be the size of the action set $\cA$ in the~\gls{MDP}. Finally,  $R_{\text{min}}$ and  $R_{\text{max}}$
denote the minimum and the maximum immediate rewards, respectively.

\section{The Softmax Bellman Operator}
\label{sec:softmax}

We start by providing the following standard Bellman operator:
\begin{equation}
	\cT \, Q(s, a) = R(s, a) + \gamma \sum_{s'} P (s' | s, a) \, \max_{a'} 
	Q (s', a').
	\label{eq:bellman4q}
\end{equation}

We propose to use the softmax Bellman operator, defined as
\begin{align}
	\cTs \, Q(s, a) = & R(s, a) + \gamma \sum_{s'} P (s' | s, a)
	\notag
	\\
	& \times \underbrace{ \sum_{a'} \frac{\exp [ \tau \, Q (s', a')]}{\sum_{\bar{a}} 
	\exp [ \tau \, Q (s', \bar{a} )]} \, Q (s', a') }_{sm_{\tau} ( Q(s', ) ) },
	\label{eq:soft}
\vspace{-0.15in}
\end{align}
where $\tau \geq 0$ denotes the inverse temperature parameter.
Note that $\cTs$ will reduce to $\cT$ when $\tau \rightarrow \infty$. 

The mellowmax operator was introduced in~\citet{AsadiLittman} as an alternative to the softmax operator in~\eqnref{eq:soft}, defined as
\begin{equation}
\vspace{-0.05in}
    mm_{\omega} \big( Q(s, ) \big) = \frac{\log \big(\frac{1}{m} \sum_{a'} \exp[\omega \, Q(s, a')] \big) }{\omega},
    \label{eq:mellow}
\end{equation}
where $\omega > 0$ is a tuning parameter. Similar to the softmax operator $sm_{\tau} \big(Q (s', ) \big)$ in~\eqnref{eq:soft}, $mm_{\omega}$ converges to the mean and max operators, when $\omega$ approaches $0$ and $\infty$, respectively. Unlike the softmax operator, however, the mellowmax operator in~\eqnref{eq:mellow} does not directly provide a probability distribution over actions, and thus additional steps are needed to obtain a corresponding policy~\cite{AsadiLittman}.  

In contrast to its typical use to improve exploration (also a purpose for the mellowmax operator in~\citet{AsadiLittman}), the softmax Bellman
operator in~\eqnref{eq:soft} is combined in this paper with the~\gls{DQN} and~\gls{DDQN} algorithms in
an off-policy fashion, where the action is selected according to the same $\epsilon$-greedy policy as in~\citet{Mnih+al:2015}.

Even though the softmax Bellman operator was shown in~\citet{littman} not to be a contraction, with a counterexample,
we are not aware of any published work showing the performance bound by using the softmax Bellman operator
in Q-iteration.
Furthermore, our~\emph{surprising} results in Section~\ref{sec:experiments} show that test scores in~\glspl{DQN} and~\glspl{DDQN} can be improved, by solely replacing the max operator in their target networks with the softmax operator. Such improvements are independent of the exploration strategy and hence motivate an investigation of the theoretical properties of the softmax Bellman operator.


Before presenting our main theoretical results, we first show how to bound the distance between the softmax-weighted and max operators, which is subsequently used in the proof for the performance bound and overestimation reduction.

\begin{definition}
$\widehat{\delta} (s) \triangleq  \sup_{Q} \max_{i, j} |Q (s, a_i) - Q(s, a_j)|$ denotes the largest distance between Q-functions w.r.t. actions for state $s$, by taking the supremum of Q-functions and the maximum of all action pairs.
\end{definition}

\begin{lemma}
By assuming $\widehat{\delta} (s) > 0$\footnote{Note that if $\widehat{\delta} (s) = 0$, then all actions are equivalent in state $s$, and softmax = max.}, we have $\forall Q$, $ \frac{\widehat{\delta} (s)}{m \exp [\tau \,\widehat{\delta} (s) ] } \leq  \max_a Q (s, a) - f_{\tau}^T \big( Q(s, ) \big) \, Q(s, ) \leq (m-1) \max \Big\{ \frac{1}{\tau+2}, \frac{ 2 Q_{\text{max}} }{1 + \exp(\tau) } \Big\}$~\footnote{The $\sup_{Q}$ is over all $Q$-functions that occur during $Q$-iteration, starting from $Q_0$ until the iteration terminates.}, where $Q_{\text{max}} = \frac{ R_{\text{max} } }{1 - \gamma}$ represents the maximum $Q$-value in Q-iteration with $\cTs$.\footnote{The Supplemental Material formally bounds $Q_{\text{max}}$.}
\label{lemma:diff_bound}
\vspace{-0.1in}
\end{lemma}

Note that another upper bound for this gap was provided in~\citet{DonoghueEtAl}. Our proof here uses a different strategy, by considering possible values for the difference between Q-values with different actions. We further derive a lower bound for this gap. \vspace{-0.05in}

\subsection{Performance Bound for $\cTs$}
\label{subsec:bound}

The optimal state-action value function $Q^{*}$ is known to be a fixed point for the  
standard Bellman operator $\cT$~\citep{WilliamsBaird}, i.e., $\cT Q^{*} = Q^{*}$. 
Since $\cT$ is a contraction with rate $\gamma$, we also know that 
$\lim_{k \rightarrow \infty} \cT^{k} \, Q_0  = Q^{*} $, for arbitrary $Q_0$.
Given these facts, one may wonder in the limit, how far iteratively applying
$\cTs$, in lieu of $\cT$, over $Q_0$ will be away from $Q^{*}$, as a function of $\tau$.

\begin{theorem}
	Let $\cT^{k} Q_0$ and $\cTs^{k} Q_0$ denote that the operators $\cT$ and 
	$\cTs$ are iteratively applied over an initial state-action value function 
	$Q_0$ for $k$ times. Then,

	$(I) \, \forall (s, a), \, \limsup \limits_{k \rightarrow \infty}  \cTs^{k} \, Q_0 (s, a) \,\leq\, Q^{*} (s, a)$ and
    \vspace{-0.1in}
	\begin{align*}
	    & \liminf \limits_{k \rightarrow \infty}  \cTs^{k} \, Q_0 (s, a) \, \geq \,
	    \\
		&  Q^{*} (s, a) \, - \frac{\gamma ( m-1 )}{(1 - \gamma) }
		  \max \Big\{ \frac{1}{\tau+2}, \frac{ 2  Q_{\text{max}} }{
		  1 + \exp(\tau) } \Big\}.
	\vspace{-0.02in}
	\end{align*}
	$(II)$ $\cTs$ converges to $\cT$ with an exponential rate, in terms of $\tau$, i.e., the upper bound of $\cT^{k} Q_0 - \cTs^{k} Q_0$ decays exponentially fast, as a function of $\tau$,  the proof of which does not depend on the bound in part $(I)$.
	\label{thm:convergebound}
\vspace{-0.1in}
\end{theorem}

A noteworthy point about part $(I)$ of Theorem~\ref{thm:convergebound} is it does not contradict the negative convergence results for $\cTs$ since its proof and result do not need to assume the convergence of $\cTs$. Furthermore, this bound implies that non-convergence for $\cTs$ is different from divergence or oscillation with an arbitrary range: Even though $\cTs$ may not converge in certain scenarios, the Q-values could still be within a reasonable range~\footnote{ The lower bound can become loose when $\tau$ is extremely small, but this degenerate case has never been used in experiments.}. $\cTs$ also has other benefits over $\cT$, as shown in Section~\ref{sec:why}. 


\begin{figure*}[!t]
	\vspace{-0.1in}
	\centering
	\subfigure{\includegraphics[width=0.17\linewidth]{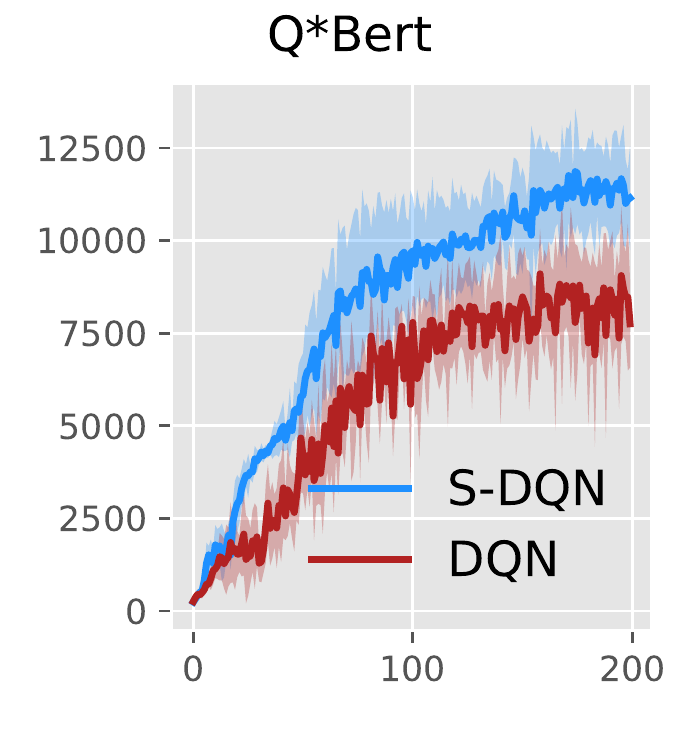}} 
    \hspace{-0.1in}
	\subfigure{\includegraphics[width=0.17\linewidth]{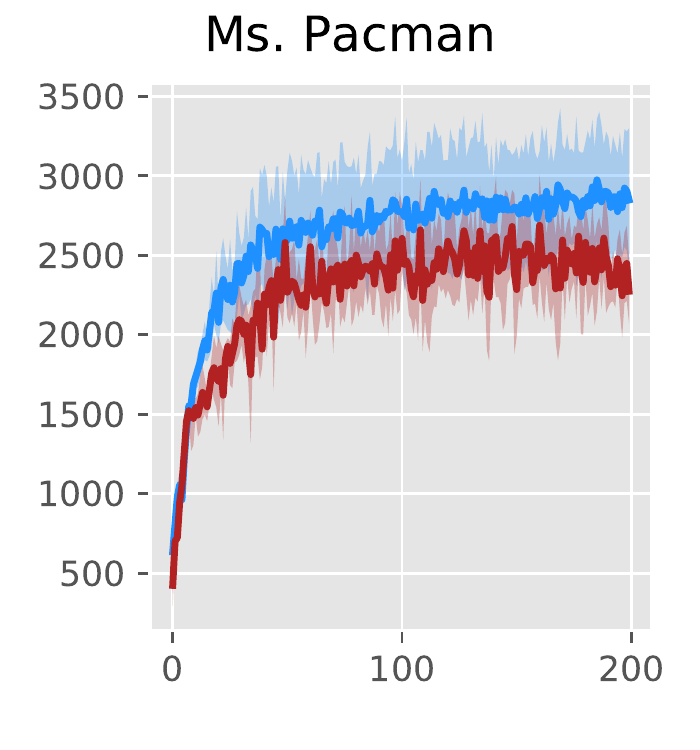}}
	\hspace{-0.1in}
	\subfigure{\includegraphics[width=0.17\linewidth]{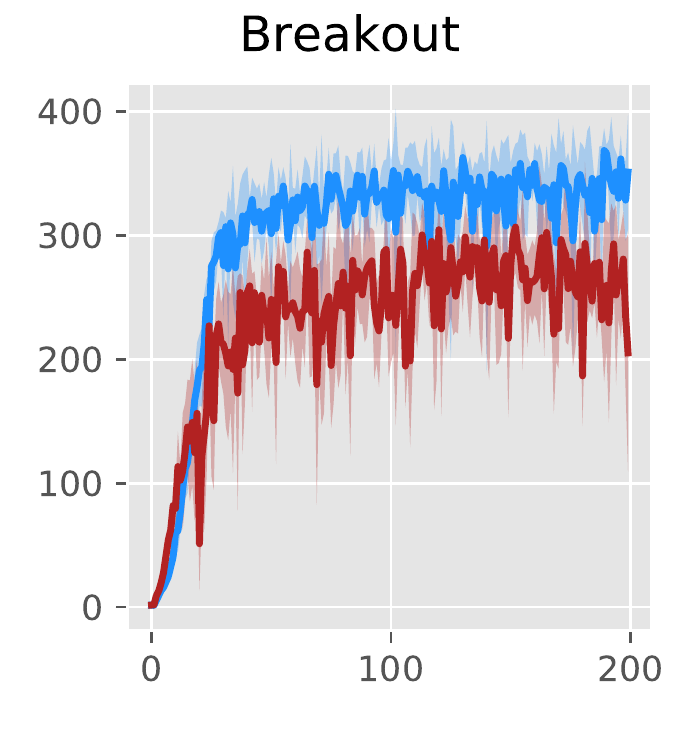}} 
	\hspace{-0.1in}
	\subfigure{\includegraphics[width=0.17\linewidth]{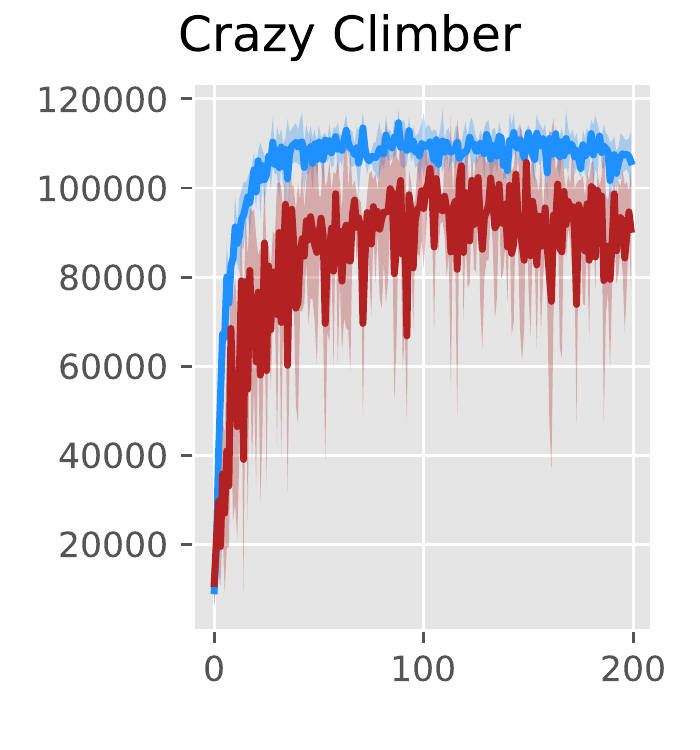}} 
	\hspace{-0.1in}
	\subfigure{\includegraphics[width=0.17\linewidth]{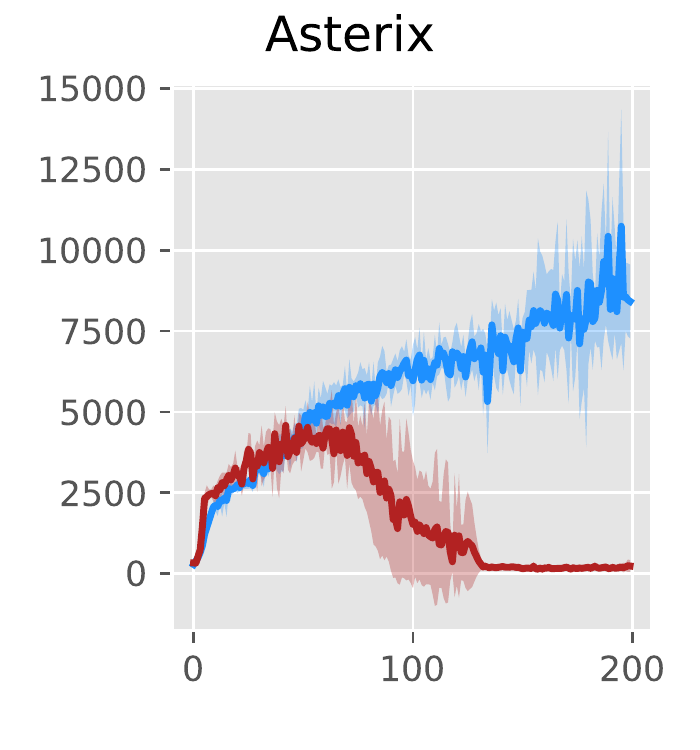}} 
	\hspace{-0.1in}
	\subfigure{\includegraphics[width=0.17\linewidth]{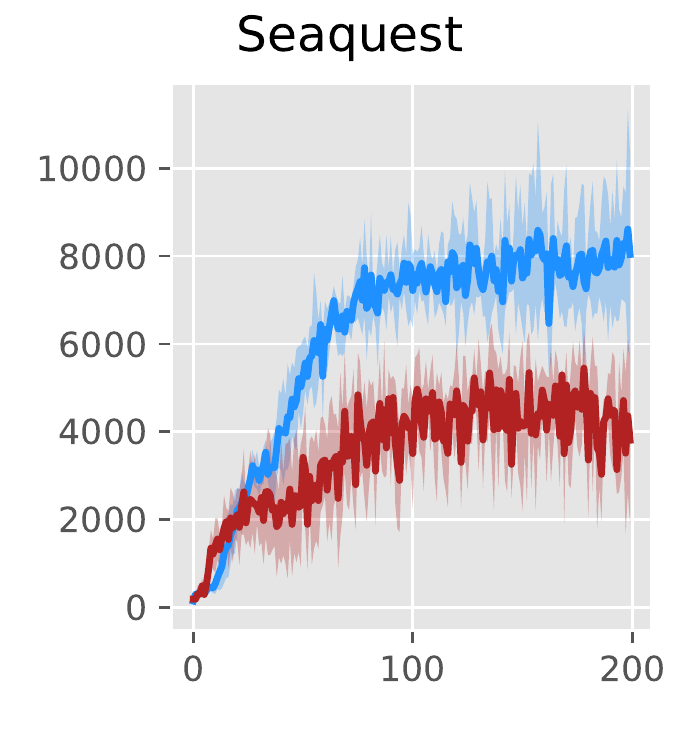}}
	
	\vspace{-0.25in}
	\subfigure{\includegraphics[width=0.17\linewidth]{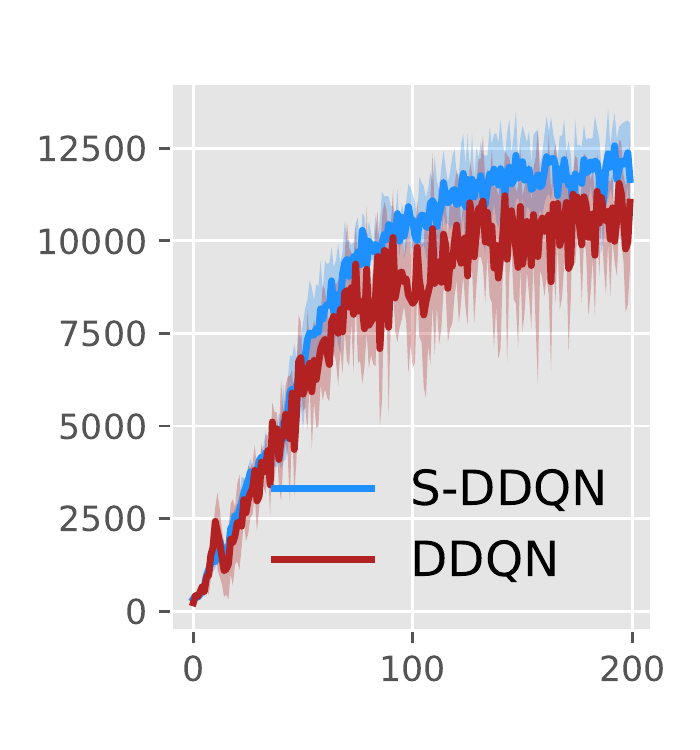}} 
    \hspace{-0.1in}
	\subfigure{\includegraphics[width=0.17\linewidth]{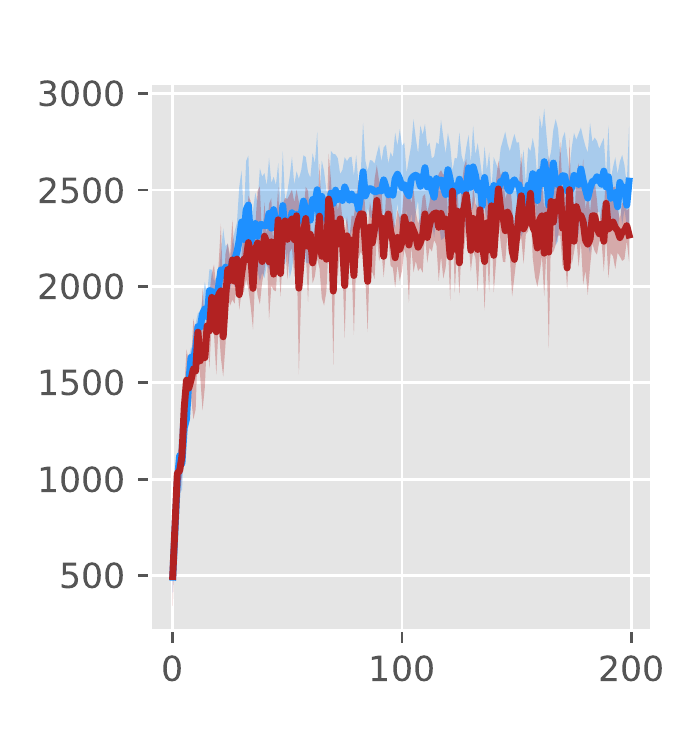}}
	\hspace{-0.1in}
	\subfigure{\includegraphics[width=0.17\linewidth]{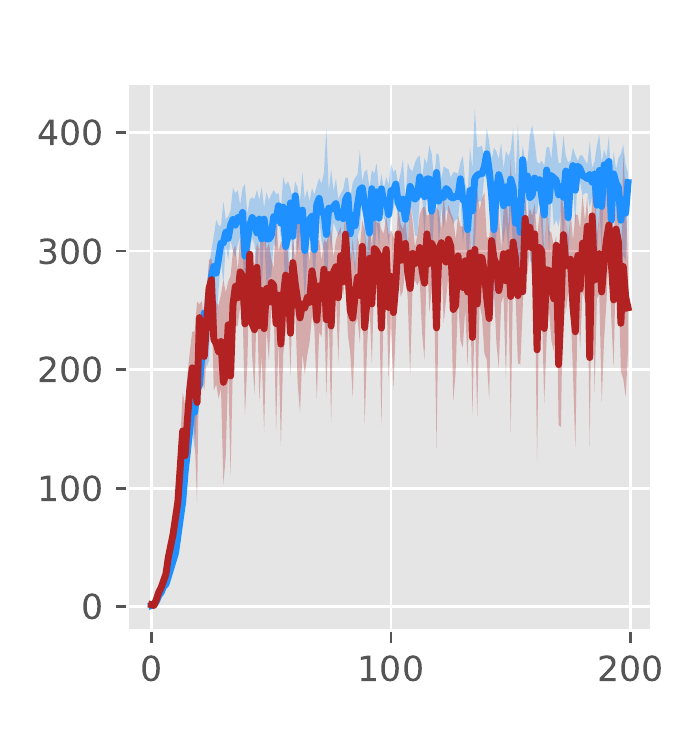}} 
	\hspace{-0.1in}
	\subfigure{\includegraphics[width=0.17\linewidth]{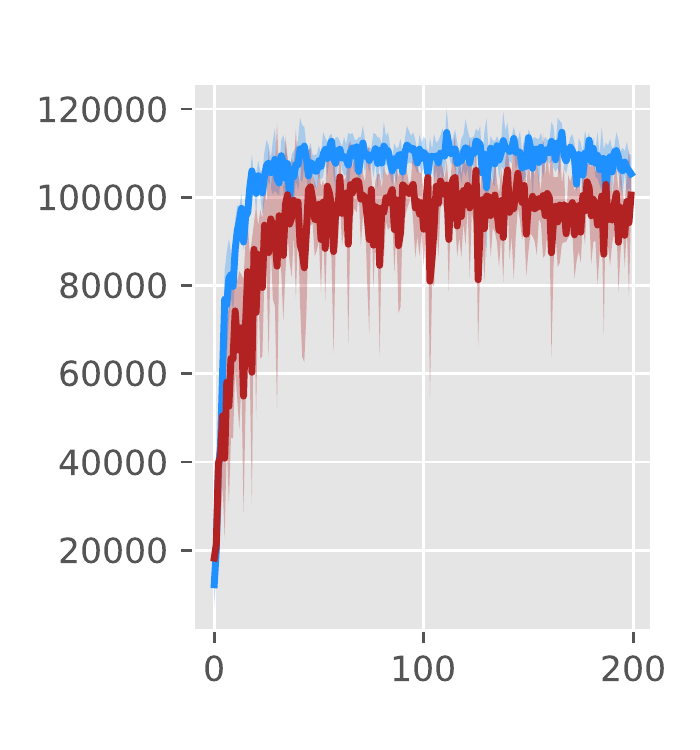}} 
	\hspace{-0.1in}
	\subfigure{\includegraphics[width=0.17\linewidth]{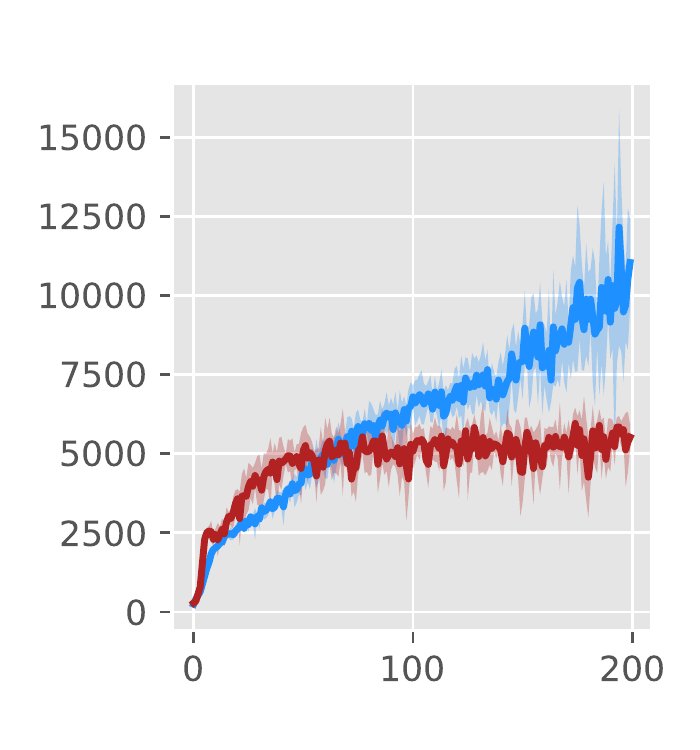}} 
	\hspace{-0.1in}
	\subfigure{\includegraphics[width=0.17\linewidth]{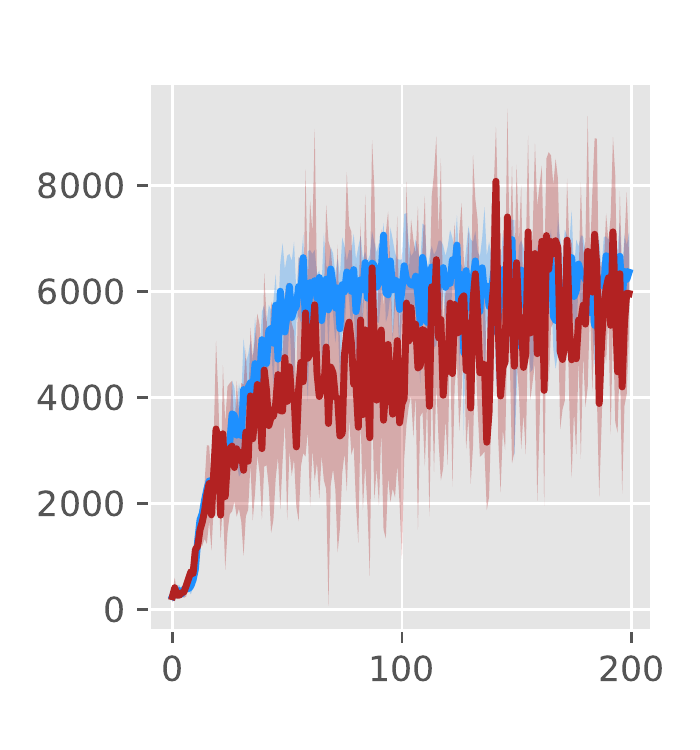}}
	
	\vspace{-0.25in}
	\subfigure{\includegraphics[width=0.17\linewidth]{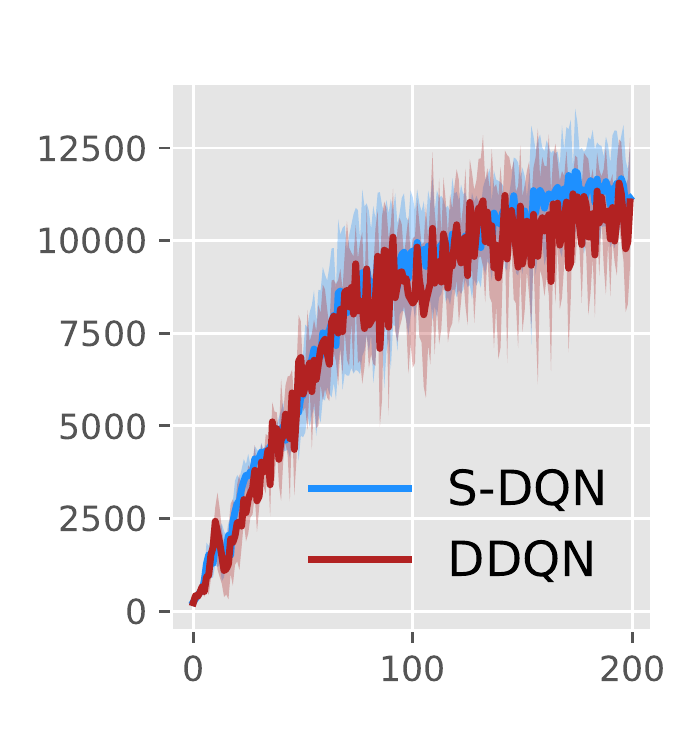}} 
    \hspace{-0.1in}
	\subfigure{\includegraphics[width=0.17\linewidth]{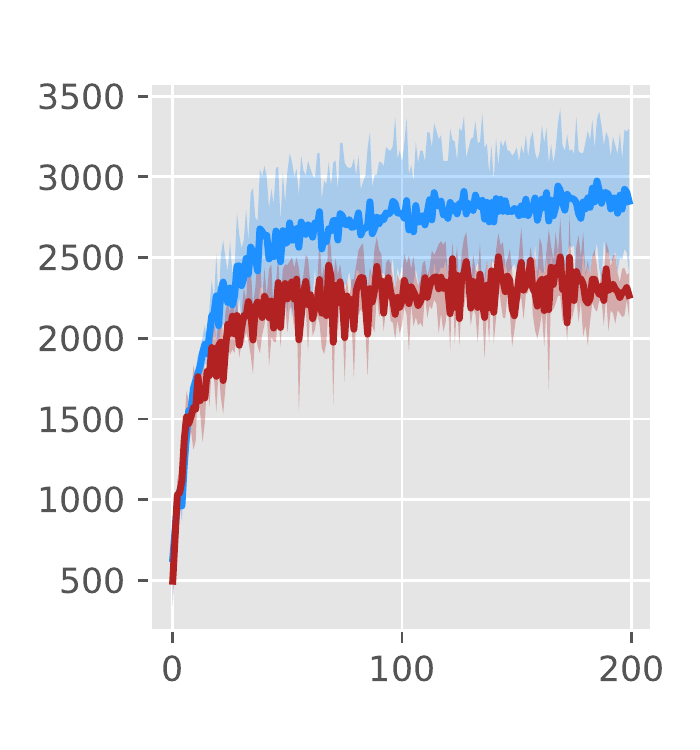}}
	\hspace{-0.1in}
	\subfigure{\includegraphics[width=0.17\linewidth]{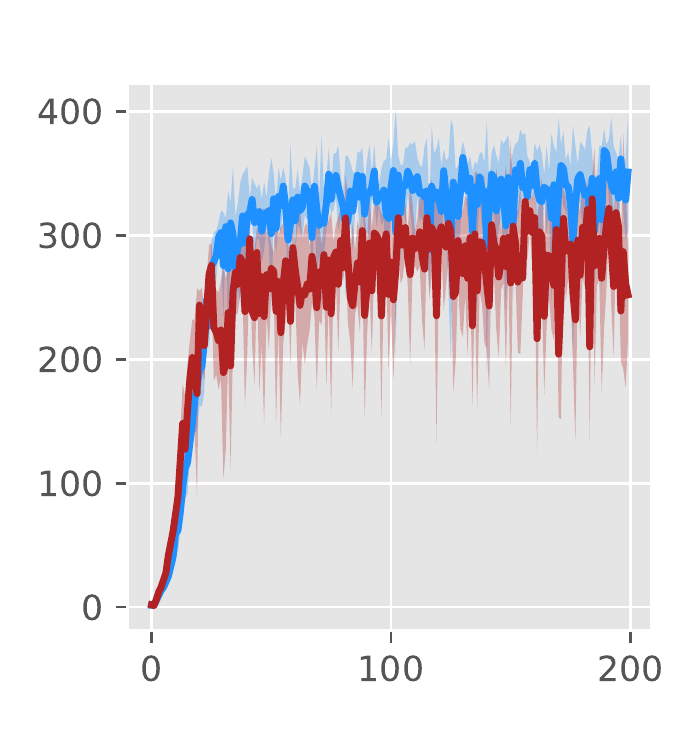}} 
	\hspace{-0.1in}
	\subfigure{\includegraphics[width=0.17\linewidth]{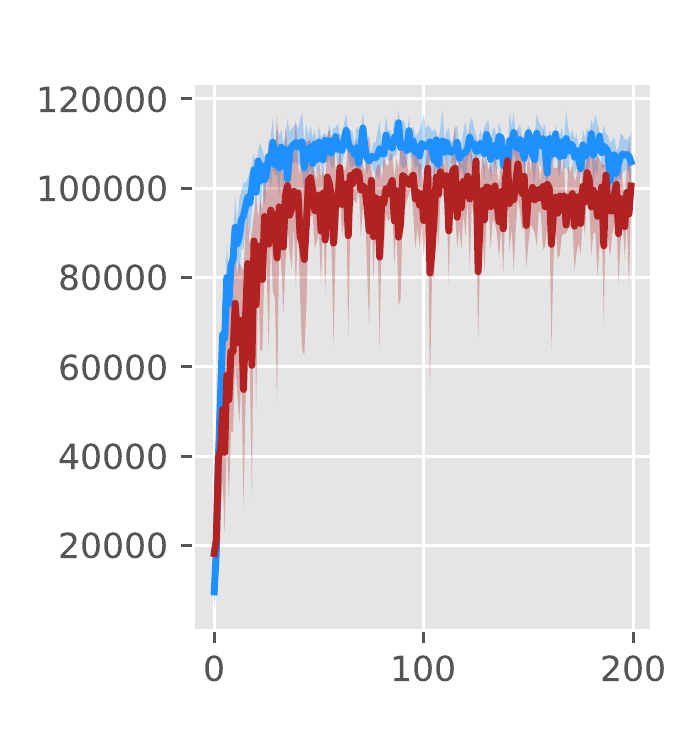}} 
	\hspace{-0.1in}
	\subfigure{\includegraphics[width=0.17\linewidth]{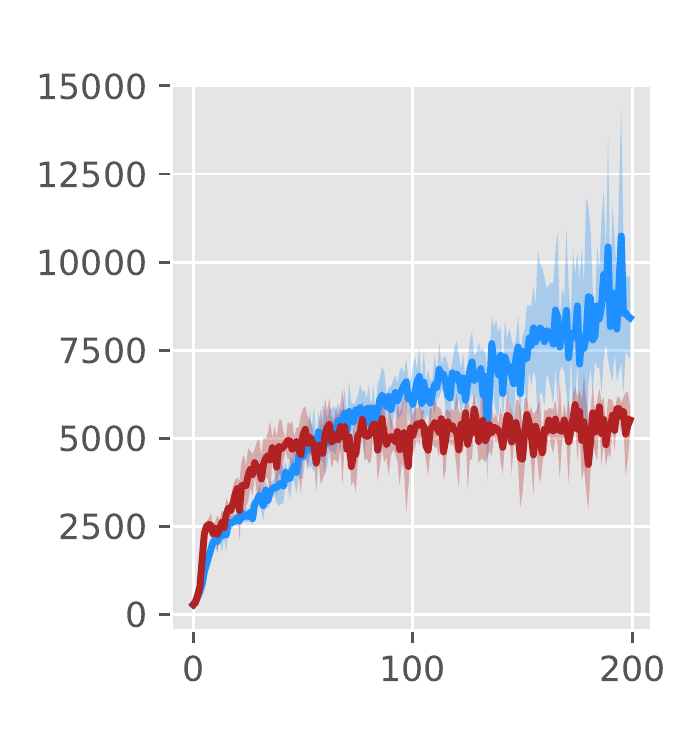}} 
	\hspace{-0.1in}
	\subfigure{\includegraphics[width=0.17\linewidth]{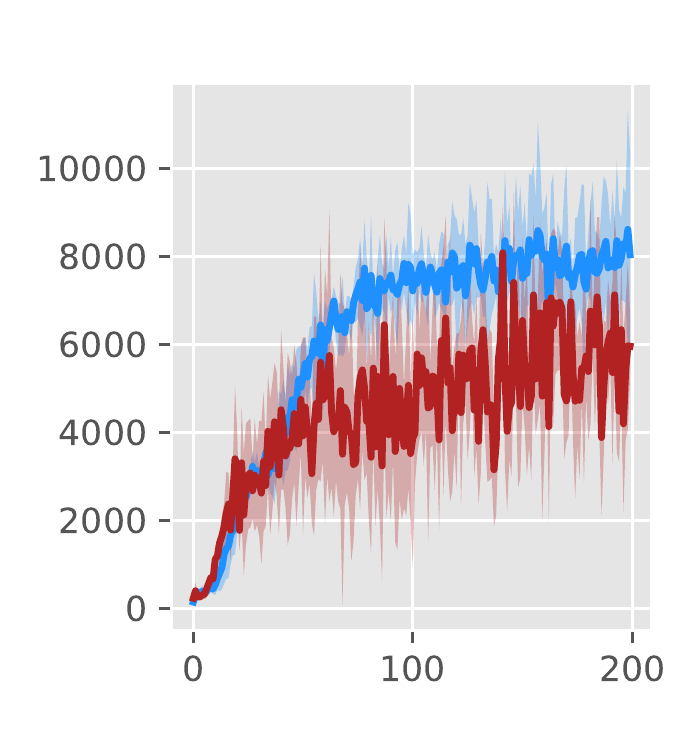}}
	
	\caption{Performance of (top row) S-DQN vs. DQN, (middle row) S-DDQN vs. DDQN, and (bottom row) S-DQN vs. DDQN  on the Atari games. X-axis and Y-axis correspond to the training epoch and test score, respectively. All curves are averaged over five independent random seeds.}
	\label{fig:score}
	\vspace{-0.2in}
\end{figure*}

Note that although the bound for mellowmax under the entropy-regularized~\gls{MDP} framework was shown in~\citet{LeeChoiOh} to have a better scalar term $\log(m)$ instead of $(m-1)$, its convergence rate is linear w.r.t. $\tau$. In contrast, our softmax Bellman operator result has an exponential rate, as shown in Part $(II)$, which potentially gives insight into why very large values of $\tau$ are not required experimentally in Section~\ref{sec:experiments}. The error bound in the sparse Bellman operator~\citep{LeeChoiOh} improves upon $\log (m)$, but is still linear w.r.t. $\tau$.  We further empirically illustrate the faster convergence for the softmax Bellman operator in Figure~\ref{fig:comp} 
of Section~\ref{sec:comparison}.

\section{Main Experiments}
\label{sec:experiments}

Our theoretical results apply to the case of a tabular value function representation and known next-state distributions, and they bound suboptimality rather than suggesting superiority of softmax. The goal of our main experiments is to show that despite being sub-optimal w.r.t. the Q-values, the softmax Bellman operator is still useful in practice by improving the~\gls{DQN} and~\gls{DDQN} algorithms. Such an improvement is entirely attributable to changing the max operator to the softmax operator in the target network.


We combine the softmax operator with~\gls{DQN} and~\gls{DDQN}, by replacing the max function therein
with the softmax function, with all other steps being the same. 
The corresponding new algorithms 
are termed as S-DQN and S-DDQN, respectively. Their exploration strategies are set to be $\epsilon$-greedy, the same as~\gls{DQN} and~\gls{DDQN}. 
Our code is provided at~\url{https://github.com/zhao-song/Softmax-DQN}.

We tested on six Atari games: Q*Bert, Ms. Pacman, Crazy Climber, Breakout, Asterix,
and Seaquest. Our code is built on the \texttt{Theano+Lasagne} implementation from
\url{https://github.com/spragunr/deep_q_rl/}. The training contains 200 epochs in total.
The test procedures and all the hyperparameters are set the same as~\gls{DQN}, with
details described in~\citet{Mnih+al:2015}. The inverse temperature parameter $\tau$ was selected based on a grid search over $\{1, 5, 10\}$. We also implemented the logarithmic cooling
scheme~\citep{MitraRomeoSangiovanni} in simulated annealing to gradually increase $\tau$, but did not observe a better policy,
compared with the constant temperature. The results statistics are obtained by running
with five independent random seeds.

\begin{table}[!t]
    \caption{Mean of test scores for different values of $\tau$ in S-DDQN 
    (standard deviation in parenthesis). Each game is denoted with its
    initial. Note that S-DDQN reduces to DDQN when $
    \tau = \infty$.} 
    \centering
    \begin{tabular}{p{0.02in}cccc}
	    \toprule
	    $\tau$ & 1 & 5 & 10 & $\infty$ \\ 
	    \midrule
	    \multirow{2}{*}{Q} & $\mathbf{12068.66}$ & 11049.28 & 11191.31 & 10577.76 \\
	    & (1085.65) & (1565.57) & (1336.35) & (1508.27) \\
	    \hline
	    \multirow{2}{*}{M} & 2492.40 & $\mathbf{2566.44}$ & 2546.18 & 2293.73 \\
	    & (183.71) & (227.24) & (259.82) & (160.50) \\
	    \hline
	    \multirow{2}{*}{B} & 313.08 & $\mathbf{350.88}$ & 303.71 & 284.64 \\
	    & (20.13) & (35.58) & (65.59) & (60.83) \\
	    \hline
	    \multirow{2}{*}{C} & 107405.11 & $\mathbf{111111.07}$ & 104049.46 & 96373.08 \\
	    & (4617.90) & (5047.19) & (6686.84) & (9244.27) \\
	    \hline
	    \multirow{2}{*}{A} & 3476.91 & $\mathbf{10266.12}$ & 6588.13 & 5523.80 \\
	    & (460.27) & (2682.00) & (1183.10) & (694.72) \\
	    \hline
	    \multirow{2}{*}{S} & 272.20 & 2701.05 & $\mathbf{6254.01}$ & 5695.35 \\
	    & (49.75) & (10.06) & (697.12) & (1862.59) \\
	    \bottomrule
	\end{tabular}
    \label{tab:score_tau}
    \vspace{-0.25in}
\end{table} 

Figure~\ref{fig:score} shows the mean and one standard deviation for the average test score
on the Atari games, as a function of the training epoch: 
In the top two rows, S-DQN and S-DDQN generally achieve higher maximum scores and faster learning than their max counterparts, illustrating the promise of replacing max
with softmax. 
For the game Asterix, often used as an example of the overestimation
issue in~\gls{DQN}~\citep{HasseltGuezSilver,AnschelBaramShimkin}, both S-DQN
and S-DDQN have much higher test scores than their max counterparts, which suggests that the
softmax operator can mitigate the overestimation bias. 
Although not as dramatic as for Asterix, the bottom row of Figure~\ref{fig:score} shows that S-DQN generally outperforms DDQN with higher test scores for the entire test suite, which suggests the possibility of using softmax operator over double Q-learning~\cite{Hasselt} to reduce the overestimation bias, in the presence of function approximation.

Table~\ref{tab:score_tau} shows the
test scores of S-DDQN with different values of $\tau$, obtained by
averaging the scores from the last 10 epochs. Note that S-DDQN achieves higher test scores
than its max counterpart, for most of the values of $\tau$. Table~\ref{tab:score_tau} further suggests
a trade-off when selecting the values for 
$\tau$: Setting $\tau$ too small will move $\cTs$ further away from the Bellman optimality, since
the max operator (corresponding to $\tau = \infty$) is employed in the standard Bellman equation. 
On the other hand, using a larger $\tau$ will lead to the issues of overestimation 
and high gradient noise, as to be discussed in Section~\ref{sec:why}.

\begin{table}[!ht]
	\vspace{-0.1in}
    \caption{Mean of test scores for different Bellman operators (standard deviation in parenthesis). The statistics are averaged over five independent random seeds.} 
    \centering
    \begin{tabular}{lccc}
	    \toprule
	     & Max & Softmax & Mellowmax \\ 
	    \midrule
	    \multirow{2}{*}{Q*Bert} & 8331.72 & 11307.10 & $\mathbf{11775.93}$ \\
	    & (1597.67) & (1332.80) & (1173.51) \\
	    \hline
	    \multirow{2}{*}{Ms. Pacman} & 2368.79 & $\mathbf{2856.82}$ & 2458.76\\
	    & (219.17) & (369.75) & (130.34) \\
	    \hline
	    \multirow{2}{*}{C. Climber} & 90923.40 & $\mathbf{106422.27}$ & 99601.47\\
	    & (11059.39) & (4821.40) & (19271.53)\\
	    \hline
	    \multirow{2}{*}{Breakout} & 255.32 & 345.56 & $\mathbf{355.94}$\\
	    & (64.69) & (34.19) & (25.85) \\
	    \hline
	    \multirow{2}{*}{Asterix} & 196.91 & 8868.00 & $\mathbf{11203.75}$\\
	    & (135.16) & (2167.35) & (3818.40)\\
	    \hline
	    \multirow{2}{*}{Seaquest} & 4090.36 & $\mathbf{8066.78}$ & 6476.20 \\
	    & (1455.73) & (1646.51) & (1952.12)\\
	    \bottomrule
	\end{tabular}
    \label{tab:score_mellow}
    \vspace{-0.1in}
\end{table}

To compare against the mellowmax operator, we combine it with DQN in the same off-policy fashion as the softmax operator~\footnote{We tried the mellowmax operator for exploration as in~\citet{AsadiLittman}, but observed much worse performance.}. Note that the root-finding approach for mellowmax in~\citet{AsadiLittman} needs to compute the optimal inverse temperature parameter for every state, and thus is too computationally expensive to be applied here. Consequently, we tune the inverse temperature parameter for both softmax and mellowmax operators from \{1, 5, 10\}, and report the best scores. Table~\ref{tab:score_mellow} shows that both softmax and mellowmax operators can achieve higher test scores than the max operator. Furthermore, the scores from softmax are higher or similar to those of mellowmax on all games except Asterix. The competitive performance of the softmax operator  here suggests that it is still preferable in certain domains, despite its non-contraction property.

\section{Why Softmax Helps?}
\label{sec:why}

One may wonder why the softmax Bellman operator can achieve the surprising performance in Section~\ref{sec:experiments}, as the greedy policy from the Bellman equation suggests that the max operator should be optimal. The softmax operator is not used for exploration in this paper, nor is it motivated by regularizing the policy, as in the entropy regularized~\glspl{MDP}~\cite{fox2016taming,SchulmanAbbeelChen,neu2017unified,NachumNorouziXuSchuurmans,AsadiLittman,LeeChoiOh}. 

As discussed in earlier work~\citep{HasseltGuezSilver,AnschelBaramShimkin}, the max operator leads to the significant issue of overestimation for the Q-function in~\glspl{DQN}, which further causes excessive gradient noise to destabilizes the optimization of neural networks.
Here we aim to provide analysis of how the softmax Bellman operator can overcome these issues, and also the corresponding evidence for overestimation and gradient noise reduction in~\glspl{DQN}. Although our analysis is focused on the softmax Bellman operator, this work could potentially give further insight into practical benefits of entropy regularization as well. 

\subsection{Overestimation Bias Reduction}
\label{subsec:bias_reduction}

Q-learning's overestimation bias, due to the max operator, was first discussed in~\citet{Thrun93b}. It was later shown in~\citet{HasseltGuezSilver} and \citet{AnschelBaramShimkin} that overestimation leads to the poor performance of~\glspl{DQN} in some Atari games. Following the same assumptions as~\citet{HasseltGuezSilver}, we can show the softmax operator reduces the overestimation bias. Furthermore, we quantify the range of the overestimation reduction, by providing both lower and upper bounds.

\begin{figure*}[!t]
    \vspace{-0.1in}
	\centering
	\subfigure{\includegraphics[width=0.325\linewidth]{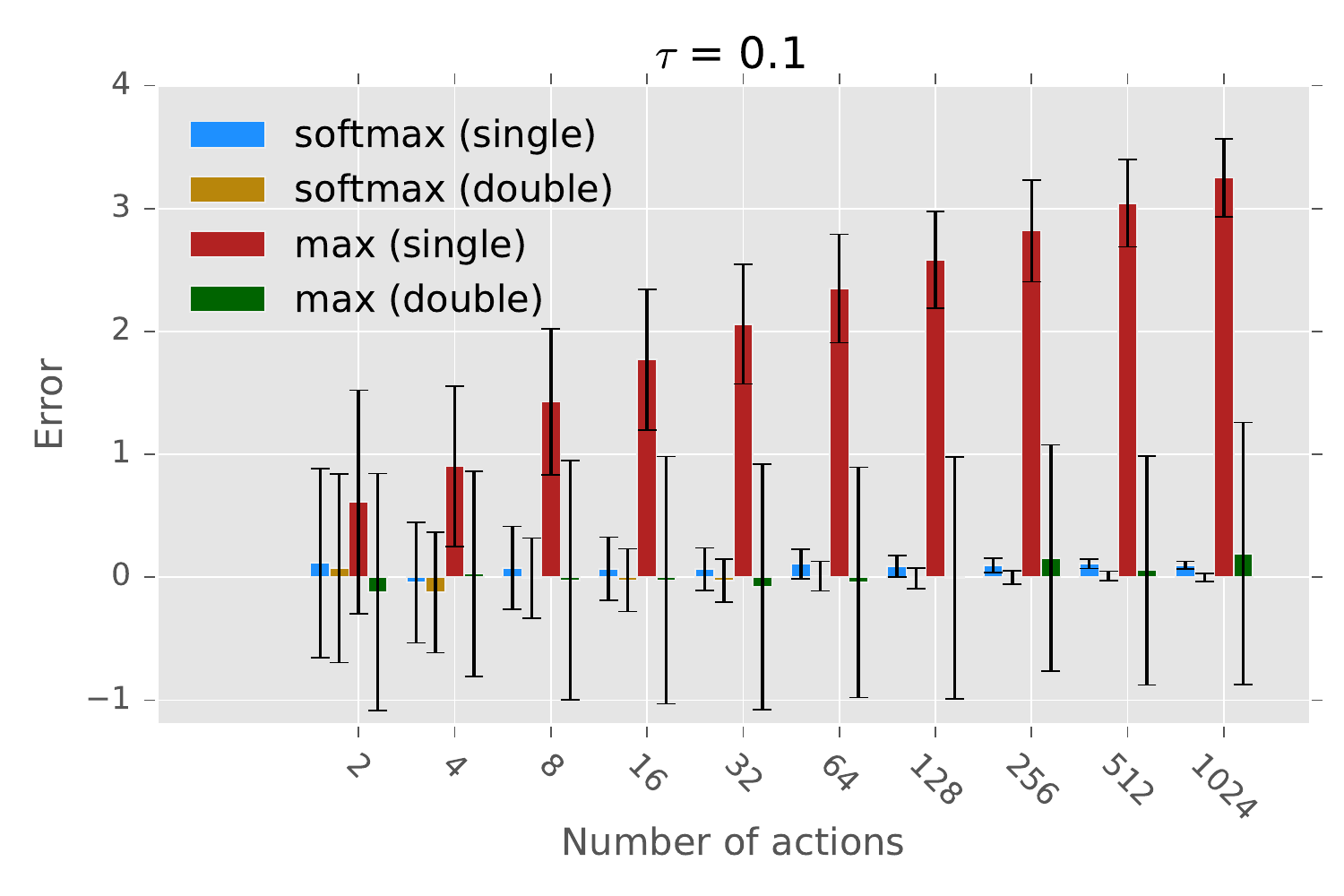}} 
	\subfigure{\includegraphics[width=0.325\linewidth]{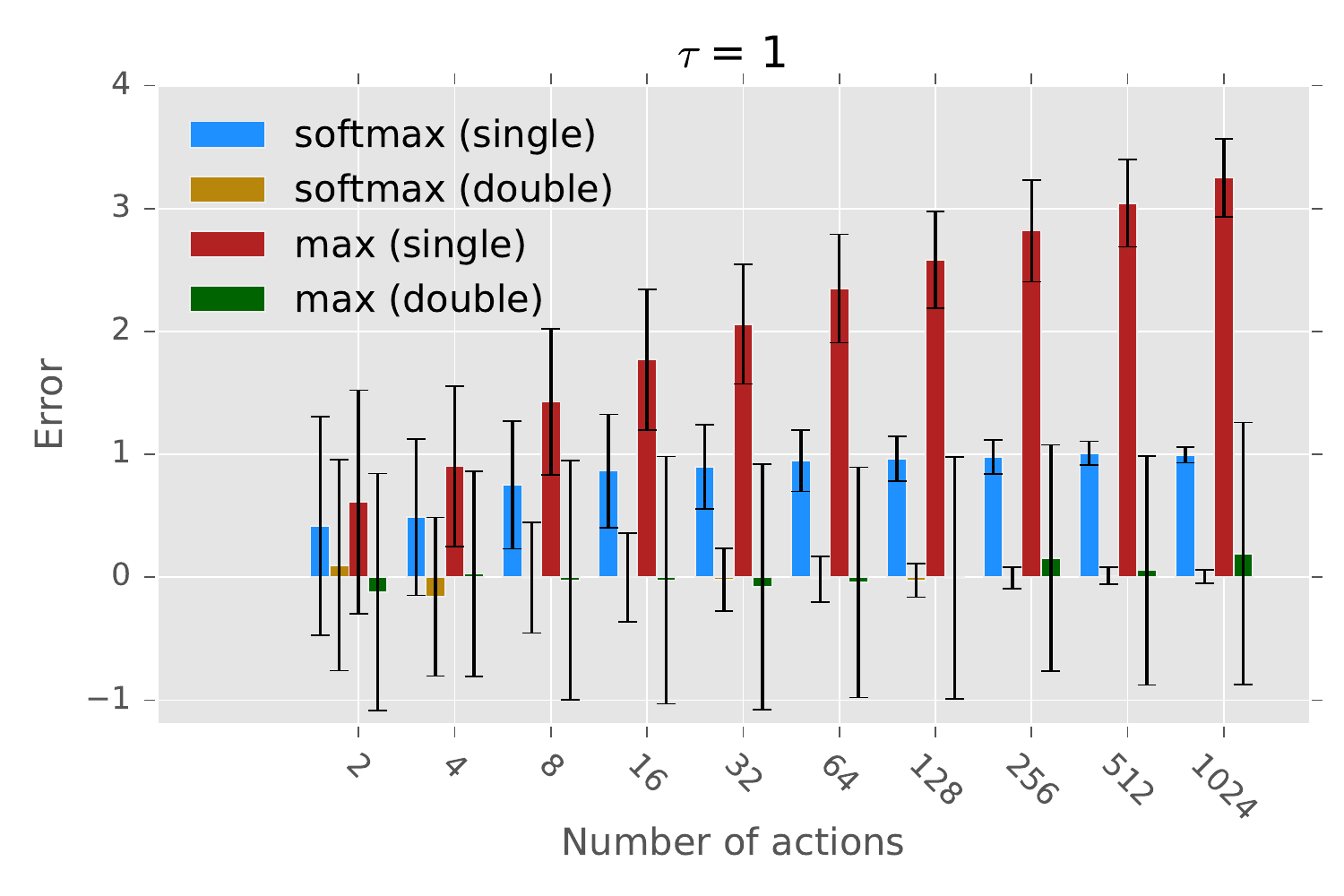}} 
	\subfigure{\includegraphics[width=0.325\linewidth]{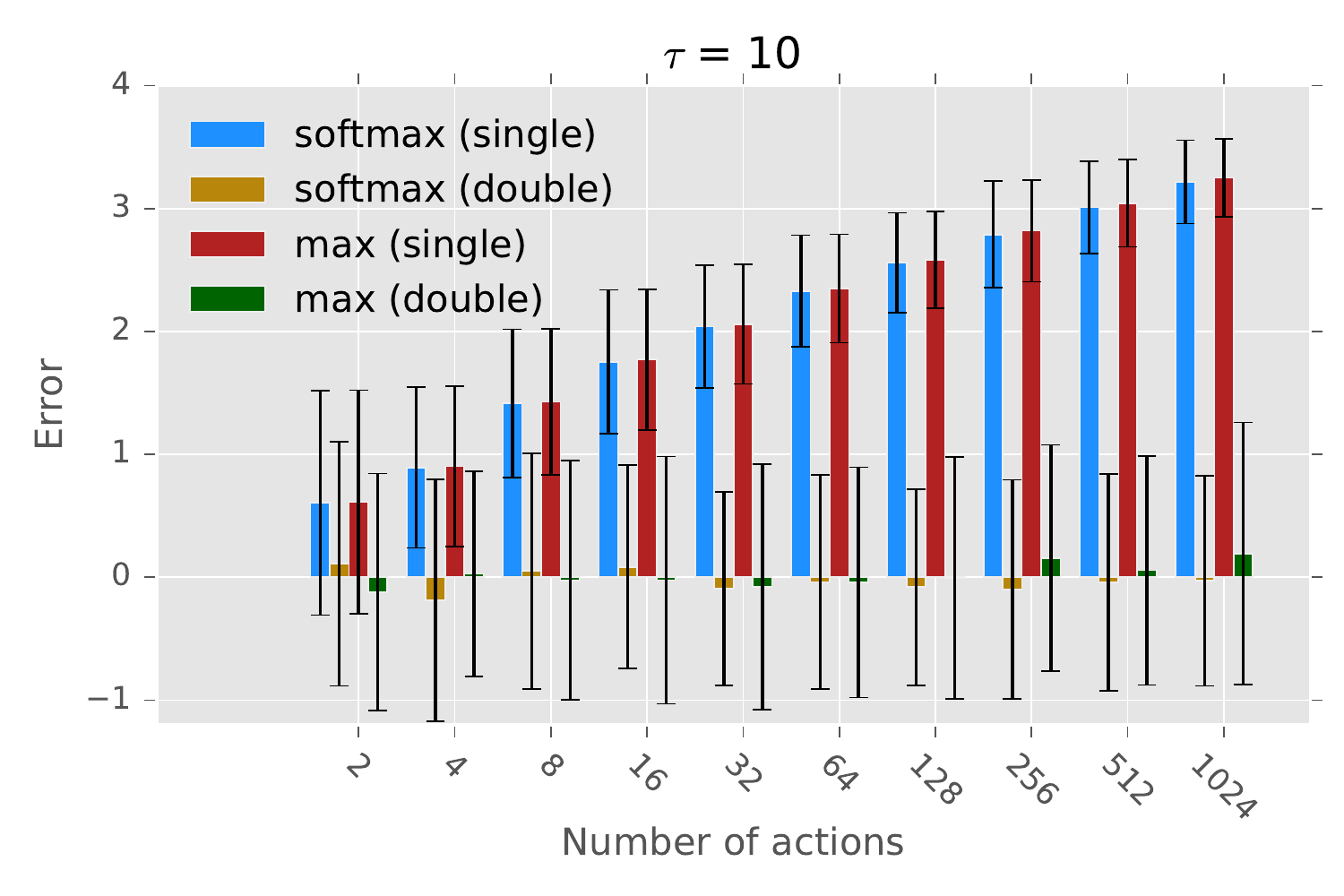}} 
	\vspace{-0.1in}
	\caption{The mean and one standard deviation for the overestimation error for different values of $\tau$.}
	\vspace{-0.15in}
	\label{fig:bias}
\end{figure*}

\begin{theorem}
Given the same assumptions as~\citet{HasseltGuezSilver}, where 
$(A1)$ there exists some $V^{*} (s)$ such that the true state-action value function
satisfies $Q^{*} (s, a) = V^{*} (s)$, for different actions. 
$(A2)$ the estimation error is modeled as $Q_t (s, a) = V^{*} (s) + \epsilon_a$, 
then 

$(I)$ the overestimation errors from $\cTs$ are smaller or equal to those of $\cT$ using the max operator, for any $\tau \geq 0$;

$(II)$ the overestimation reduction by using $\cTs$ in lieu of $\cT$ is within $\big[ \frac{\widehat{\delta} (s)}{m \exp [\tau \,\widehat{\delta} (s) ] }, \, (m-1) \max \{ \frac{1}{\tau+2}, \frac{ 2  Q_{\text{max}} }{1 + \exp(\tau) } \} \big]$;

$(III)$ the overestimation error for $\cTs$ monotonically increases w.r.t. $\tau \in [0, \infty)$.

\label{thm:bias}
\end{theorem}
\vspace{-0.1in}

An observation about Theorem~\ref{thm:bias} is that for any positive value of $\tau$, there will still be some potential for overestimation bias because noise can also influence the softmax operator. Depending upon the amount of noise, it is possible that, unlike double DQN, the reduction caused by softmax could exceed the bias introduced by max. This can be seen in our experimental results below, which show that it is possible to have negative error (overcompensation) from the use of softmax. However, when combined with double Q-learning, this effect becomes very small and decreases with the number of actions.

To elucidate Theorem~\ref{thm:bias}, we simulate standard normal variables
${\epsilon_a}$ with $100$ independent trials for each action $a$, using the same 
setup as~\citet{HasseltGuezSilver}. 
Figure~\ref{fig:bias} shows the mean and one 
standard deviation of the overestimation bias, for different values of $\tau$ in the softmax operator. For both single and double 
implementations of the softmax operator, they achieve smaller overestimation errors 
than their max counterparts, thus validating our theoretical results. 
Note that the gap becomes smaller when $\tau$ increases, which
is intuitive as $\cTs \rightarrow \cT $ when $\tau \rightarrow \infty$,
and also consistent with the monotonicity result in Theorem~\ref{thm:bias}.

\begin{figure}[!h]
	\vspace{-0.1in}
	\centering
	\subfigure{\includegraphics[width=0.49\linewidth]{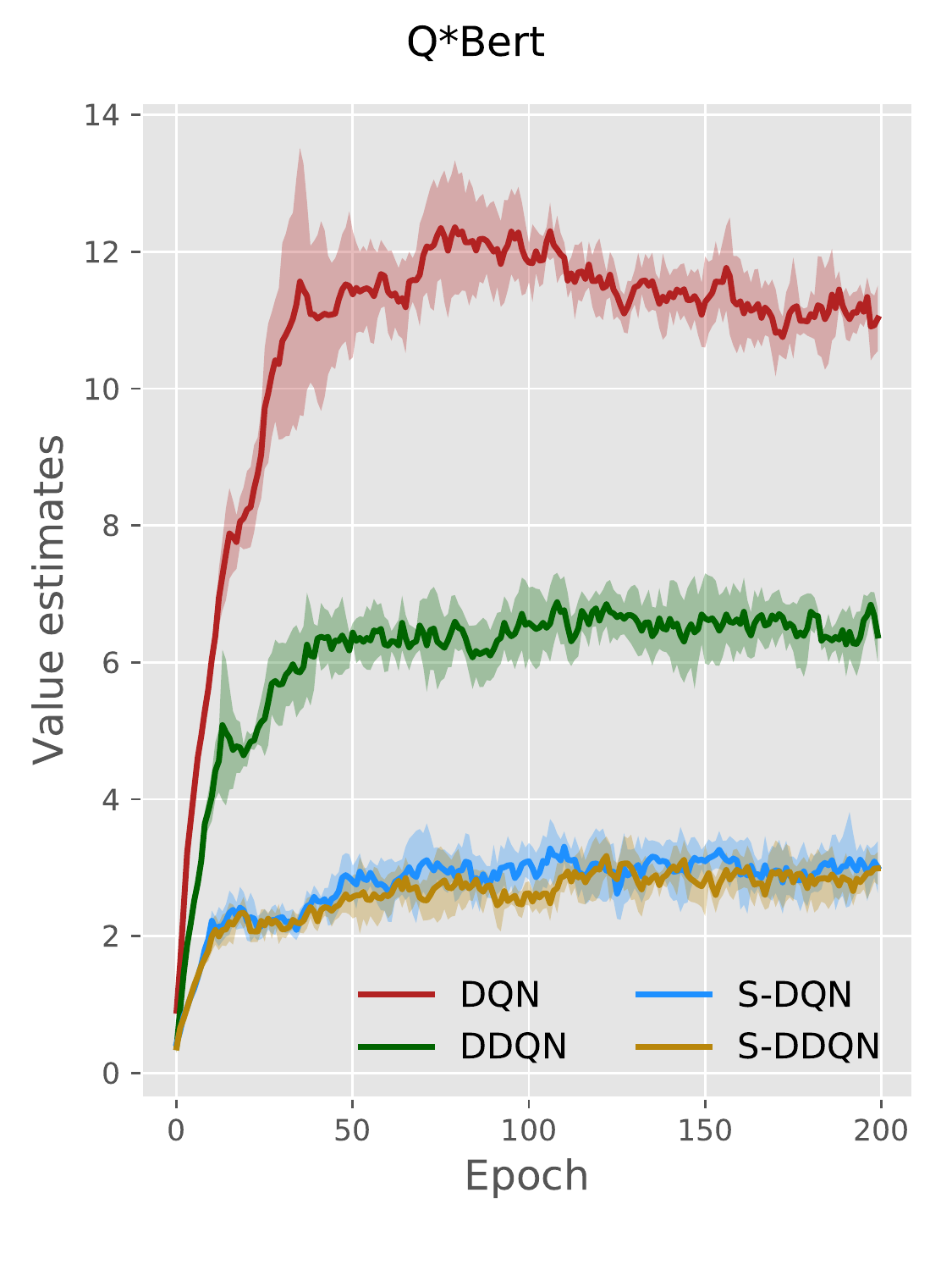}} 
	\vspace{-0.1in}
	\subfigure{\includegraphics[width=0.49\linewidth]{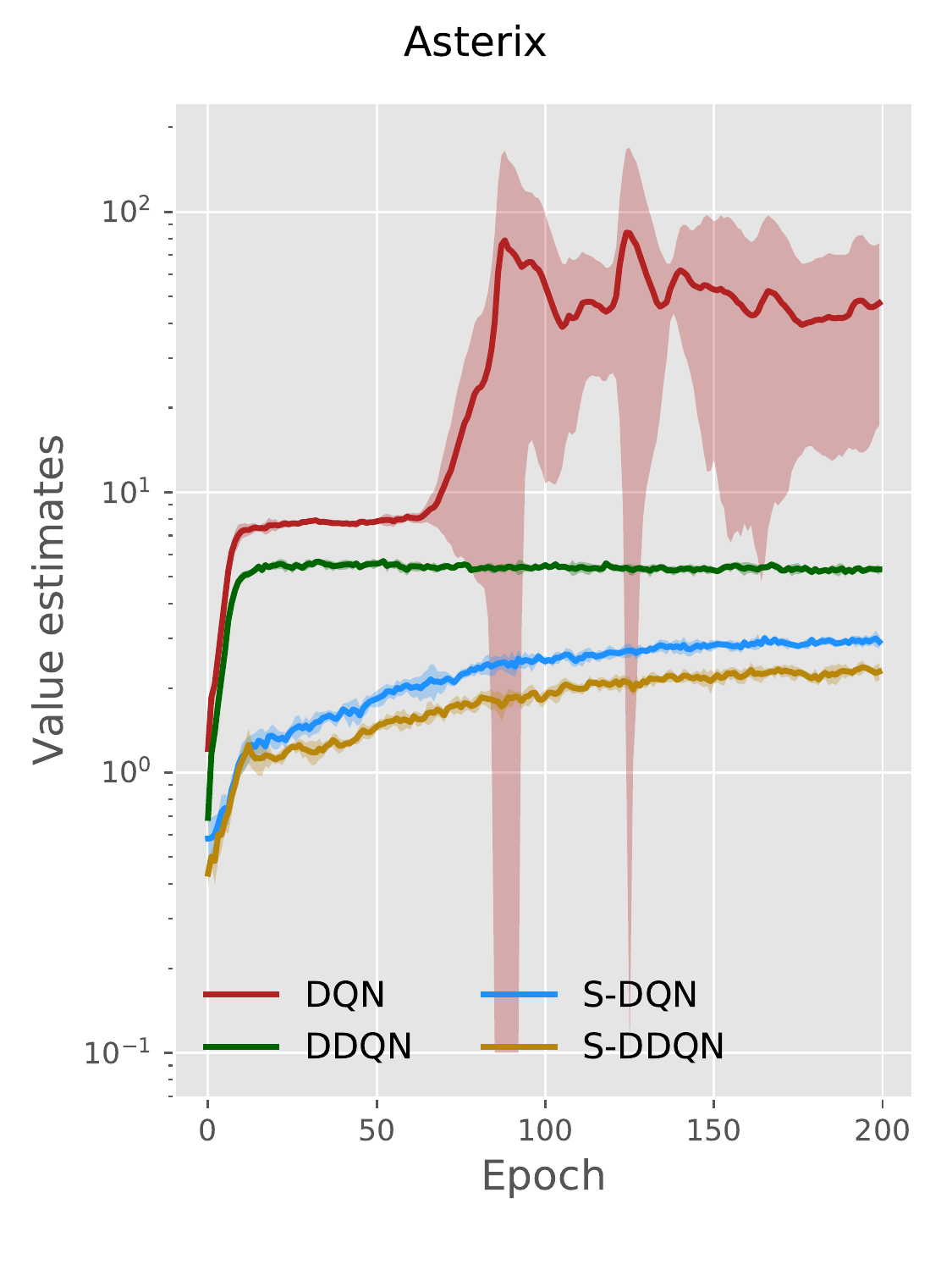}} 
	\vspace{-0.15in}
	\caption{Mean and one standard deviation of the estimated Q-values on the Atari games,
	for different methods.}
	\label{fig:qvalue}
\end{figure}

\begin{figure}[!h]
	\vspace{-0.1in}
	\centering
	\subfigure{\includegraphics[width=0.49\linewidth]{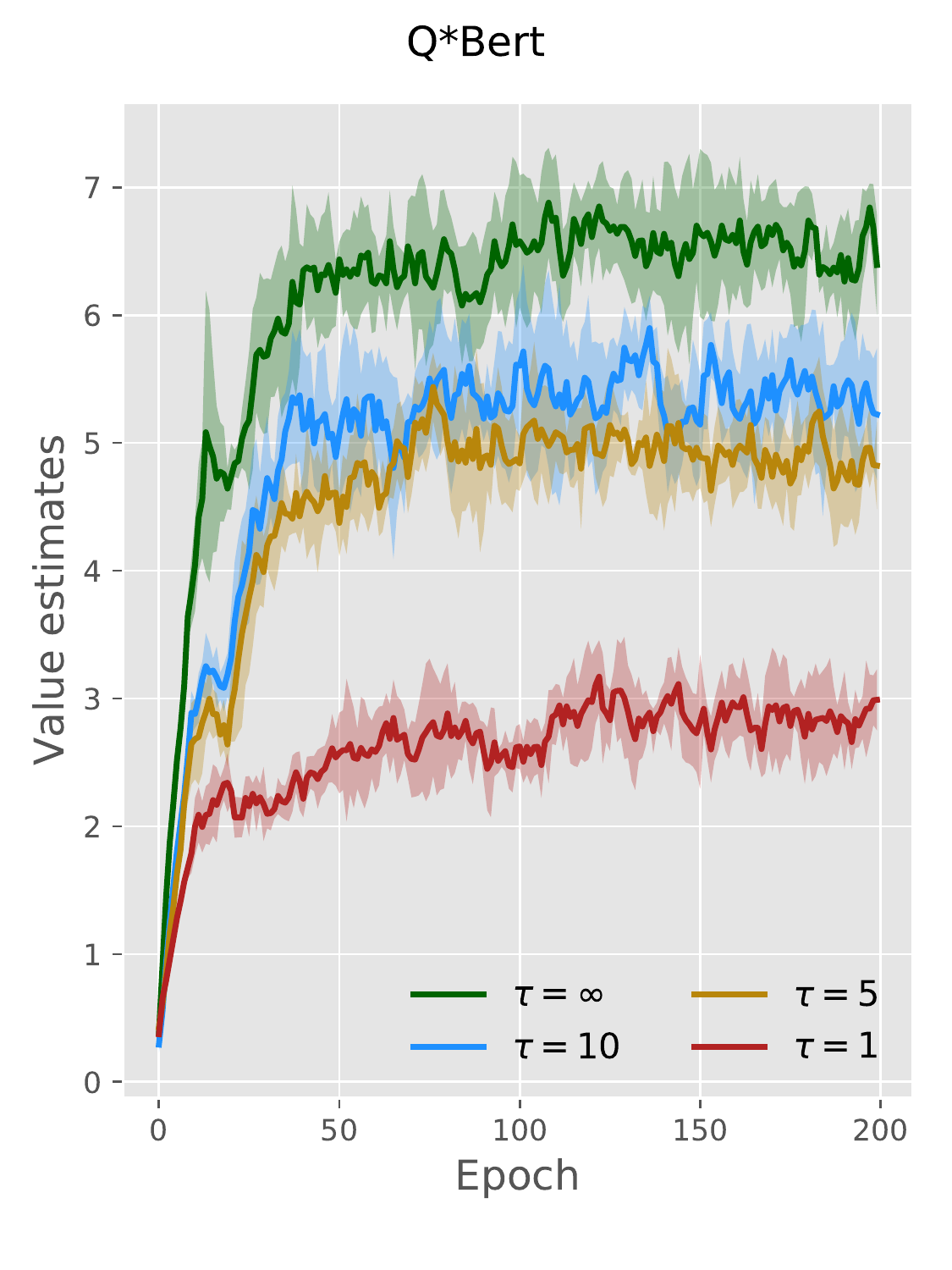}} 
	\vspace{-0.1in}
	\subfigure{\includegraphics[width=0.49\linewidth]{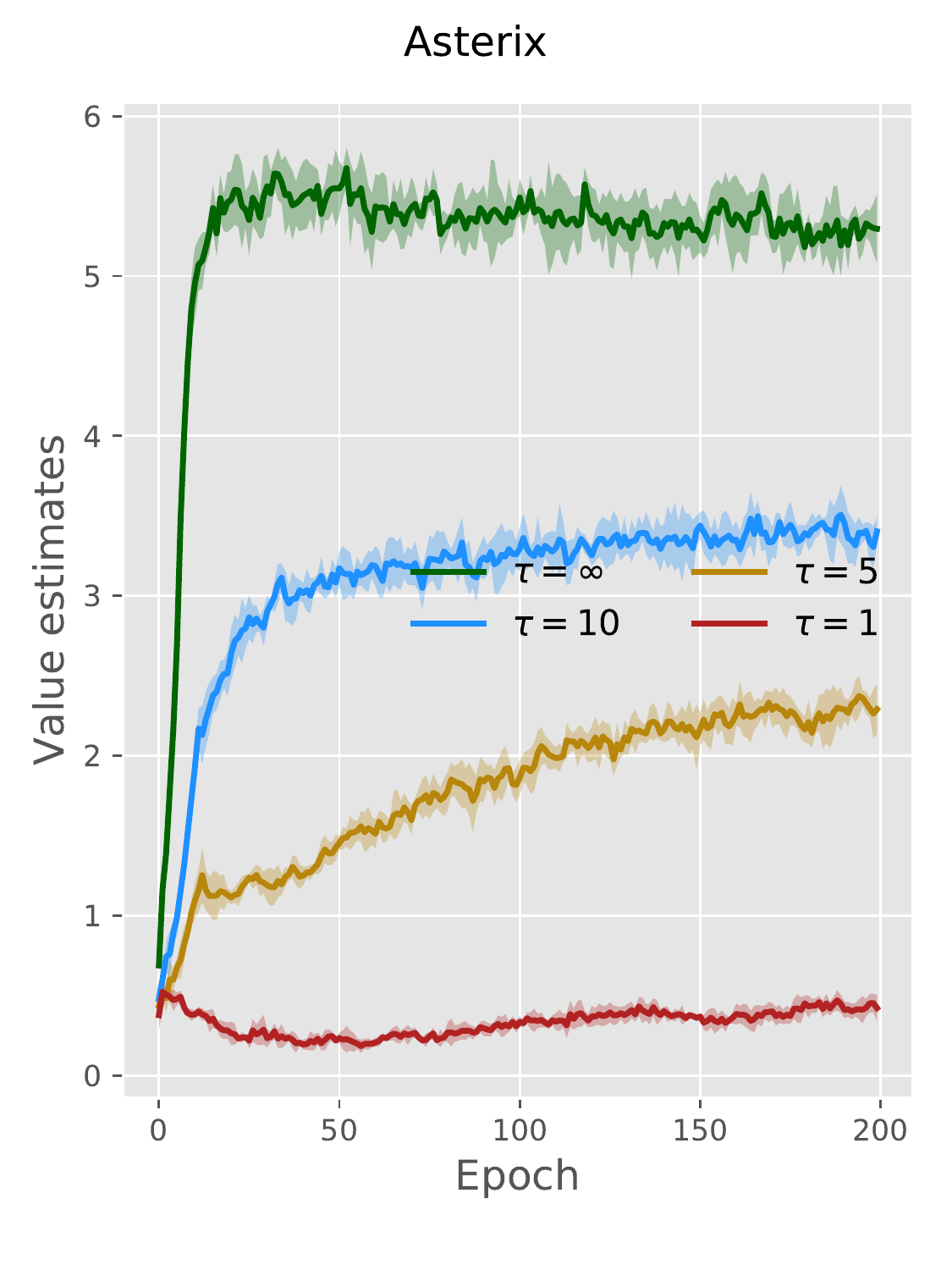}}
	\vspace{-0.1in}
	\caption{Mean and one standard deviation of the estimated Q-values on the Atari games,
	for different values of $\tau$ in S-DDQN.}
	\label{fig:qvalue_tau}
	\vspace{-0.1in}
\end{figure}
\vspace{-0.15in}

We further check the estimated Q-values on Atari games~\footnote{Due to the space limit,
we plot fewer games from here, and provide the full plots in Supplemental Material.}, by plotting the
corresponding mean and one standard deviation in Figure~\ref{fig:qvalue}.
Both S-DQN and S-DDQN achieve smaller values than~\gls{DQN}, which verifies that the softmax operator can reduce the overestimation
bias. This reduction can partly explain the higher test scores for S-DQN and S-DDQN 
in Figure~\ref{fig:score}. Moreover, we plot the estimated Q-values of S-DDQN for different values of $\tau$ in Figure~\ref{fig:qvalue_tau}, and the result is consistent with Theorem~\ref{thm:bias} that increasing $\tau$ leads to larger estimated Q-values and, thus, more risk of overestimation.   It also shows that picking a value of $\tau$ that is very small may introduce a risk of underestimation that can hurt performance, as shown for Asterix, with $\tau=1$. (See also Table \ref{tab:score_tau}.)

\subsection{Gradient Noise Reduction}
\label{subsec:variance_reduction}

\citet{Mnih+al:2015} observed that large gradients can be detrimental to the optimization of~\glspl{DQN}. This was noted in the context of large reward values, but it is possible that overestimation bias can also introduce variance in the gradient and have a detrimental effect. To see this, we first notice that the gradient in the~\gls{DQN} can be represented as
\begin{equation}
	\nabla_{\btheta}  \cL = \mathbb{E}_{s, a, r, s'} \Big\{ \big[ Q_{\btheta} (s, a) - \cT Q_{\btheta^{-}} (s, a) \big] \,
	\nabla_{\btheta} Q_{\btheta} (s, a) \Big\}.
	\label{eq:gradient}
\end{equation}
When the Q-value is overestimated, the magnitude for $\nabla_{\btheta} Q_{\btheta} (s, a)$ in~\eqnref{eq:gradient} could become excessively large and so does the gradient, which may further lead to the instability during optimization. This is supported by Figure~\ref{fig:qvalue}, which shows higher variance in Q-values for DQN vs. others. Since we show in Section~\ref{subsec:bias_reduction} that using softmax in lieu of max in the Bellman operator can reduce the overestimation bias, the gradient estimate in~\eqnref{eq:gradient} can be more stable with softmax. Note that stabilizing the gradient during~\gls{DQN} training was also shown in~\citet{HasseltGuezHesselMnihSilver}, but in the context of adaptive target normalization.

To evaluate the gradient noise on Atari games, we report the $\ell_2$ norm of the gradient in the final fully-connected linear layer, by averaging over 50 independent inputs. As shown in Figure~\ref{fig:grad}, S-DQN and S-DDQN achieve smaller gradient norm and variance than their counterparts, which implies that the softmax operator can facilitate variance reduction in the optimization for the neural networks. For the game Asterix, we also observe that the overestimation of Q-value in~\gls{DQN} eventually causes the gradient to explode, while its competitors avoid this issue by achieving reasonable Q-value estimation. Finally, Figure~\ref{fig:grad_tau} shows that increasing $\tau$ generally leads to higher gradient variance, which is also consistent with our analysis in Theorem~\ref{thm:bias} that the overestimation monotonically increases w.r.t. $\tau$.

\begin{figure}[!h]
	\vspace{-0.1in}
	\centering
	\subfigure{\includegraphics[width=0.49\linewidth]{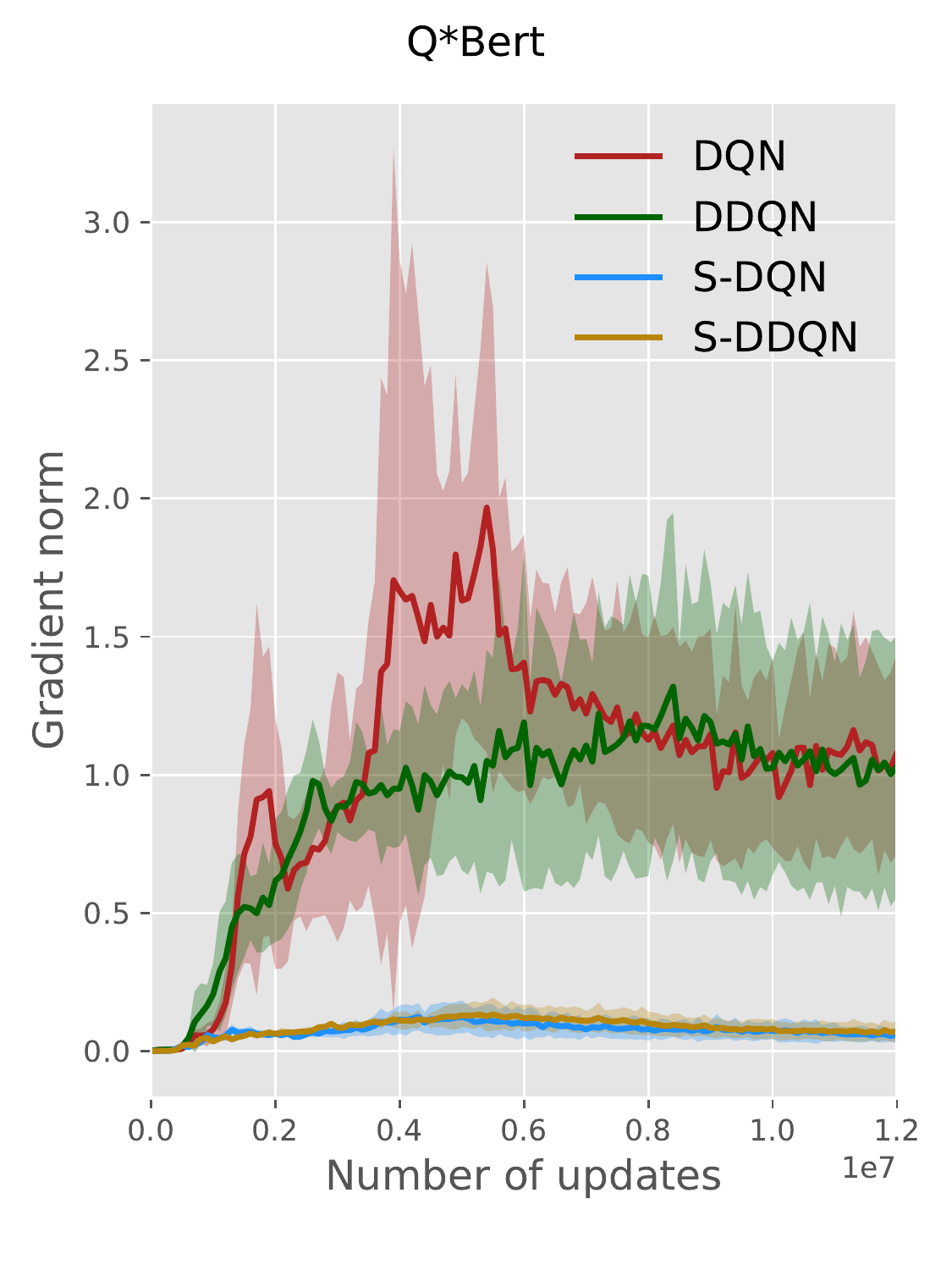}} 
	\vspace{-0.1in}
	\subfigure{\includegraphics[width=0.49\linewidth]{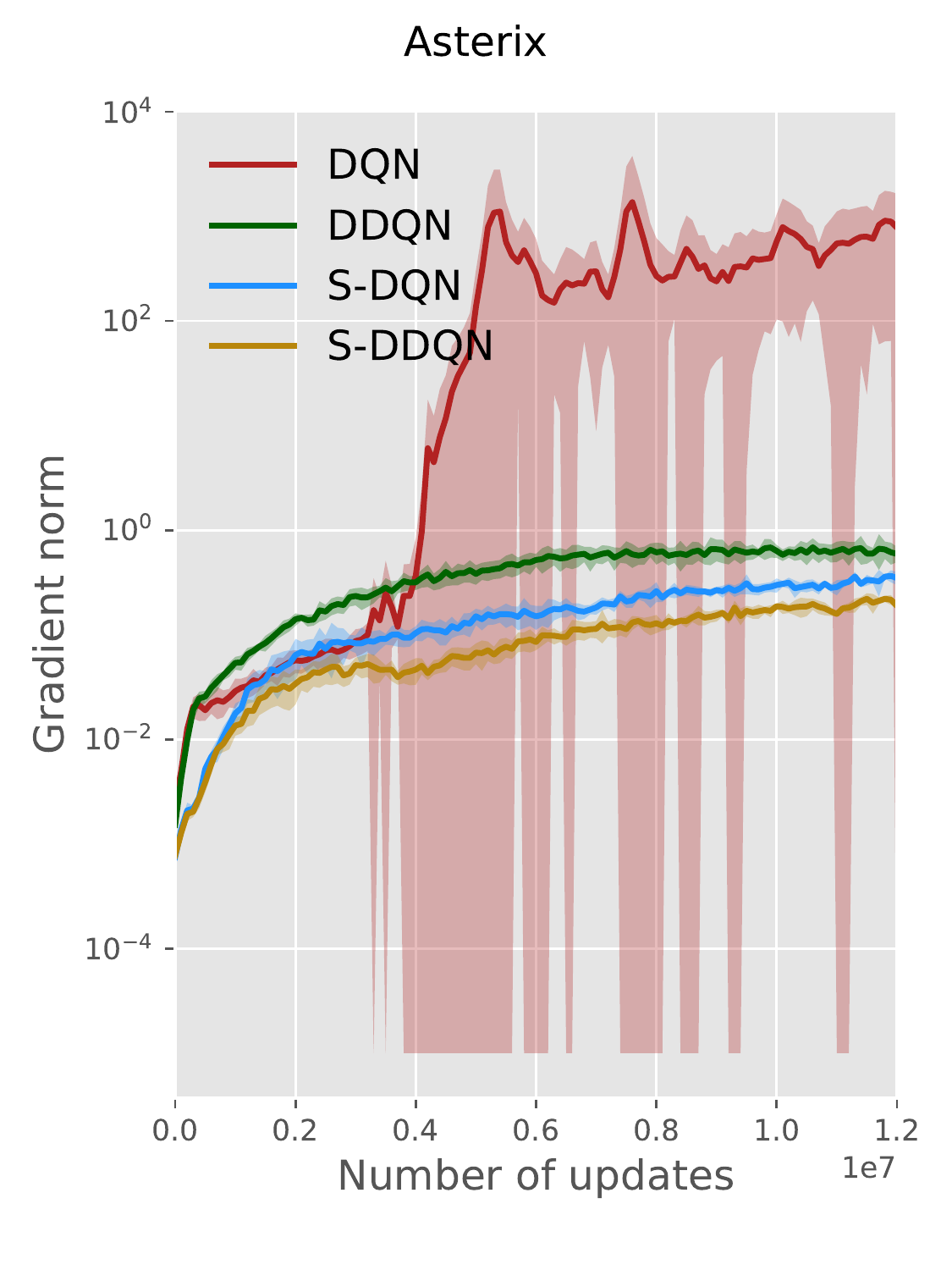}} 
	\vspace{-0.1in}
	\caption{Mean and one standard deviation of the gradient norm on the Atari games,
	for different methods.}
	\label{fig:grad}
	\vspace{-0.1in}
\end{figure}

\begin{figure}[!h]
	\vspace{-0.1in}
	\centering
	\subfigure{\includegraphics[width=0.49\linewidth]{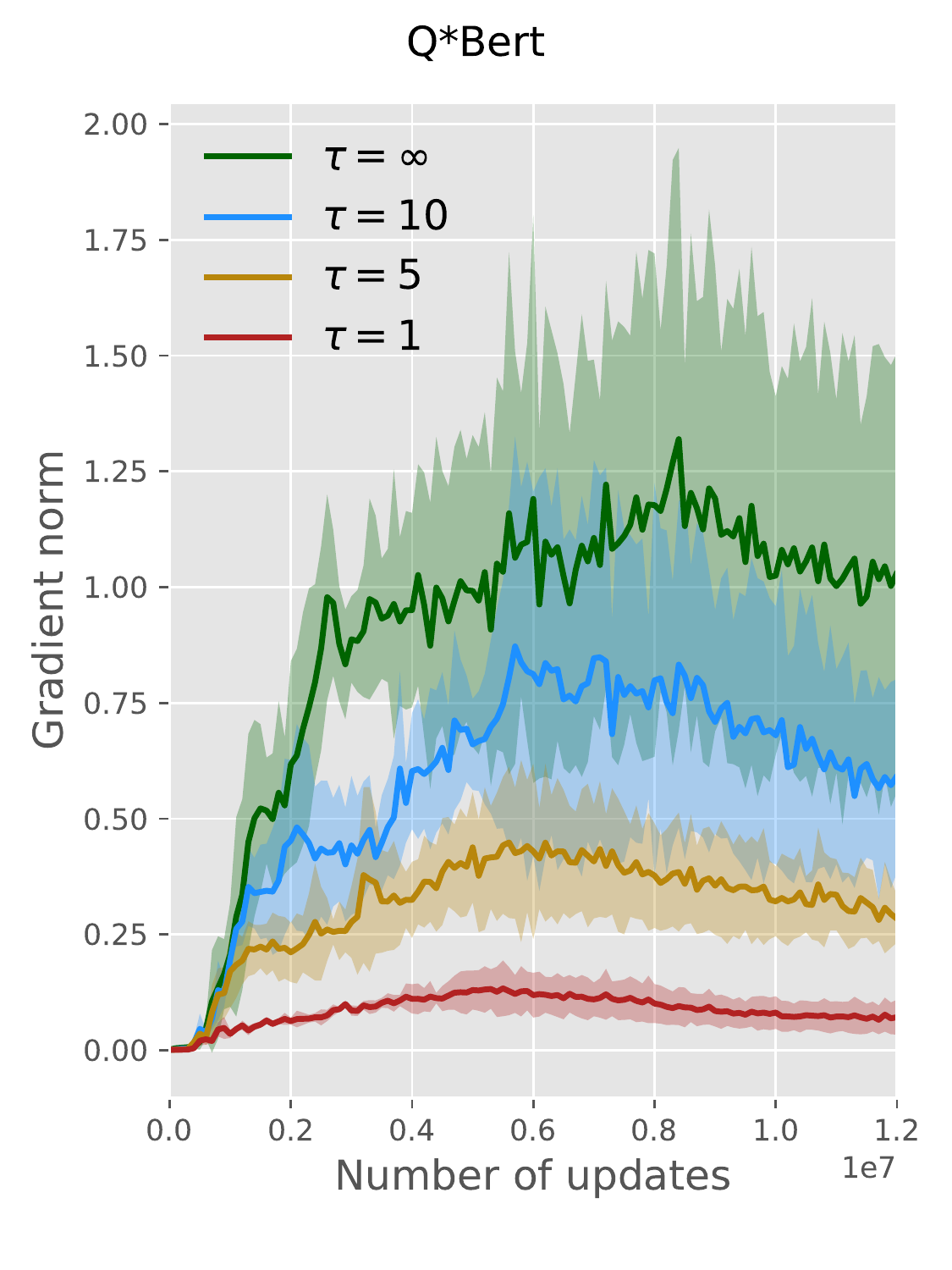}} 
	\vspace{-0.1in}
	\subfigure{\includegraphics[width=0.49\linewidth]{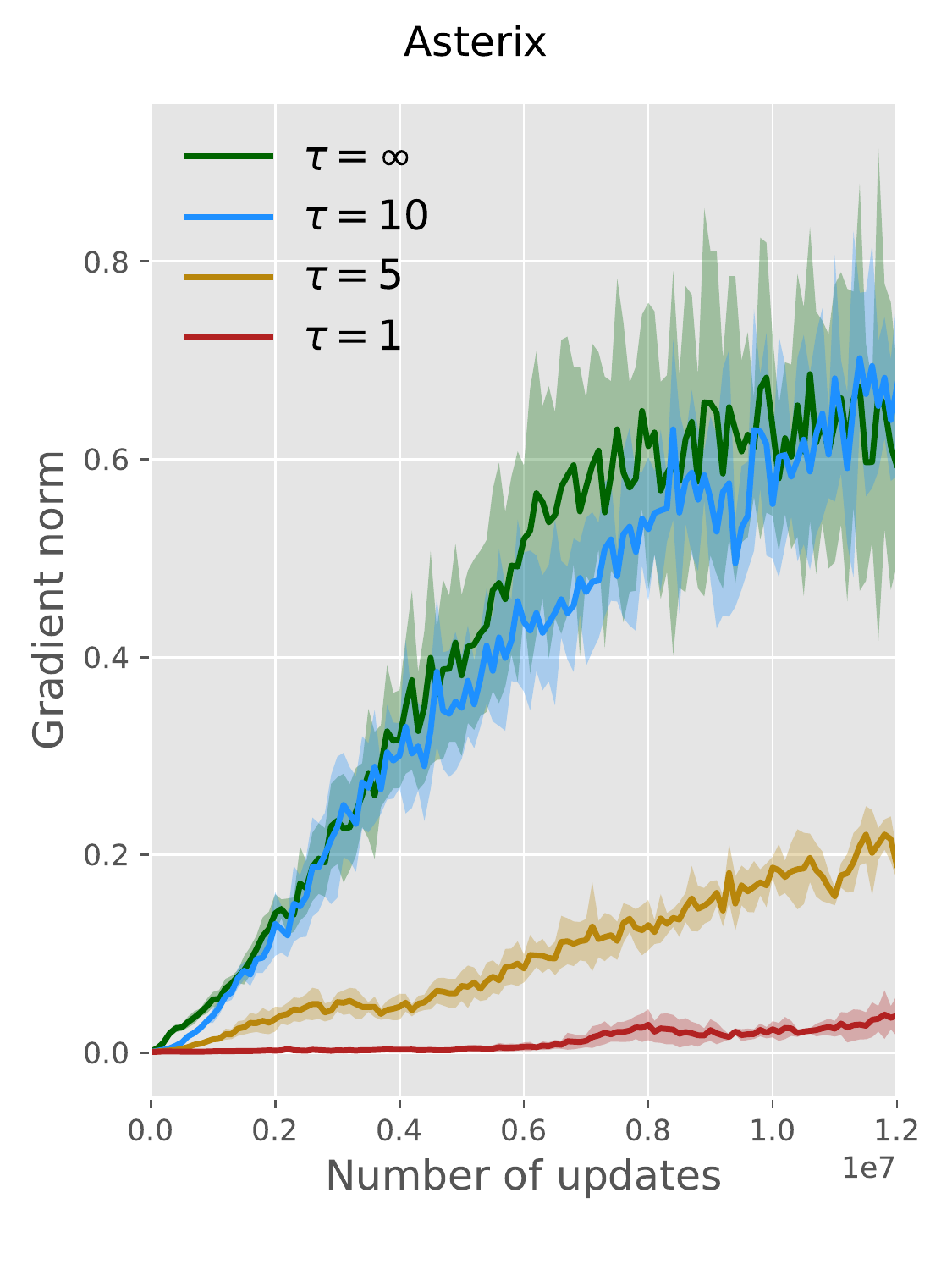}}
	\vspace{-0.1in}
	\caption{Mean and one standard deviation of the gradient norm on the Atari games,
	for different values of $\tau$ in S-DDQN.}
	\label{fig:grad_tau}
	\vspace{-0.1in}
\end{figure}

\section{A Comparison for Bellman Operators}
\label{sec:comparison}

So far we have covered three different types of Bellman operators, where the softmax and mellowmax variants employ the softmax and the log-sum-exp functions in lieu of the max function in the standard Bellman operator, respectively. A thorough comparison among these Bellman operators will not only exhibit their differences and connections, but also reveal the trade-offs when selecting them.

\begin{table}[!h]
\vspace{-0.15in}
\caption{A comparison of different Bellman operators 
(B.O. Bellman optimality; O.R. overestimation reduction; P.R. policy representation;
D.Q. double Q-learning). }
\vspace{0.1in}
\centering
\begin{tabular}{lccccc}
\toprule
& B.O. & Tuning & O.R. & P.R. & D.Q. \\ 
\midrule
Max 	   & Yes & No  &  -  & Yes & Yes\\
Mellowmax  & No  & Yes & Yes & No  & No\\
Softmax    & No  & Yes & Yes & Yes & Yes\\
\bottomrule
\end{tabular}
\label{tab:comp}
\vspace{-0.05in}
\end{table} 

Table~\ref{tab:comp} shows the comparison from different criteria: Bellman equation implies that the optimal greedy policy corresponds to the max operator only. In contrast to the mellowmax and the softmax operators, max also does not need to tune the inverse temperature parameter. On the other hand, the overestimation bias rooted in the max operator can be alleviated by either of the softmax and mellowmax operators. Furthermore, as noted in~\citet{AsadiLittman}, the mellowmax operator itself cannot directly represent a policy, and needs to be transformed into a softmax policy, where numerical methods are necessary to determine the corresponding state-dependent temperature parameters. The lack of an explicit policy representation also prevents the mellowmax operator from being directly applied in double Q-learning.  

For mellowmax and softmax, it is interesting to further investigate their relationship, especially given their comparable performance in the presence of function approximation, as shown in Table~\ref{tab:score_mellow}. First, we notice the following equivalence between these two operators, in terms of the Bellman updates: For the softmax-weighted Q-function, i.e., $g_{ \scriptscriptstyle Q (s', )} ( \tau ) = \sum_{a'} \frac{\exp [ \tau \, Q (s', a')]}{\sum_{\bar{a}} \exp [ \tau \, Q (s', \bar{a} )]} \, Q (s', a')$, there exists an $\omega$ such that the mellowmax-weighted Q-function achieves the same value. This can be verified by the facts that both softmax and mellowmax operators $(i)$ converge to the mean and max operators when the inverse temperature parameter approaches $0$ and $\infty$, respectively; $(ii)$ are continuous on $\tau$ and $\omega$, with $[0, \infty)$ as the support.\footnote{The operators remain fundamentally different since this equivalence would hold only for a particular set of Q-values. The conversion between softmax and mellomax parameters would need to be done on a state-by-state basis, and would change as the Q-values are updated.}

Furthermore, we compare via simulation the approximation error to the max function and the overestimation error, for mellowmax and softmax. The setup is as follows: The inverse temperature parameter $\tau$ is chosen from a linear grid from 0.01 to 100. We generate standard normal random variables, and compute the weighted Q-functions for the cases with different number of actions. The approximation error is then measured in terms of the difference between the max function and the corresponding softmax and mellowmax functions. The overestimation error is measured the same as in Figure~\ref{fig:bias}. The result statistics are reported by averaging over 100 independent trials. Figure~\ref{fig:comp} shows the trade-off between converging to the Bellman optimality and overestimation reduction in terms of the inverse temperature parameter, where the softmax operator approaches the max operator in a faster speed while the mellowmax operator can further reduce the overestimation error.

\begin{figure}[!t]
    \vspace{-0.15in}
	\centering
	\subfigure{\includegraphics[width=0.495\linewidth]{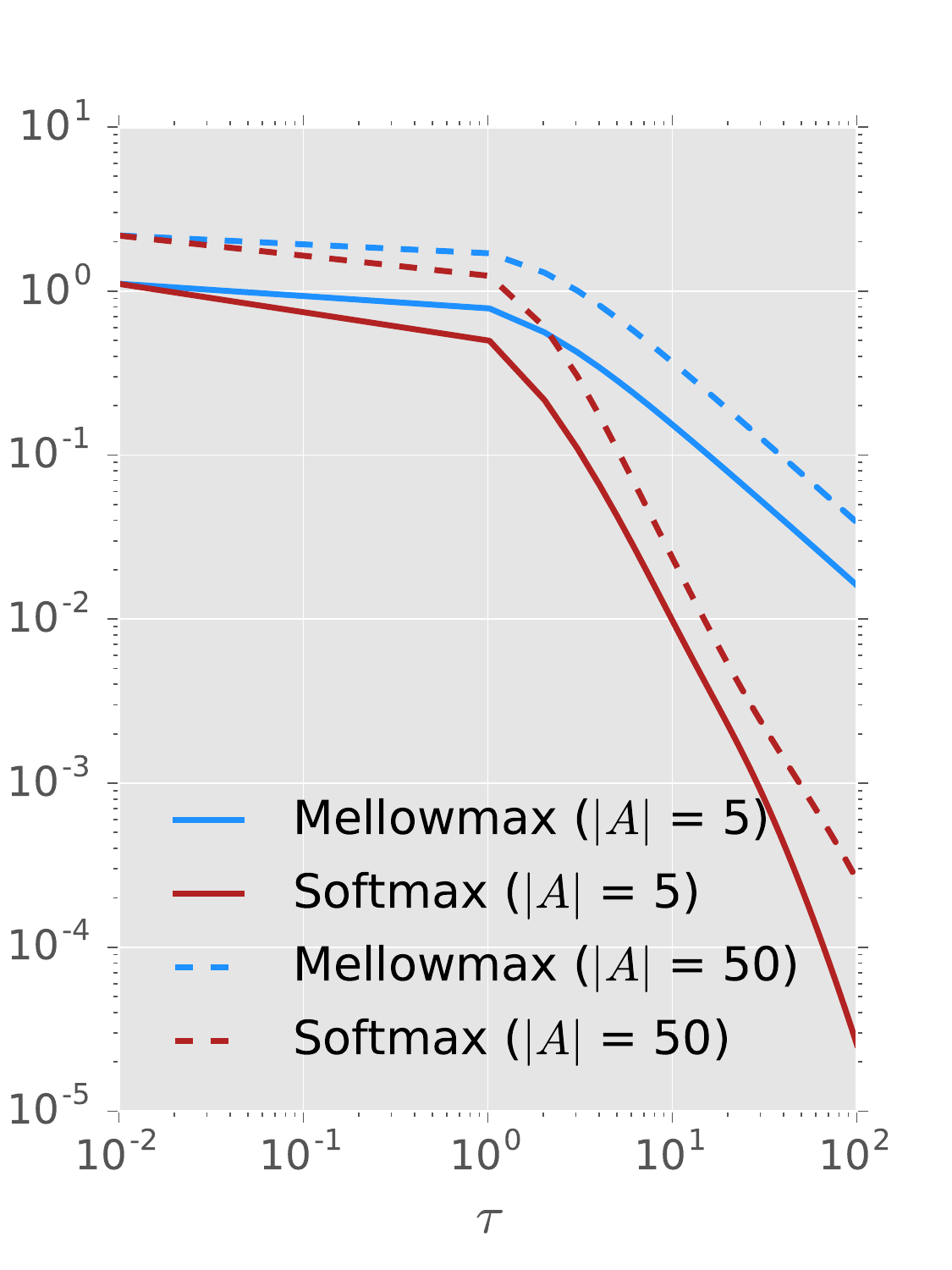}} 
	\hfill
	\subfigure{\includegraphics[width=0.495\linewidth]{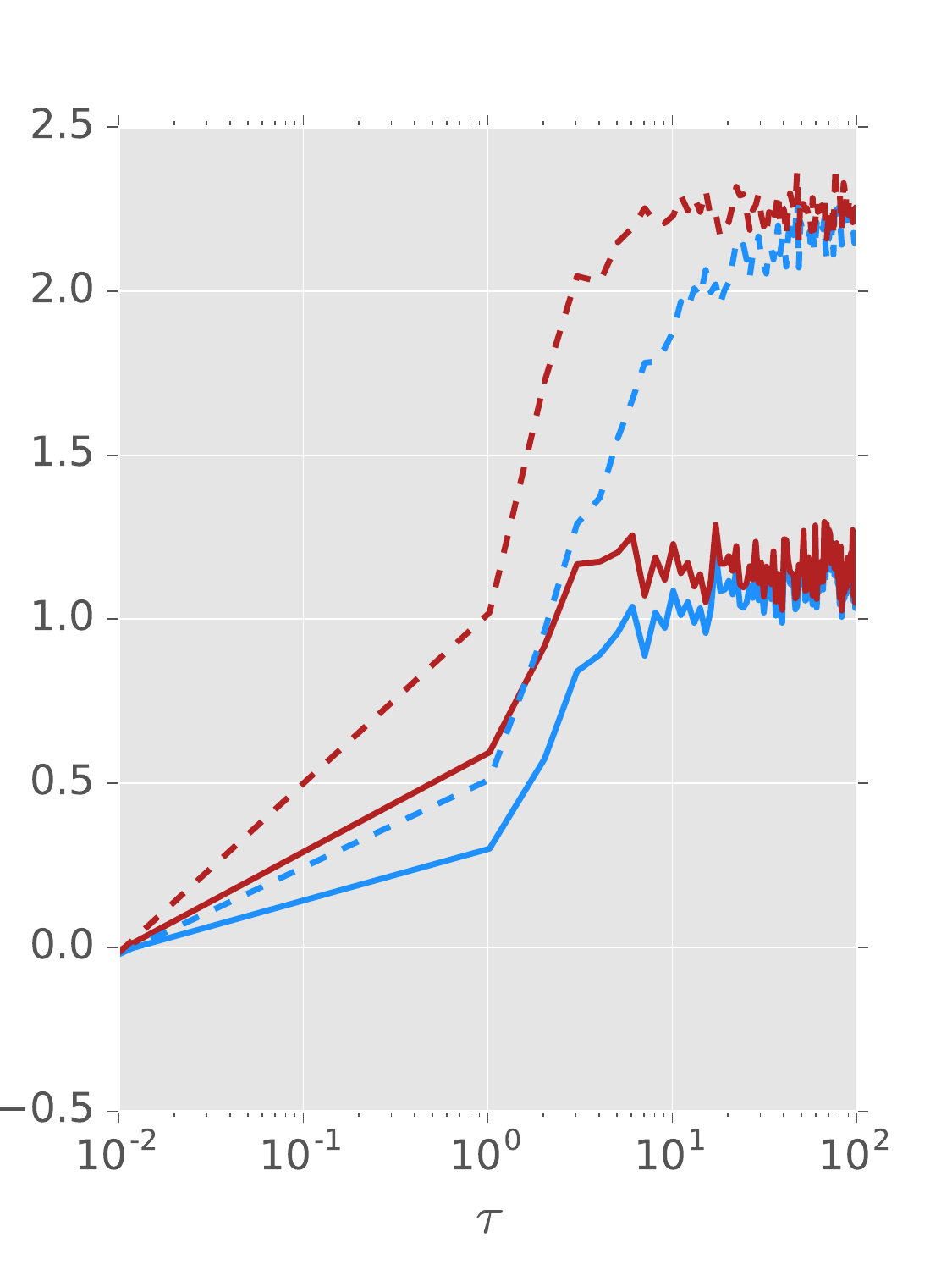}} 
	\caption{(Left) Approximation error to the max function; (Right) Overestimation error, as a function of $\tau$ (softmax) or $\omega$ (mellowmax).}
	\vspace{-0.2in}
	\label{fig:comp}
\end{figure}

\section{Related Work}
\label{sec:related}

Applying the softmax function in Bellman updates has long been viewed as problematic, as~\citet{littman} showed that the corresponding Boltzmann operator could be expansive in a specific~\gls{MDP}, via a counterexample. The primary use of the softmax function in~\gls{RL} has been focused on improving the exploration performance~\cite{sutton98}. Our work here instead aims to bring evidence and new perspective that the softmax Bellman operator can still be beneficial beyond exploration.

The mellowmax operator~\cite{AsadiLittman} was proposed based on the~\texttt{log-sum-exp} function, which can be proven as a contraction and has received recent attention due to its convenient mathematical properties. The use of the log-sum-exp function in the backup operator dates back to~\citet{todorov2007linearly}, where the control cost was regularized by a~\gls{KL} divergence on the transition probabilities. \citet{fox2016taming} proposed a G-learning scheme with soft updates to reduce the bias for the Q-value estimation, by regularizing the cost with the~\gls{KL} divergence on policies.
~\citet{AsadiLittman} further applied the log-sum-exp function to the on-policy updates, and showed that the state-dependent inverse temperature parameter can be numerically computed. More recently, the log-sum-exp function has also been used in the Boltzmann backup operator for entropy-regularized \gls{RL}~\cite{SchulmanAbbeelChen,haarnoja2017reinforcement,neu2017unified,NachumNorouziXuSchuurmans}. One interesting observation is that the optimal policies in these work admit a form of softmax, and thus can be beneficial for exploration. In~\citet{LeeChoiOh}, a sparsemax operator based on Tsallis entropy was proposed to improve the error bound for mellowmax. Note that these operators were not shown to explicitly address the instability issues in the~\gls{DQN}, and their corresponding policies were used to improve the exploration performance, which is different from our focus here.

Among variants of the~\gls{DQN} in~\citet{Mnih+al:2015},
\gls{DDQN}~\citep{HasseltGuezSilver} first identified the overestimation
bias issue caused by the max operator, and mitigated it via the 
double Q-learning algorithm~\citep{Hasselt}. \citet{AnschelBaramShimkin} later
demonstrated that averaging over the previous Q-value estimates in 
the learning target can also reduce the bias, while the analysis
for the bias reduction is restricted to a small~\gls{MDP} with 
a special structure. Furthermore, an adaptive normalization scheme
was developed in~\citet{HasseltGuezHesselMnihSilver} and was demonstrated to
reduce the gradient norm. 
The softmax function was also employed in the categorical~\gls{DQN}~\citep{BellemareDabneyMunos},
but with a different purpose of generating distributions in its distributional Bellman operator.
Finally, we notice that other variants 
\citep{WangEtAl2016,MnihA3C,SchaulQuanAntonoglouSilver,HeLiuSchwingPeng,rainbow,dabney2018distributional}
have also been proposed
to improve the vanilla~\gls{DQN}, though they were not explicitly designed
to tackle the issues of overestimation error and gradient noise.

\section{Conclusion and Future Work}
\label{sec:conclusion}

We demonstrate surprising benefits of the softmax Bellman operator when combined with~\glspl{DQN}, which further suggests that the softmax operator can be used as an alternative for the double Q-learning to reduce the overestimation bias. Our theoretical results provide new insight into the convergence properties of the softmax Bellman operator, and quantify how it reduces the overestimation bias. To gain a deep understanding about the relationship among different Bellman operators, we further compare them from different criteria, and show the trade-offs when choosing them in practice. 

An interesting direction for future work is to provide more theoretical analysis for the performance trade-off when selecting the inverse temperature parameter in~$\cTs$, and subsequently design an efficient cooling scheme.


\section*{Acknowledgements}
\label{sec:acknowledgements}
We thank the anonymous reviewers for their helpful comments and suggestions. This work was partially supported by NSF RI grant 1815300. Opinions, findings, conclusions or recommendations herein are those of the authors and not necessarily those of the NSF.

\bibliography{IEEEabrv,ref}
\bibliographystyle{icml2019}

\newpage

\twocolumn[
\begin{center}
    \bf \Large{Supplemental Material\\[1.5em]}
\end{center}
]

\renewcommand\thesection{A\arabic{section}}
\renewcommand\theequation{A\arabic{equation}}
\renewcommand\thefigure{A\arabic{figure}}
\renewcommand\thelemma{A\arabic{lemma}}
\renewcommand\thetheorem{A\arabic{theorem}}
\renewcommand\thetable{A\arabic{table}}
\setcounter{equation}{0}
\setcounter{section}{0}
\setcounter{figure}{0}
\setcounter{theorem}{0}
\setcounter{table}{0}

\section{Proof for Performance Bound}

We first show that for all Q-functions that occur during Q-iteration with $\cTs$, their corresponding Q-values are bounded.
\begin{lemma}
Assuming $\forall (s, a)$, the initial Q-values $Q_0 (s, a) \in [R_{\text{min}}, R_{\text{max}}]$, the Q-values during Q-iteration with $\cTs$ are within $[Q_{\text{min}}, Q_{\text{max}}]$, with $Q_{\text{min}} = \frac{R_{\text{min}}}{1 - \gamma}$ and $Q_{\text{max}} = \frac{R_{\text{max}}}{1 - \gamma}$.
\label{lemma:qbound}
\end{lemma}

\begin{proof}
    The upper bound can be obtained by showing $\forall(s, a)$, the Q-values at the i$th$ iteration are bounded as
    \begin{equation}
        Q_i (s, a) \leq \sum_{j=0}^{i} \gamma^{j} R_{\text{max}}.
        \label{eq:qmax_bound}
    \end{equation}
    We then prove~\eqnref{eq:qmax_bound} by induction as follows. The lower bound can be proven similarly.
    
    $(i)$ When $i=1$, we start from the definition of $\cTs$ in~\eqnref{eq:soft} and the assumption of $Q_0$ to have
    \begin{align*}
        Q_1 (s, a) &= \cTs \, Q_0 (s, a)
        \\
        &\leq  R_{\text{max}} + \gamma \sum_{s'} P (s' | s, a) \, \max_{a'} Q_{0} (s', a')
        \\
        &\leq R_{\text{max}} + \gamma \sum_{s'} P (s' | s, a) \, R_{\text{max}}
        \\
        &= (1 + \gamma) R_{\text{max}}.
    \end{align*}
    \vspace{-0.1in}
    
    $(ii)$ Assuming~\eqnref{eq:qmax_bound} holds when $i=k$, i.e., $Q_k (s, a) \leq \sum_{j=0}^{k} \gamma^{j} R_{\text{max}}$. Then,
    \begin{align*}
        Q_{k+1} (s, a) &= \cTs Q_k (s, a)
        \\
        &\leq  R_{\text{max}} + \gamma \sum_{s'} P (s' | s, a) \, \max_{a'} Q_{k} (s', a')
        \\
        &\leq R_{\text{max}} + \gamma \sum_{s'} P (s' | s, a) \, \sum_{j=0}^{k} \gamma^{j} R_{\text{max}}
        \\
        &= \sum_{j=0}^{k+1} \gamma^{j} R_{\text{max}}.
    \end{align*}
\end{proof}

\begin{corollary}
Assuming $R_{\text{max}} \geq -R_{\text{min}} \geq 0$ WLOG, we have $| Q(s, a_i) - Q(s, a_j) | \leq 2 \frac{R_{\text{max}}}{1 - \gamma}, \forall Q$ and $\forall s$.
\label{coro:diffbound}
\end{corollary}
\begin{proof}
    This follows by using the assumption and the results in Lemma~\ref{lemma:qbound}.
\end{proof}

\begin{proof}[Proof of Lemma~\ref{lemma:diff_bound}]
    We first sort the sequence $\{Q(s, a_i)\}$ such that $Q (s, a_{\small[1]} ) \geq \ldots \geq Q (s, a_{[m]} )$.
	Then, $\forall Q$ and $\forall s$, we have
	\begin{align}
		&\max_a Q (s, a) - f_{\tau}^T \big( Q(s, ) \big) \, Q(s, ) 
		\notag
		\\
		&= Q (s, a_{[1]}) - \frac{\sum_{i=1}^m \exp \big[ \tau Q (s, a_{[i]}) \big] \, 
		Q (s, a_{[i]}) }{ \sum_{i=1}^{m} \exp \big[ \tau Q (s, a_{[i]}) \big]}
		\notag
		\\
		&= \frac{\sum_{i=1}^m \exp \big[ \tau Q (s, a_{[i]}) \big] \, 
		\big[ Q (s, a_{[1]}) - Q (s, a_{[i]}) \big] }{ \sum_{i=1}^{m} \exp 
		\big[ \tau Q (s, a_{[i]}) \big]}.
		\label{eq:intro_delta}	
	\end{align}
	By introducing $\delta_i (s) = Q (s, a_{[1]}) - Q (s, a_{[i]})$, and noting
	$\delta_i (s) \geq 0$ and $\delta_1 (s) = 0$, we can proceed from~\eqnref{eq:intro_delta}
	as
	\begin{align}
		& \quad \frac{\sum_{i=1}^m \exp \big[ \tau Q (s, a_{[i]}) \big] \, 
		\big[ Q (s, a_{[1]}) - Q (s, a_{[i]}) \big] }{ \sum_{i=1}^{m} \exp 
		\big[ \tau Q (s, a_{[i]}) \big]}
		\notag
		\\
		&= \frac{\sum_{i=1}^m \exp [- \tau \delta_i (s) ] \, 
		\delta_i (s) }{ \sum_{i=1}^{m} \exp 
		[ - \tau \delta_i (s) ]}
		\notag
		\\
		&= \frac{\sum_{i=2}^m \exp [- \tau \delta_i (s) ] \, 
		\delta_i (s)  }{ 1 + \sum_{i=2}^{m} \exp 
		[- \tau \delta_i (s) ]}.
		\label{eq:qdiff}
	\end{align}
	Now, we can proceed from~\eqnref{eq:qdiff} to prove each direction separately as follows.
    
    \begin{enumerate}
    
    \item[$(i)$] Upper bound: First note that for any two non-negative sequences $\{x_i\}$ and $\{y_i\}$,
    \begin{equation}
        \frac{\sum_i x_i}{1 + \sum_i y_i} \leq \sum_i \frac{x_i}{1 + y_i}.
        \label{eq:seq_div}
    \end{equation}
    We then apply~\eqnref{eq:seq_div} to~\eqnref{eq:qdiff} as
	\begin{align}
	    \frac{\sum_{i=2}^m \exp [ - \tau \delta_i (s) ] \, 
		\delta_i (s)  }{ 1 + \sum_{i=2}^{m} \exp 
		[ - \tau \delta_i (s) ]}
		&\leq \sum_{i=2}^{m} \frac{\exp [ - \tau \delta_i (s) ] \, \delta_i (s)}{
		1 + \exp [ - \tau \delta_i (s) ]} 
		\notag
		\\
		&= \sum_{i=2}^{m} \frac{ 
		\delta_i (s)}{ 1 + \exp [ \tau \delta_i (s) ]}.
		\label{eq:delta_sum}
	\end{align}
	Next, we bound each term in~\eqnref{eq:delta_sum}, by considering the
	following two cases:
	
	1) $\delta_i (s) > 1$: $\frac{ \delta_i (s)}{ 1 + \exp [ \tau \delta_i (s) ]} \leq 
	\frac{ \delta_i (s)}{ 1 + \exp ( \tau )} \leq \frac{ 2 Q_{\text{max}} }{ 
	1 + \exp ( \tau )}$, where we apply Corollary~\ref{coro:diffbound} to bound $\delta_i (s)$.
	
	2) $0 \leq \delta_i (s) \leq 1$: $\frac{ \delta_i (s)}{ 1 + \exp [ \tau \delta_i (s) ]} 
	= \frac{1}{ \frac{2}{\delta_i (s)} + \tau + 0.5 \tau^2 \delta_i (s) + \cdots } 
	\leq \frac{1}{\tau+2}$, where we first expand the denominator using 
	Taylor series for the exponential function.	
	
	By combining these two cases with~\eqnref{eq:delta_sum}, we achieve the upper bound.
	
	\item[$(ii)$] Lower bound:
	\begin{align}
	    & \quad \frac{\sum_{i=2}^m \exp [ - \tau \delta_i (s) ] \, 
		\delta_i (s)  }{ 1 + \sum_{i=2}^{m} \exp 
		[ - \tau \delta_i (s) ]} 
		\notag
		\\
		&\geq \frac{\sum_{i=2}^m \exp [ - \tau \delta_i (s) ] \, 
		\delta_i (s)}{m}
		\notag
		\\
		& \geq \frac{\sum_{i=2}^m \delta_i (s)}{m \exp [\tau \,\widehat{\delta} (s) ] }
		\notag
		\\
		& \geq \frac{\widehat{\delta} (s)}{m \exp [\tau \,\widehat{\delta} (s) ] }.
	\end{align}
	
	\end{enumerate}
	
\end{proof}

\begin{proof}[Proof of~\thmref{thm:convergebound}]

    We first prove the upper bound by induction as follows.
    
    $(i)$ When $i = 1$, we start from the definitions for $\cT$ and $\cTs$ in~\eqnref{eq:bellman4q}
	and~\eqnref{eq:soft}, and proceed as
	\begin{align*}
		& \quad \cT Q_0 (s, a) - \cTs \, Q_0 (s, a) 
		\\
		= & \gamma \sum_{s'} P (s' | s, a) \, \big[ \max_{a'} Q_0 (s', a') - 
		f_{\tau}^T \big( Q_0 (s', ) \big) Q_0 (s', ) \big] 
		\\
		\geq & 0 .
	\end{align*}
	
	$(ii)$ Suppose this claim holds when $i = l$, i.e., $\cT^{l} Q_0 (s, a) \geq \cTs^{l} \, Q_0 (s, a)$. When $i = l+1$, we have
	\begin{align*}
	    & \cT^{l+1} Q_0 (s, a) - \cTs^{l+1} \, Q_0 (s, a) 
	    \\
	    = & \cT \cT^{l} Q_0 (s, a) - \cTs \cTs^{l} \, Q_0 (s, a)
	    \\
	    \geq & \cT \cTs^{l} \, Q_0 (s, a) - \cTs \cTs^{l} \, Q_0 (s, a)
	    \\
	    \geq & 0.
	\end{align*}
	Since $Q^{*}$ is the fixed point for $\cT$, we know $\lim_{k \rightarrow \infty} \cT^{k} \, Q_0 (s, a) = Q^{*} (s, a)$. Therefore, $\limsup \limits_{k \rightarrow \infty} \cTs^{k} \, Q_0 (s, a) \leq Q^{*} (s, a)$.
	
	To prove the lower bound, we first conjecture that 
	\begin{equation}
		\cT^{k} \, Q_0 (s, a) - \cTs^{k} \, Q_0 (s, a) \,\leq\, \sum_{j=1}^k \gamma^{j} \, \zeta,
		\label{eq:conjecture_zeta}
	\end{equation}
	where $\zeta = \sup_{Q} \max_{s} [ \max_a Q (s, a) - f_{\tau}^T \big( Q(s, ) \big) \, Q(s, ) ]$ denotes
	the supremum of the difference between the max and softmax operators, over all
	Q-functions that occur during Q-iteration, and state $s$. \eqnref{eq:conjecture_zeta} is proven using induction as follows.
	
	$(i)$ When $i = 1$, we start from the definitions for $\cT$ and $\cTs$ in~\eqnref{eq:bellman4q}
	and~\eqnref{eq:soft}, and proceed as
	\begin{align*}
		& \quad \cT Q_0 (s, a) - \cTs \, Q_0 (s, a) 
		\\
		= & \gamma \sum_{s'} P (s' | s, a) \, \big[ \max_{a'} Q_0 (s', a') - 
		f_{\tau}^T \big( Q_0 (s', ) \big) Q_0 (s', ) \big] 
		\\
		\leq & \gamma \sum_{s'} P (s' | s, a) \, \zeta = \gamma \zeta.
	\end{align*}
	
	$(ii)$ Suppose the conjecture holds when $i = l$, i.e., $\cT^{l} Q_0 (s, a) - \cTs^{l} \, Q_0 (s, a) \leq 
	\sum_{j=1}^l \gamma^{j} \zeta$, then
	\begin{align*}
		& \cT^{l+1} Q_0 (s, a) - \cTs^{l+1} \, Q_0 (s, a) 
		\\
		= &  \cT \, \cT^{l} Q_0 (s, a) - \cTs^{l+1} \, Q_0 (s, a)
		\\
		\leq &  \cT \,\big[\, \cTs^{l} \, Q_0 (s, a) + 
		\sum_{j=1}^l \gamma^{j} \zeta \,\big]\, - \cTs^{l+1} \, Q_0 (s, a)
		\\
		= & \sum_{j=1}^l \gamma^{j+1} \zeta + (\cT - \cTs) \, \cTs^{l} \, Q_0 (s, a)
		\\
		\leq & \sum_{j=1}^l \gamma^{j+1} \zeta + \gamma \zeta = \sum_{j=1}^{l+1} \gamma^{j} \zeta,
	\end{align*}
	where the last inequality follows from the definition of $\zeta$. 
	By using the fact that $\lim_{k \rightarrow \infty} \cT^{k} \, Q_0 (s, a) = Q^{*} (s, a)$
	again and applying Lemma~\ref{lemma:diff_bound} to bound $\zeta$, we finish the proof for Part $(I)$. 
	
	To prove part $(II)$, note that as a byproduct of~\eqnref{eq:delta_sum} in the proof of Lemma~\ref{lemma:diff_bound}, \eqnref{eq:conjecture_zeta} can be bounded as
	\begin{align}
		\cT^{k} \, Q_0 (s, a) - & \cTs^{k} \, Q_0 (s, a)
		\leq 
		\notag
		\\
		&\frac{\gamma (1 - \gamma^k) }{1 - \gamma}
		\sum_{i=2}^{m} \frac{ \delta_i (s)}{ 1 + \exp [ \tau \delta_i (s)]}.
		\label{eq:upper_bound_full}
	\end{align}
	From the definition of $\delta_i (s)$, we know $\delta_{m} (s) \geq \delta_{m-1} (s) \geq \ldots \geq \delta_{2} (s) \geq 0$. Furthermore, there must exist an index $i^{*} \leq m$ such that $\delta_{i} > 0, \forall i^{*} \leq i \leq m$ (otherwise the upper bound becomes zero). Subsequently, we can proceed from ~\eqnref{eq:upper_bound_full} as
	\begin{align*}
	    &\quad \frac{\gamma (1 - \gamma^k)}{1 - \gamma}
		\sum_{i=2}^{m} \frac{ \delta_i (s)}{ 1 + \exp [ \tau \delta_i (s) ]} 
		\\
		&= \frac{\gamma (1 - \gamma^k)}{1 - \gamma}
		\sum_{i=i^{*}}^{m} \frac{ \delta_i (s)}{ 1 + \exp [ \tau \delta_i (s) ]}
		\\
		& \leq \frac{\gamma (1 - \gamma^k)}{1 - \gamma}
		\sum_{i=i^{*}}^{m} \frac{ \delta_i (s)}{\exp [ \tau \delta_i (s) ]}
		\\
		& \leq \frac{\gamma (1 - \gamma^k)}{1 - \gamma}
		\sum_{i=i^{*}}^{m} \frac{ \delta_i}{\exp [ \tau \delta_{i^{*} } (s) ]}
		\\
		& = \frac{\gamma (1 - \gamma^k)}{1 - \gamma}
		 \exp [ - \tau \delta_{i^{*} } (s) ] \sum_{i=i^{*}}^{m} \delta_i (s),
	\end{align*}
	which implies an exponential convergence rate in terms of $\tau$ and hence
	proves part $(II)$.
\end{proof}

\section{Proofs for Overestimation Reduction}

\begin{lemma}
$g_{\bx} (\tau) = \frac{\sum_{i=1}^m [ \exp (\tau x_i) x_i ] }{\sum_{i=1}^m \exp (\tau x_i) }$ is a monotonically increasing function for $\tau \in [0, \infty)$.
\label{lemma:mono}
\end{lemma}

\begin{proof}
	The gradient of $g_{\bx} (\tau)$ can be computed as
	\begin{align*}
		&\frac{ \partial g_{\bx} (\tau) }{\partial \tau} = \Big\{ \big[ \sum_{i=1}^{m} \exp (\tau x_i) \, x_i^2 \big]
		\big[ \sum_{i=1}^{m} \exp (\tau x_i)\big] -
		\\
		&\big[\sum_{i=1}^{m} \exp (\tau x_i) \, x_i \big]^2 \Big\} \,\Big/\,
		\big[ \sum_{j=1}^{m} \exp ( \tau x_j ) \big]^2 \geq 0,
	\end{align*}
	where the last step holds because of the Cauchy-Schwarz inequality. 
\end{proof} 

The overestimation bias due to the max operator can be observed by plugging
assumption $(A2)$ in Theorem~\ref{thm:bias} into~\eqnref{eq:bellman4q} as
\begin{align*}
	&\mathbb{E} \big[ \max_a \big(Q_t (s, a) \big) - \max_a \big( Q_{*} (s, a) \big) \big] 
	\\
	= & \mathbb{E} \big[ \max_a \big(Q_t (s, a) - V_{*} (s) \big) \big] 
	\\
	= & \mathbb{E} \big[ \max_a (\epsilon_a) \big],
\end{align*}
and $ \max_a (\epsilon_a)$ is typically positive for a large action set
and the noise satisfying a normal distribution, or a uniform distribution
with the symmetric support.

\begin{proof}[Proof of Theorem~\ref{thm:bias} ]
	First, the overestimation error from $\cTs$ can be represented as
	\begin{align}
		& \quad \; \mathbb{E} \bigg\{ \sum_a \frac{\exp[\tau Q_t (s, a)] }{ \sum_{\bar{a}} 
		\exp[\tau Q_t (s, \bar{a} )] } Q_t (s, a ) - V^{*} (s) \bigg\}
		\notag
		\\
		&= \mathbb{E} \bigg\{ \sum_a \frac{\exp[\tau V^{*} (s) + \tau \epsilon_a] }{ 
		\sum_{\bar{a}} \exp[\tau V^{*} (s) + \tau \epsilon_{\bar{a} } )] } \,
		[V^{*} (s) + \epsilon_a ] -  V^{*} (s) \bigg\}
		\notag
		\\
		&= \mathbb{E} \bigg\{ \sum_a \frac{\exp[\tau \epsilon_a] }{\sum_{\bar{a}} \exp[\tau \epsilon_{\bar{a} }] }
		\epsilon_a \bigg\}
		\label{eq:soft_error}
		\\
		&\leq \mathbb{E} \big[ \max_{a} (\epsilon_a) \big].
		\notag
	\end{align}
	
	To prove Part $(II)$, note that the overestimation reduction of $\cTs$ from $\cT$ can then be represented as
	\begin{align*}
	    & \mathbb{E} \big[ \max_{a} (\epsilon_a) - \sum_a \frac{\exp[\tau \epsilon_a] }{\sum_{\bar{a}} \exp[\tau \epsilon_{\bar{a} }] }
		\epsilon_a \big] 
		\\
		=& \mathbb{E} \bigg\{ \max_{a} \big[ \epsilon_a + V^{*} (s) \big] - \sum_a \frac{\exp[\tau \epsilon_a] }{\sum_{\bar{a}} \exp[\tau \epsilon_{\bar{a} }] }
		\big[ \epsilon_a + V^{*} (s) \big]  \bigg\} 
		\\
		=& \mathbb{E} \bigg\{ \max_{a} \big[ Q_t (s, a ) \big] - \sum_a \frac{\exp[\tau \epsilon_a] }{\sum_{\bar{a}} \exp[\tau \epsilon_{\bar{a} }] }
		\big[ Q_t (s, a ) \big]  \bigg\} 
		\\
		=& \mathbb{E} \bigg\{ \max_{a} \big[ Q_t (s, a ) \big] - \sum_a \frac{\exp[\tau \epsilon_a + \tau V^{*} (s)] }{\sum_{\bar{a}} \exp[\tau \epsilon_{\bar{a} } + \tau V^{*} (s) ] }
		\\
		& \times \big[ Q_t (s, a ) \big]  \bigg\} 
		\\
		=& \mathbb{E} \bigg\{ \max_{a} \big[ Q_t (s, a ) \big] - \sum_a \frac{\exp[\tau Q_t (s, a )] }{\sum_{\bar{a}} \exp[\tau Q_t (s, \bar{a} ) ] }
		\big[ Q_t (s, a ) \big]  \bigg\}. 
	\end{align*}
	Subsequently, we can employ Lemma~\ref{lemma:diff_bound} to obtain the range.
	
	Finally, the monotonicity for the overestimation error in terms of $\tau$ follows, by noting the term inside the expectation of~\eqnref{eq:soft_error} can be 
	represented as $g_{\epsilon } (\tau)$, which is a monotonic function of $\tau$, according to
	Lemma~\ref{lemma:mono}.
\end{proof}

\section{Additional Plots and Setups}

Figures~\ref{fig:qvalue_supp} and~\ref{fig:grad_supp} are the
full version of the corresponding figures in the main text, by plotting all six games. The corresponding values for $\tau$ in S-DQN and S-DDQN are provided in Table~\ref{tab:tau_value}.

\begin{table}[!ht]
\vspace{-0.0in}
\caption{Values of $\tau$ used for S-DQN and S-DDQN in Figures~\ref{fig:qvalue_supp} and~\ref{fig:grad_supp}. }
\vspace{0.1in}
\centering
\begin{tabular}{lcccccc}
\toprule
& Q & M & B & C & A & S \\ 
\midrule
S-DQN  & 1 & 1 & 5 & 1 & 5 & 5\\
S-DDQN & 1 & 5 & 5 & 5 & 5 & 10\\
\bottomrule
\end{tabular}
\label{tab:tau_value}
\vspace{-0.0in}
\end{table} 

Figures~\ref{fig:score_tau_supp}, ~\ref{fig:qvalue_tau_supp}, and~\ref{fig:grad_tau_supp} show the scores, Q-values, and gradient
norm, for different values of $\tau$, for S-DDQN.

\begin{figure}[!h]
	\centering
	\subfigure{\includegraphics[width=0.495\linewidth]{./fig/QBert_qvalue_2col.pdf}} 
	\subfigure{\includegraphics[width=0.495\linewidth]{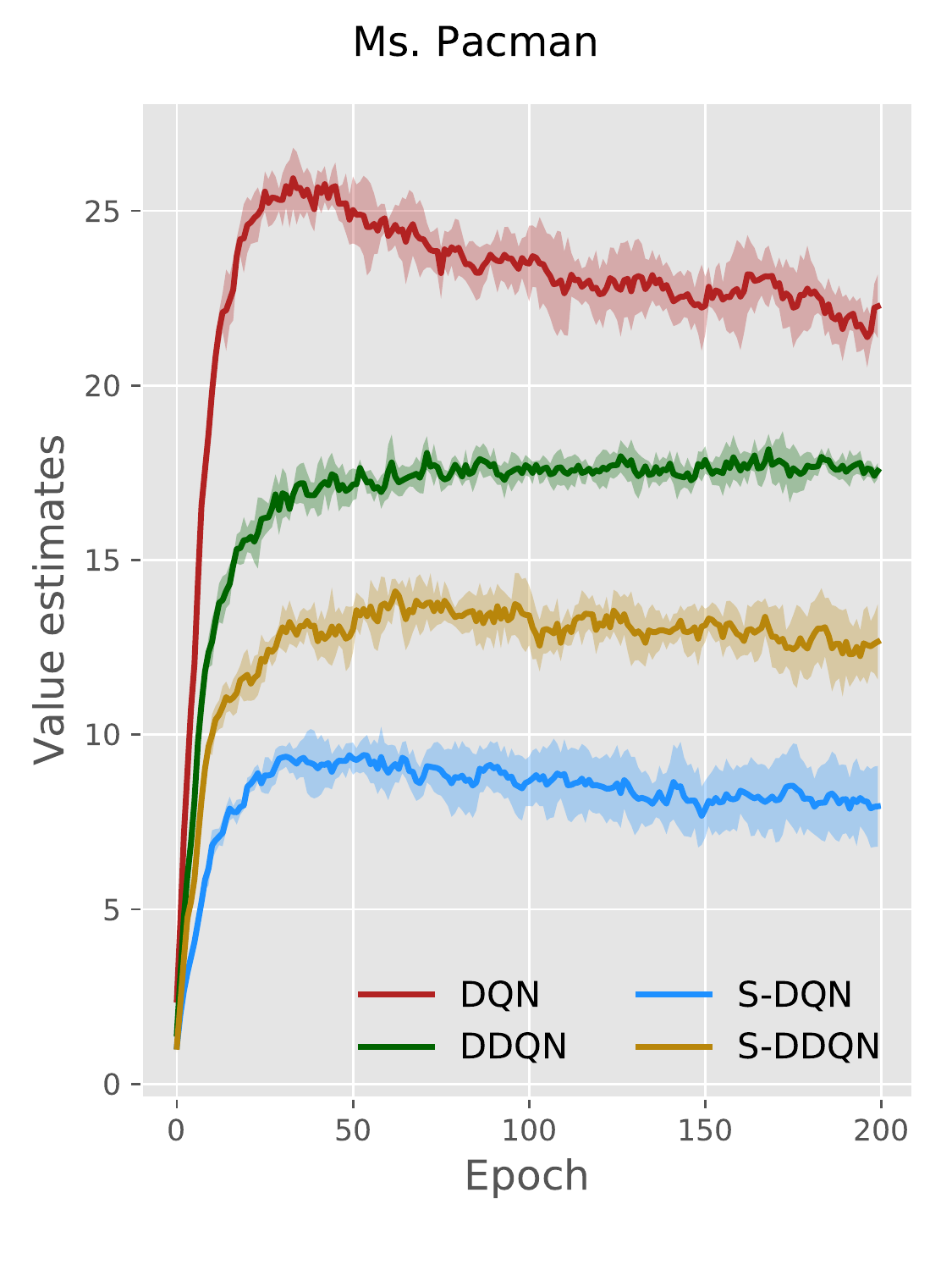}}
	\subfigure{\includegraphics[width=0.495\linewidth]{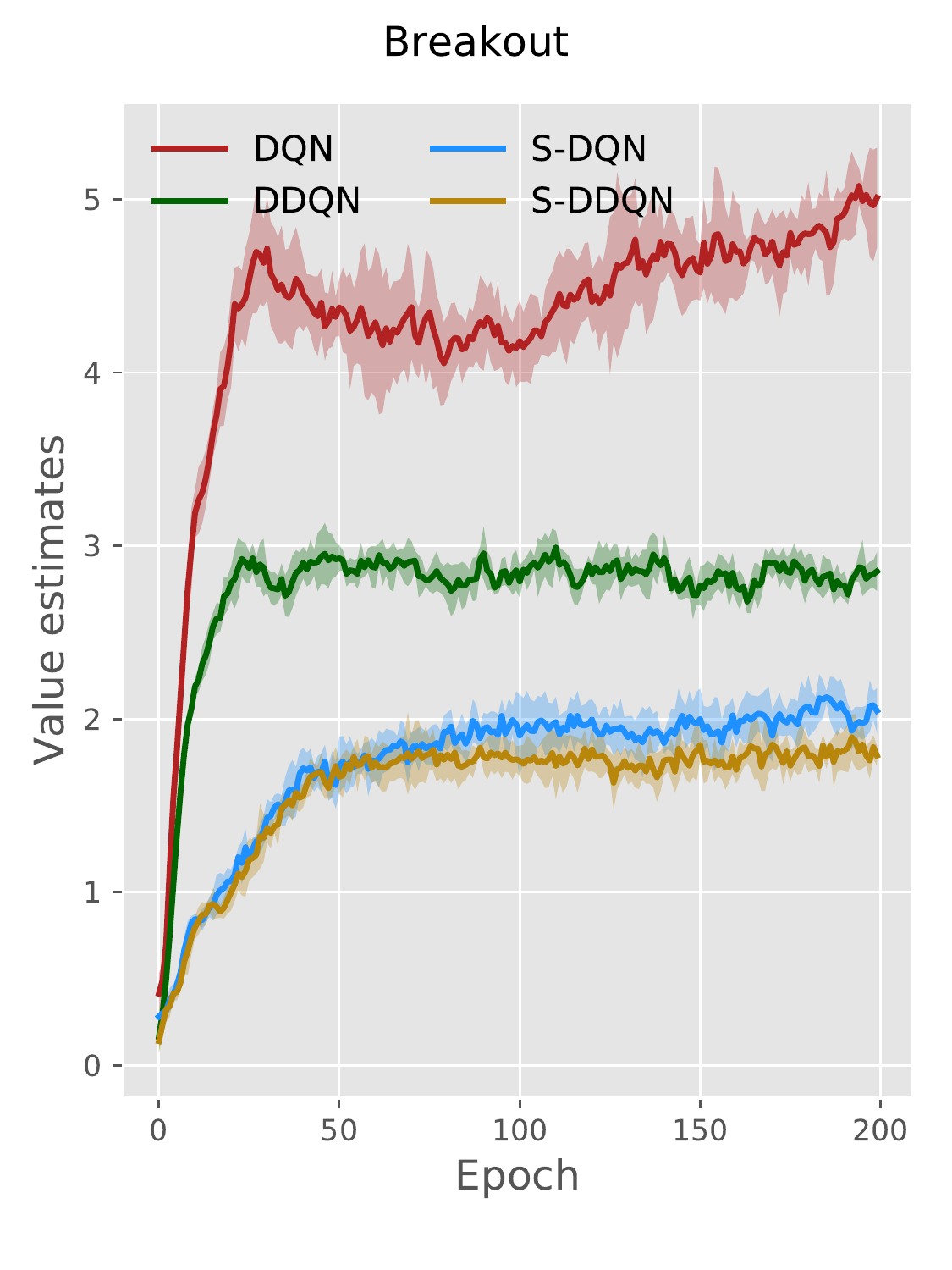}}
	\subfigure{\includegraphics[width=0.495\linewidth]{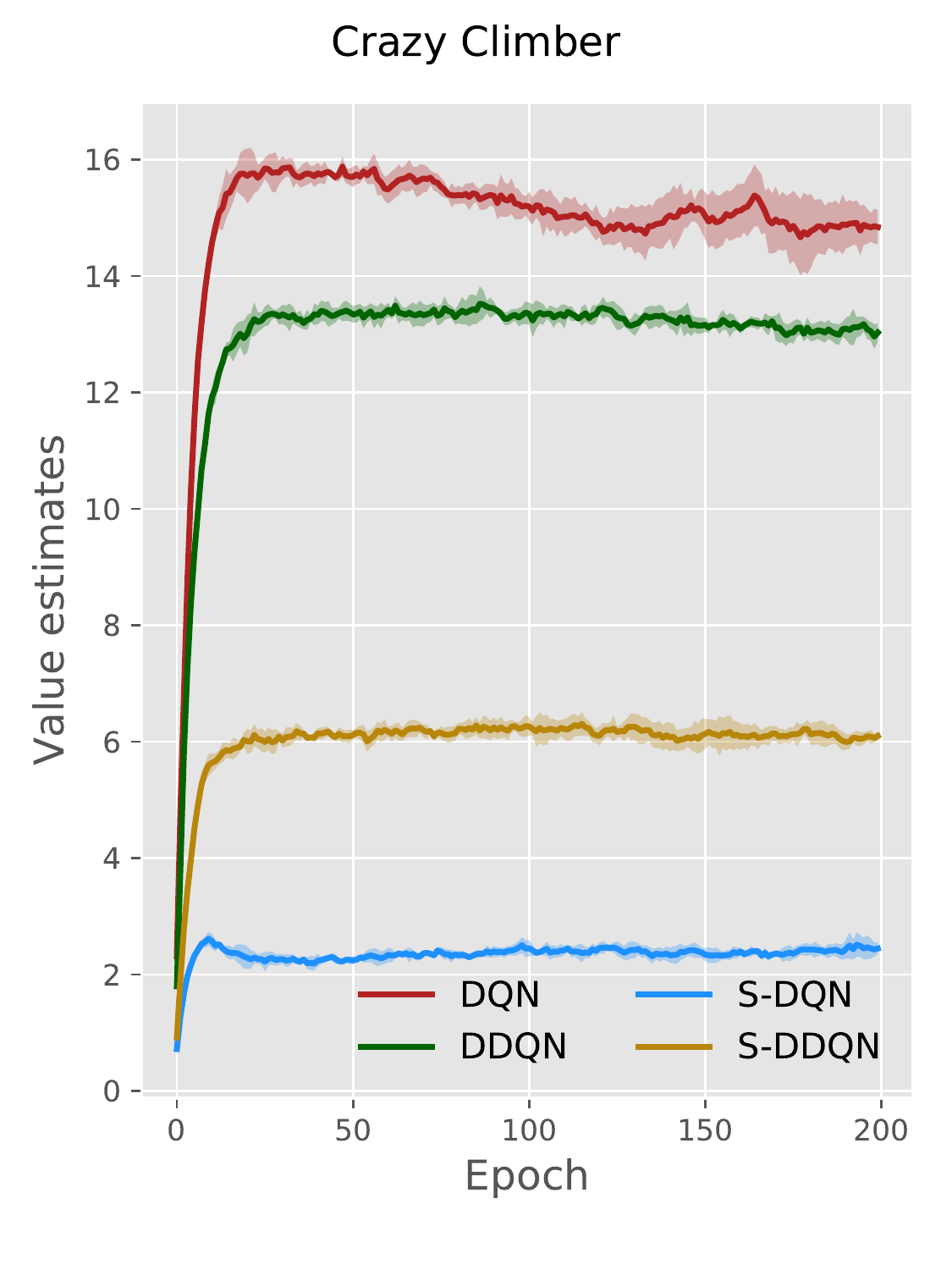}}
	\subfigure{\includegraphics[width=0.495\linewidth]{./fig/Asterix_qvalue_best.pdf}} 
	\subfigure{\includegraphics[width=0.495\linewidth]{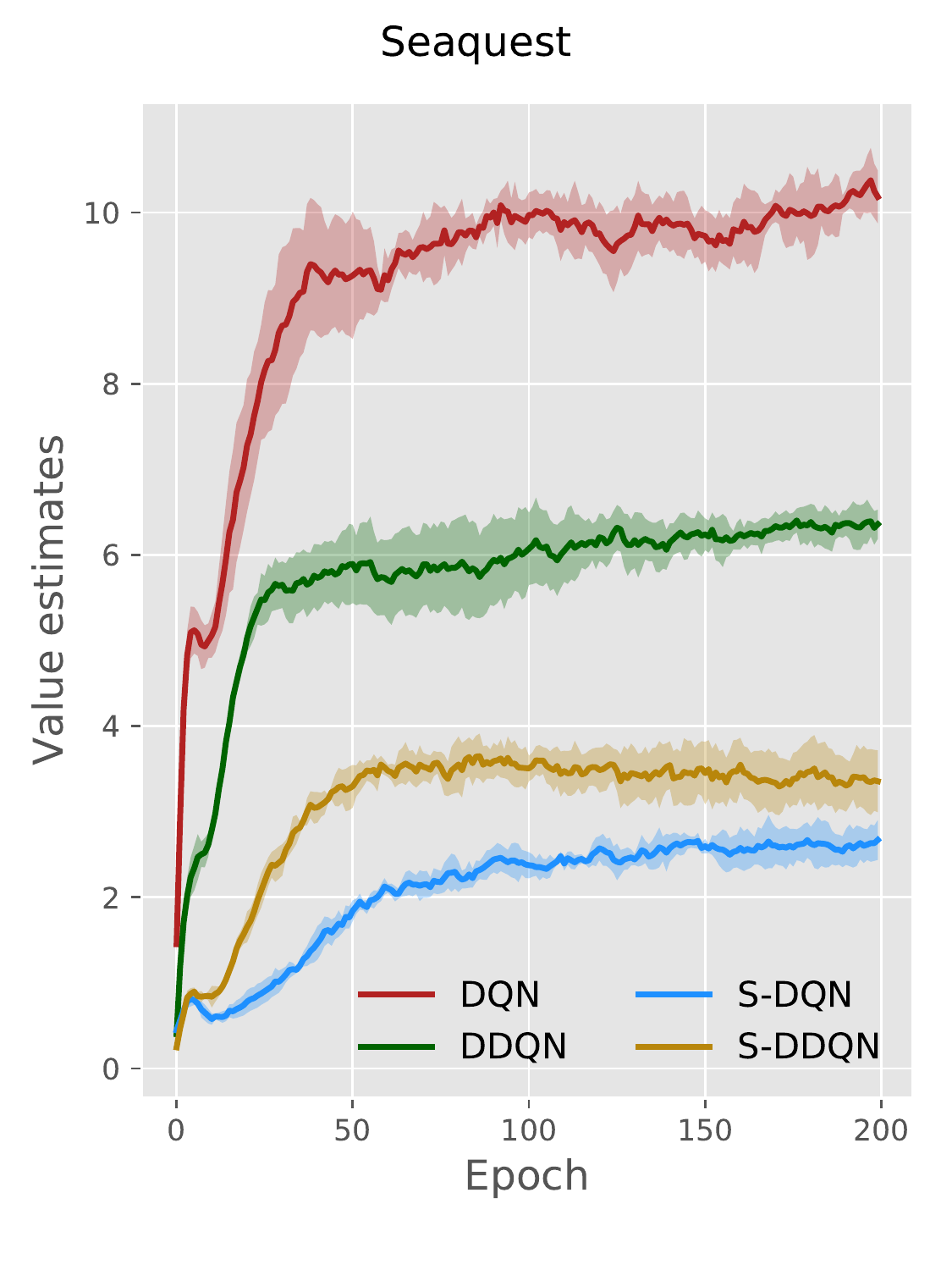}} 	
	\caption{Mean and one standard deviation of the estimated Q-values on the Atari games,
	for different methods.}
	\label{fig:qvalue_supp}
\end{figure}

\begin{figure}[!h]
	\centering
	\subfigure{\includegraphics[width=0.495\linewidth]{./fig/QBert_res_grad.pdf}} 
	\subfigure{\includegraphics[width=0.495\linewidth]{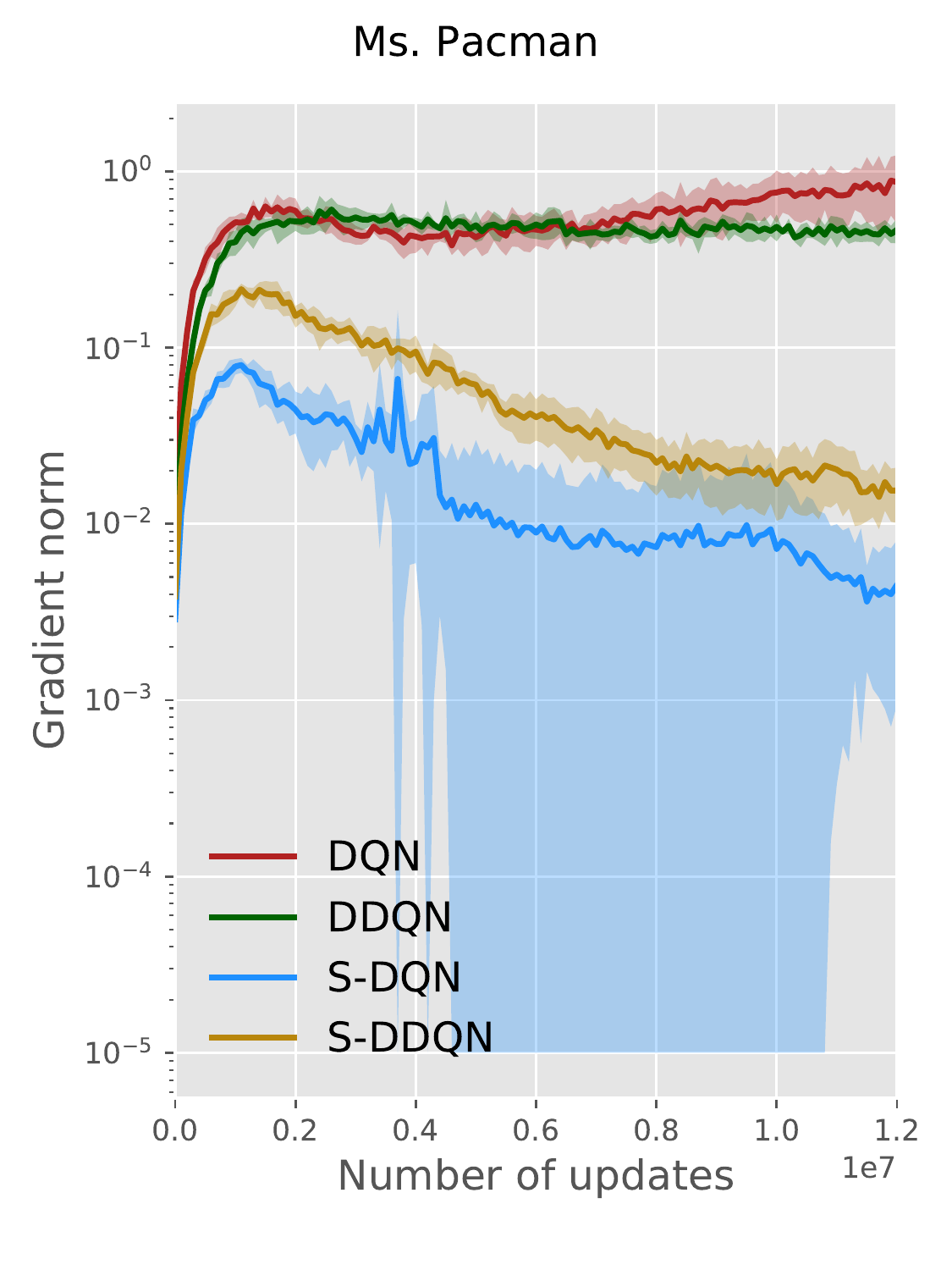}}
	\subfigure{\includegraphics[width=0.495\linewidth]{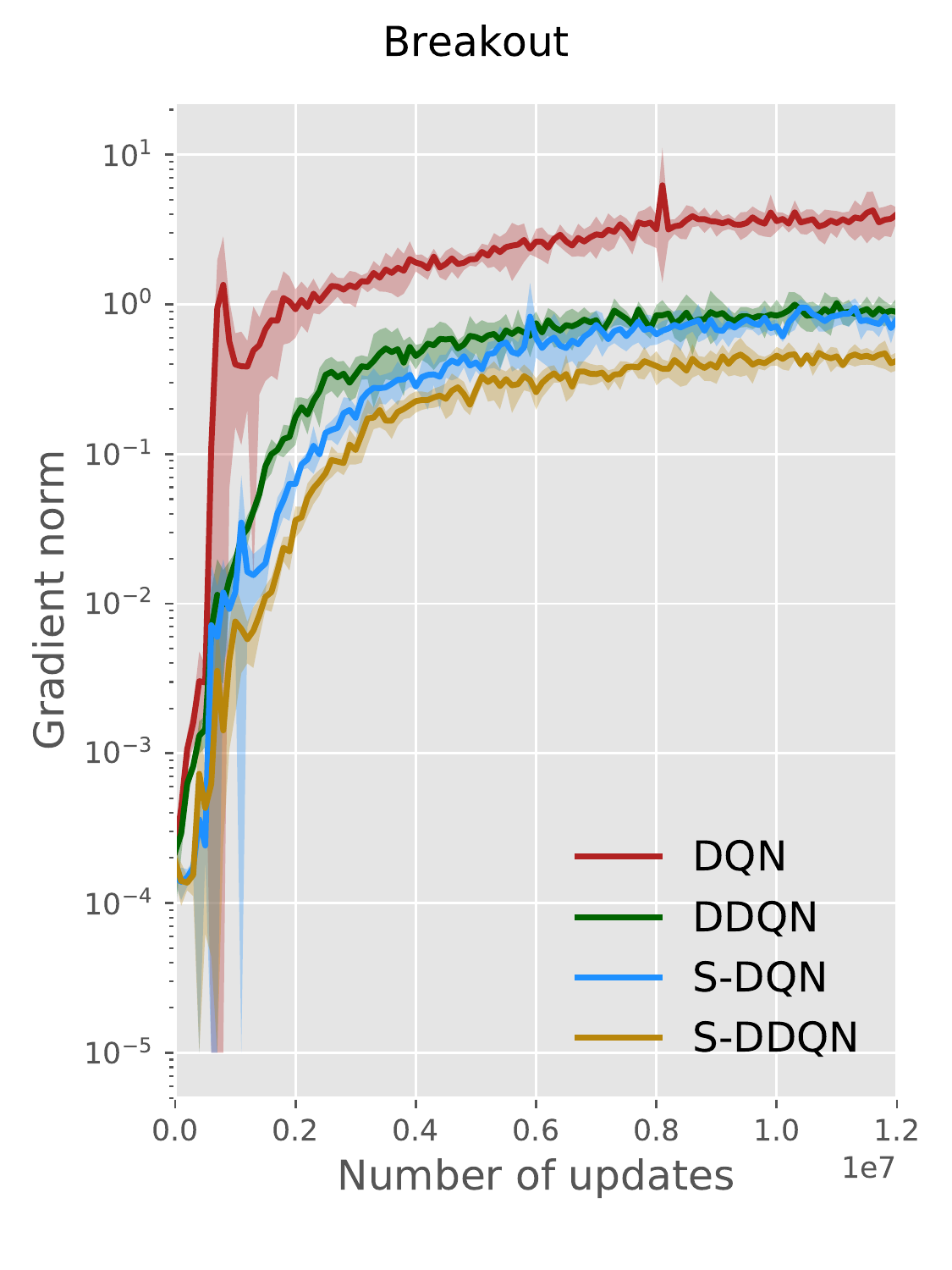}}
	\subfigure{\includegraphics[width=0.495\linewidth]{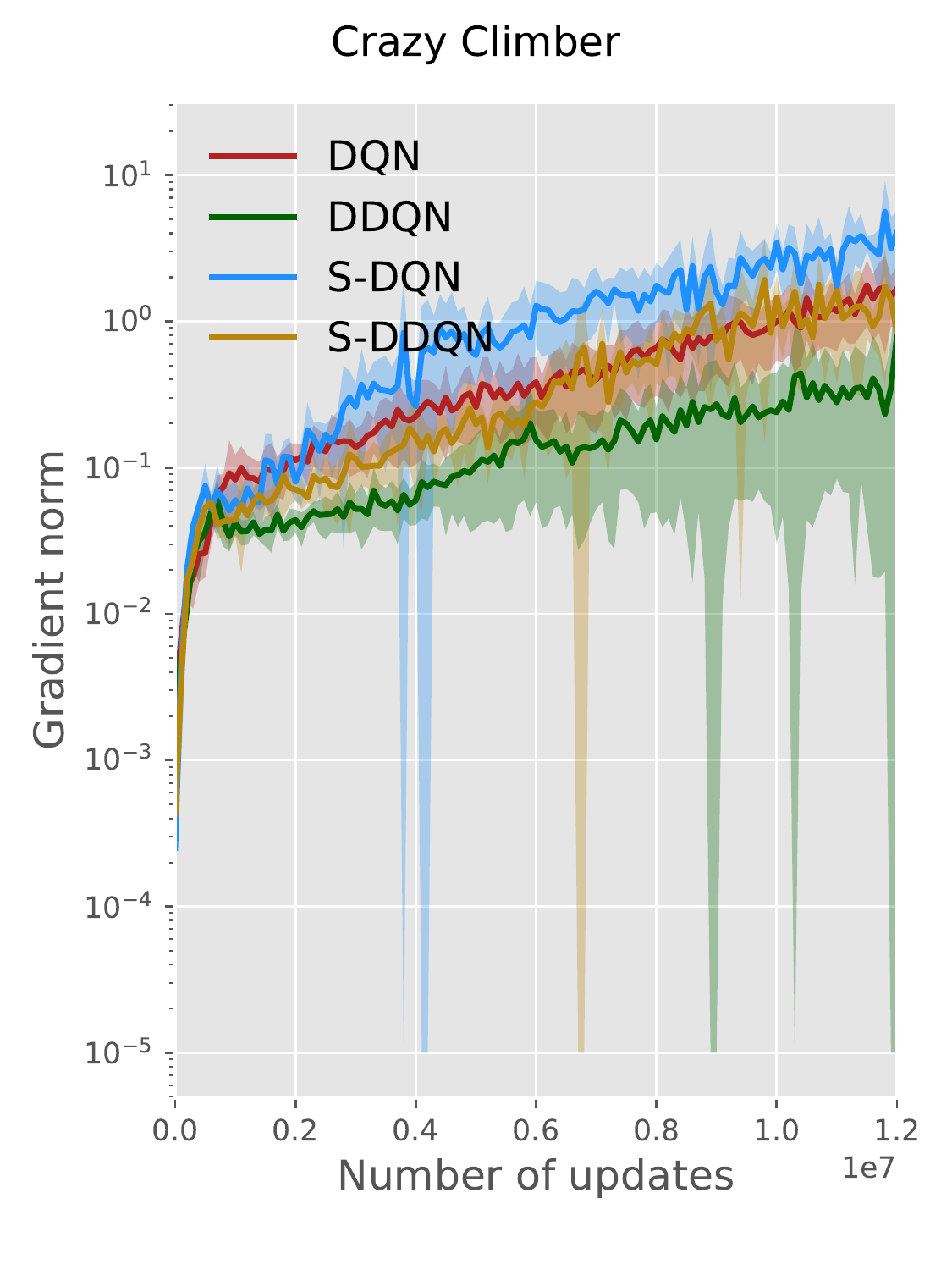}}
	\subfigure{\includegraphics[width=0.495\linewidth]{./fig/Asterix_res_grad_best.pdf}}
	\subfigure{\includegraphics[width=0.495\linewidth]{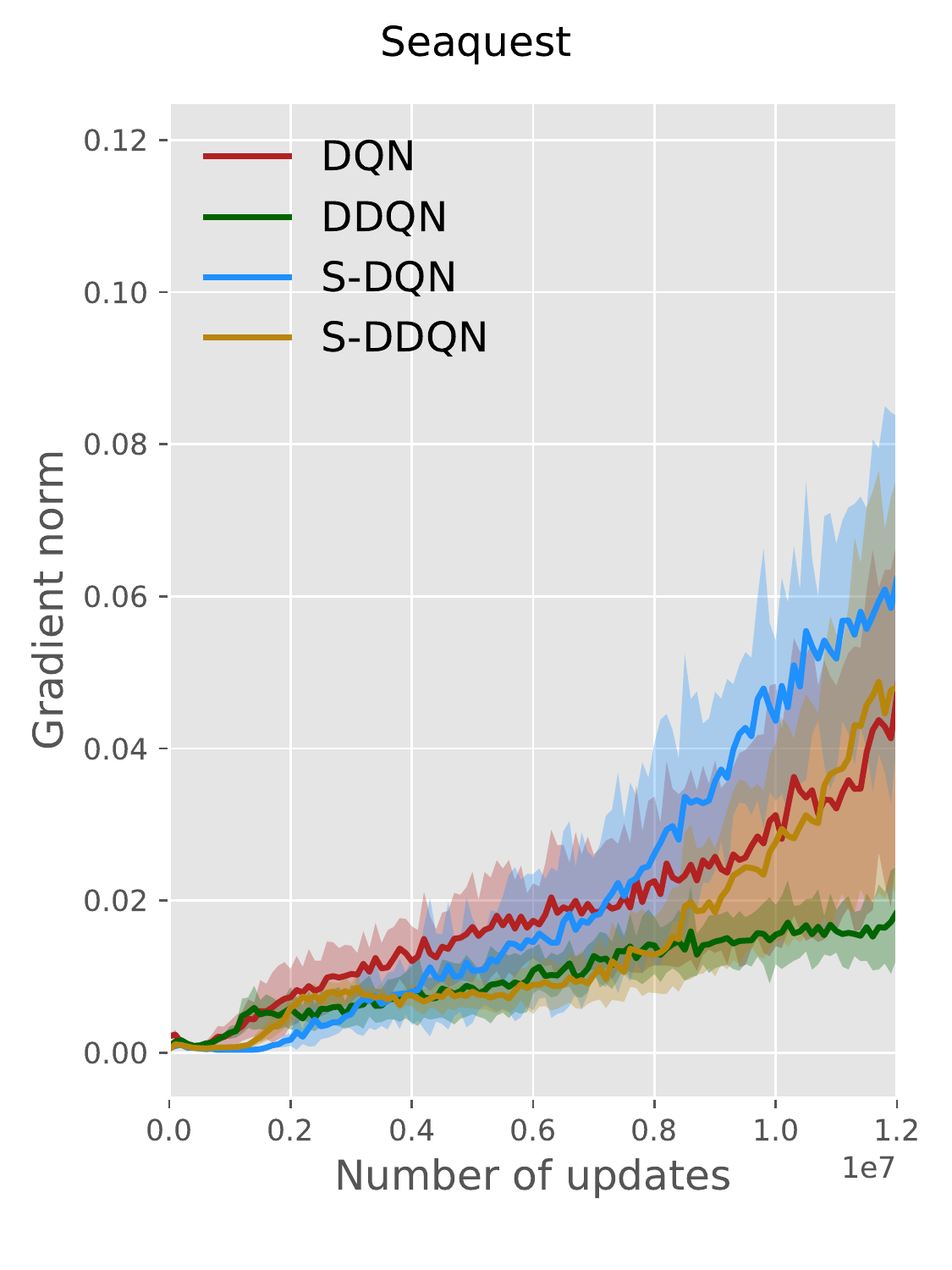}}	 
	\caption{Mean and one standard deviation of the gradient norm on the Atari games,
	for different methods.}
	\label{fig:grad_supp}
\end{figure}

\begin{figure}[!h]
	\centering
	\subfigure{\includegraphics[width=0.495\linewidth]{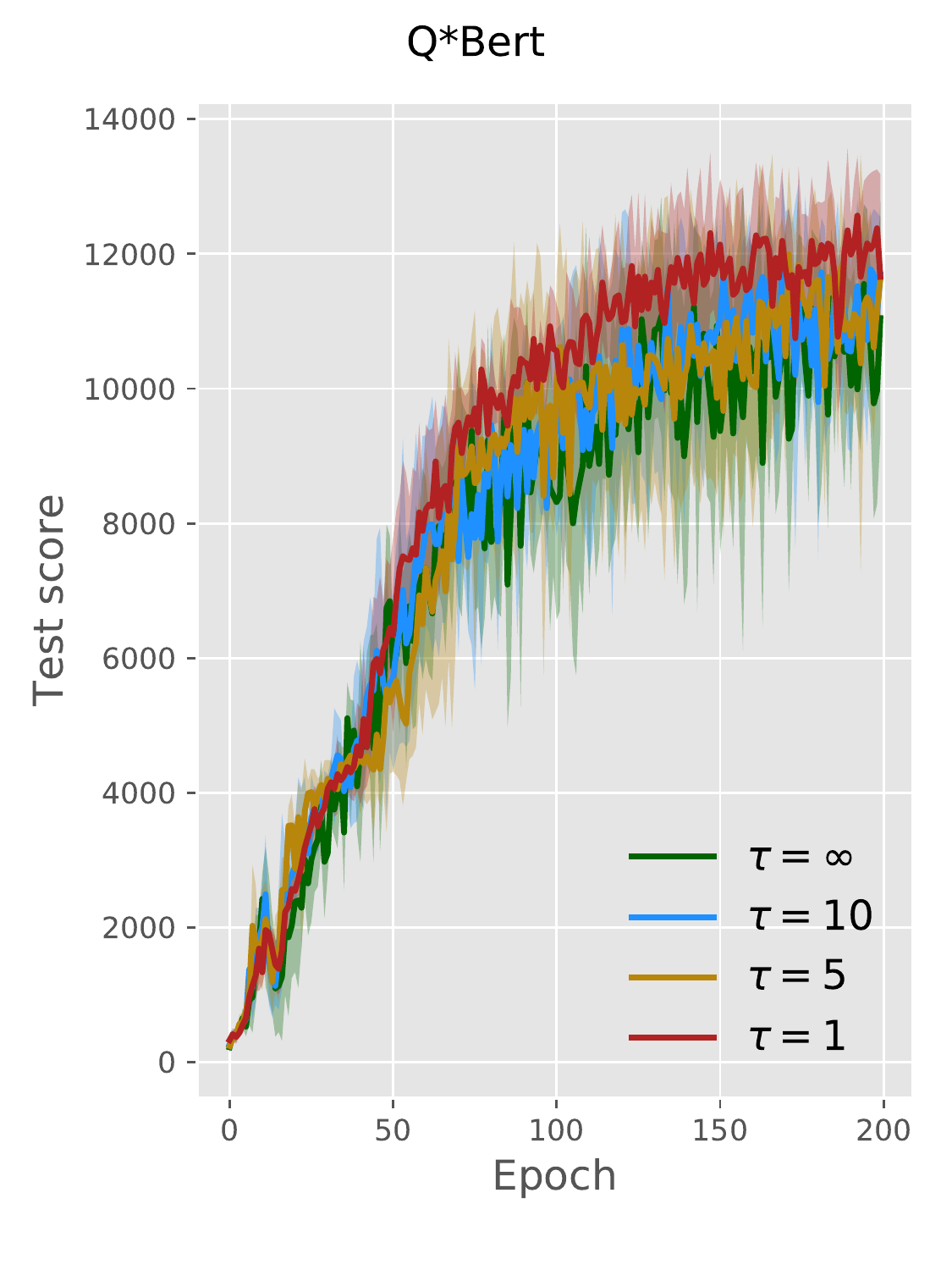}} 
	\subfigure{\includegraphics[width=0.495\linewidth]{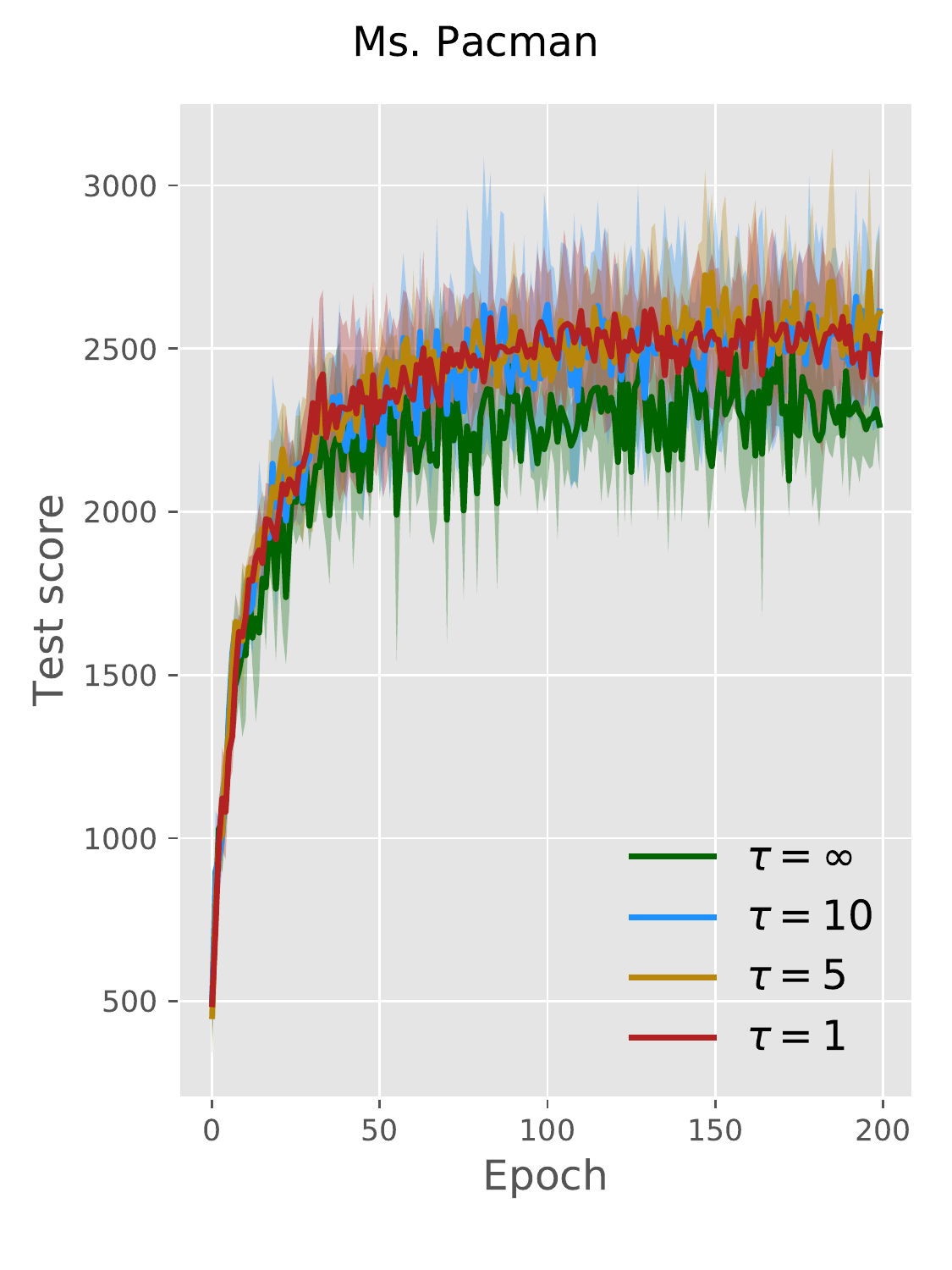}}
	\subfigure{\includegraphics[width=0.495\linewidth]{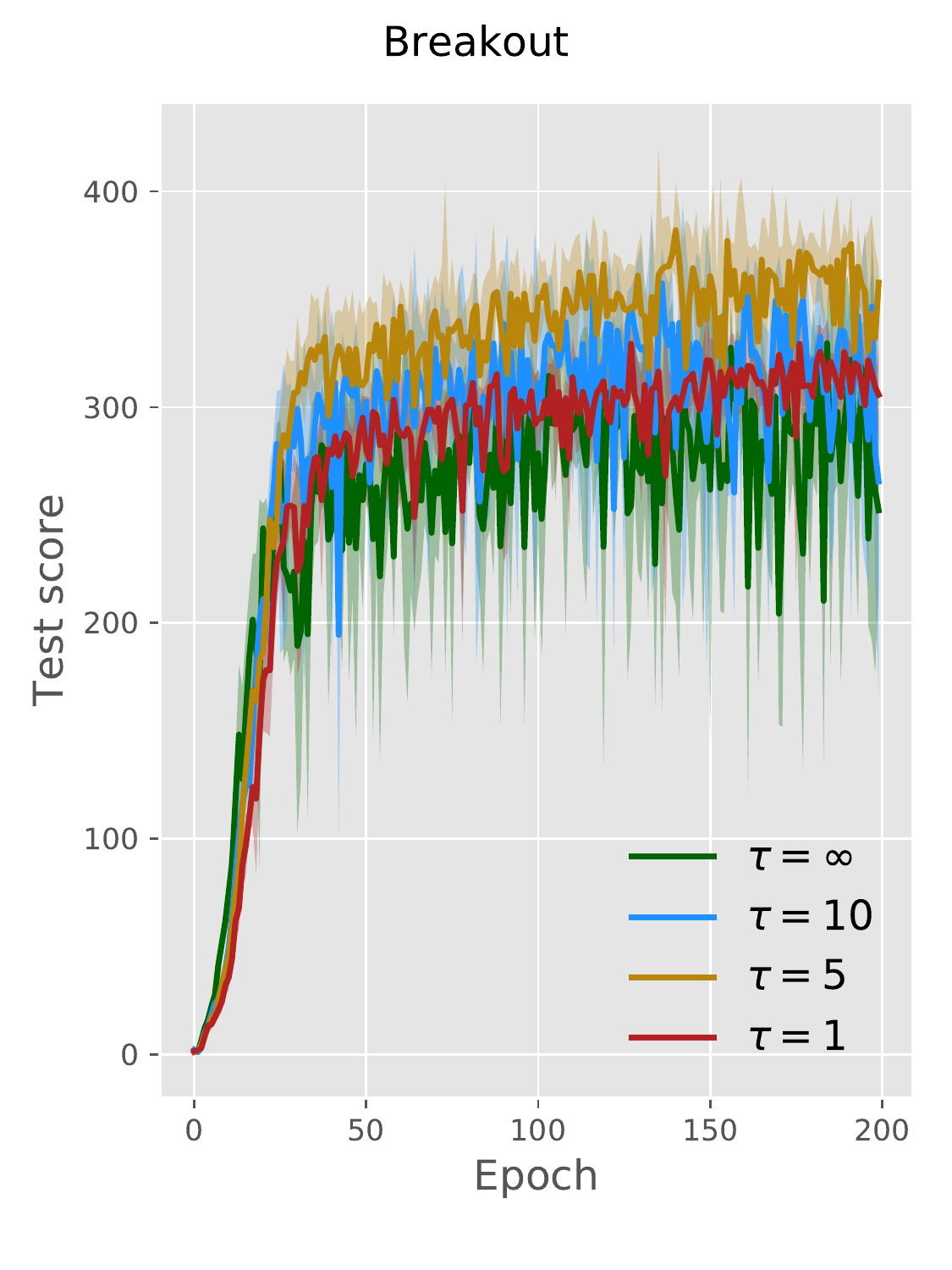}}
	\subfigure{\includegraphics[width=0.495\linewidth]{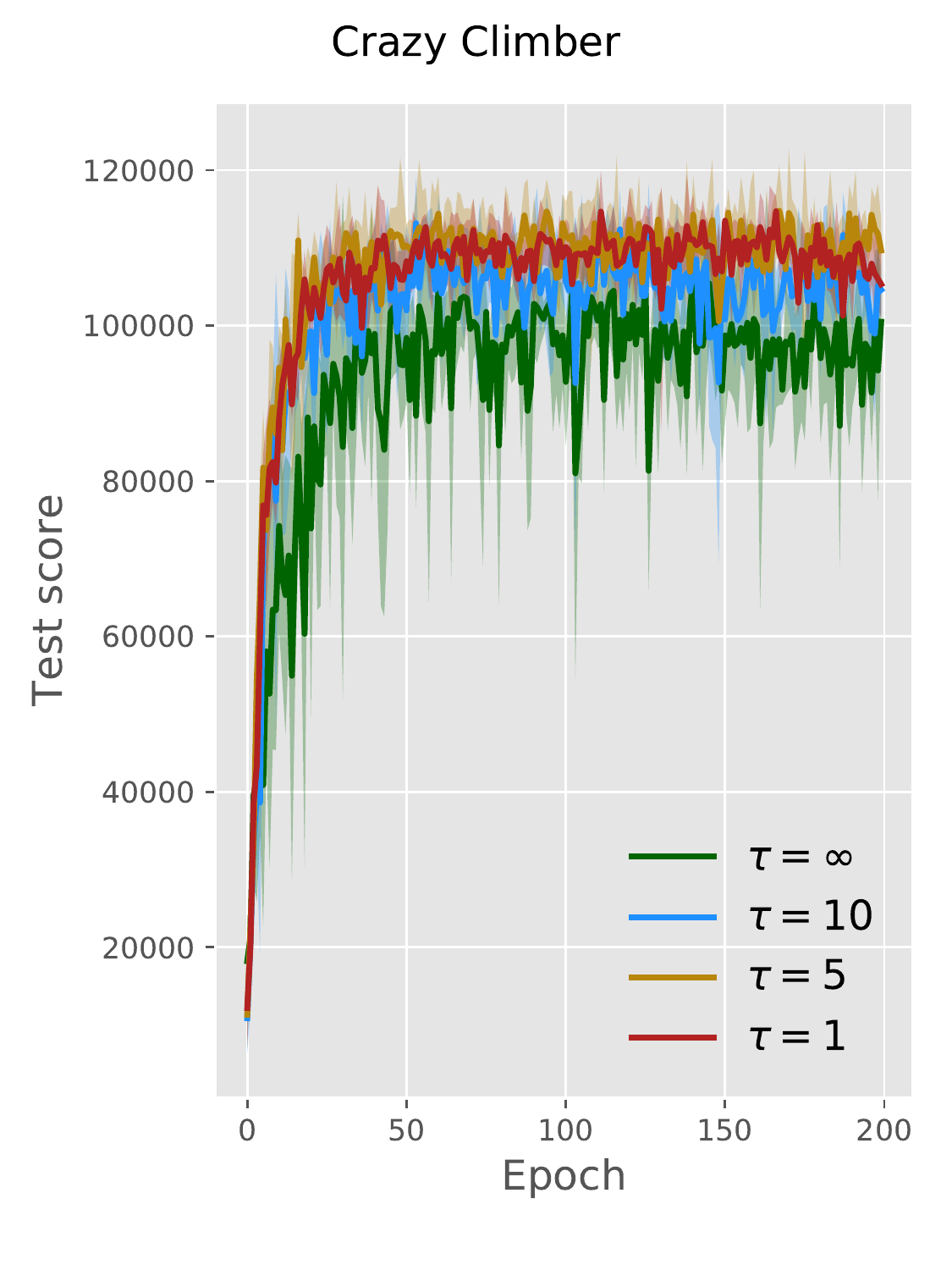}} 
	\subfigure{\includegraphics[width=0.495\linewidth]{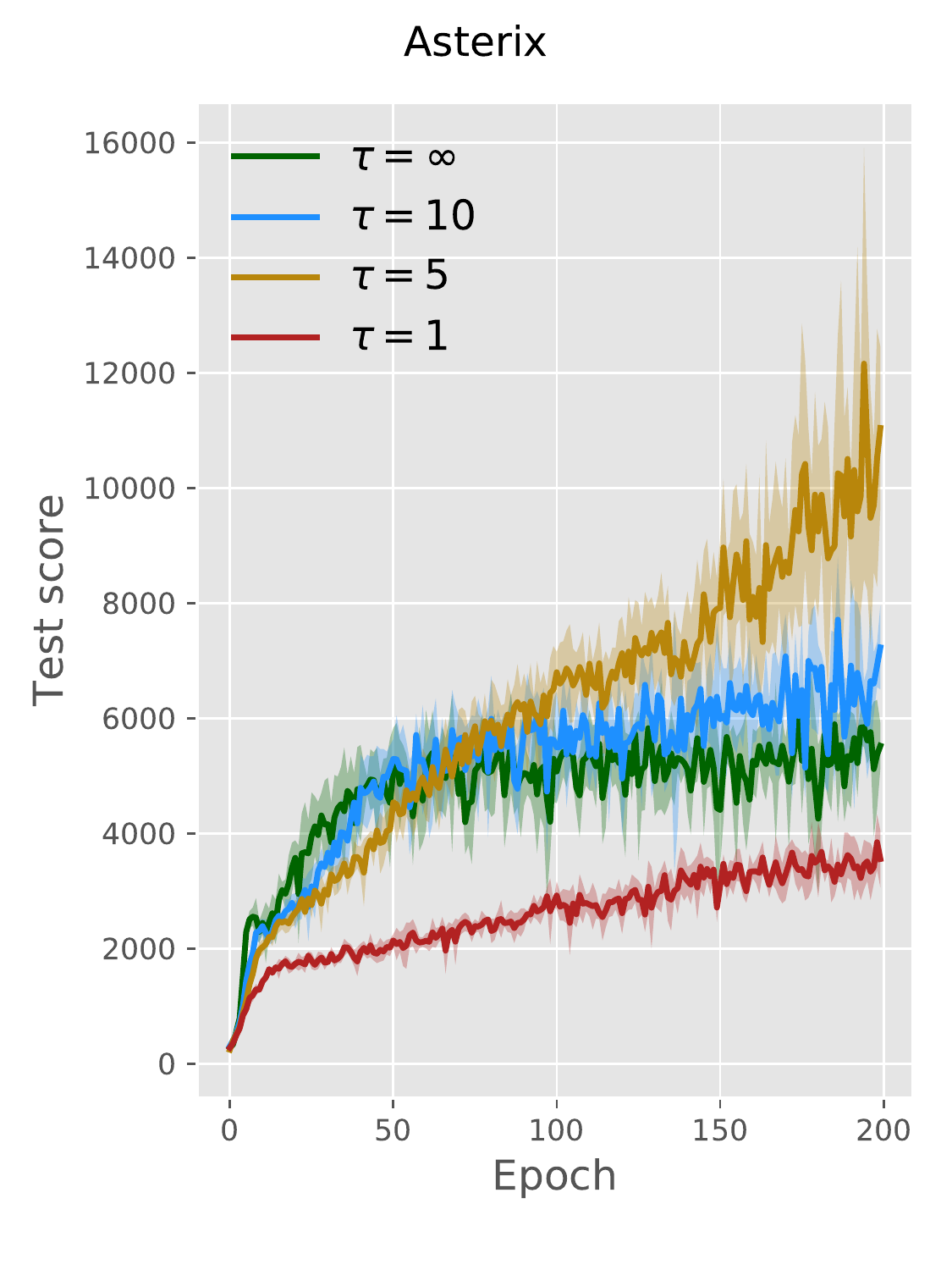}} 
	\subfigure{\includegraphics[width=0.495\linewidth]{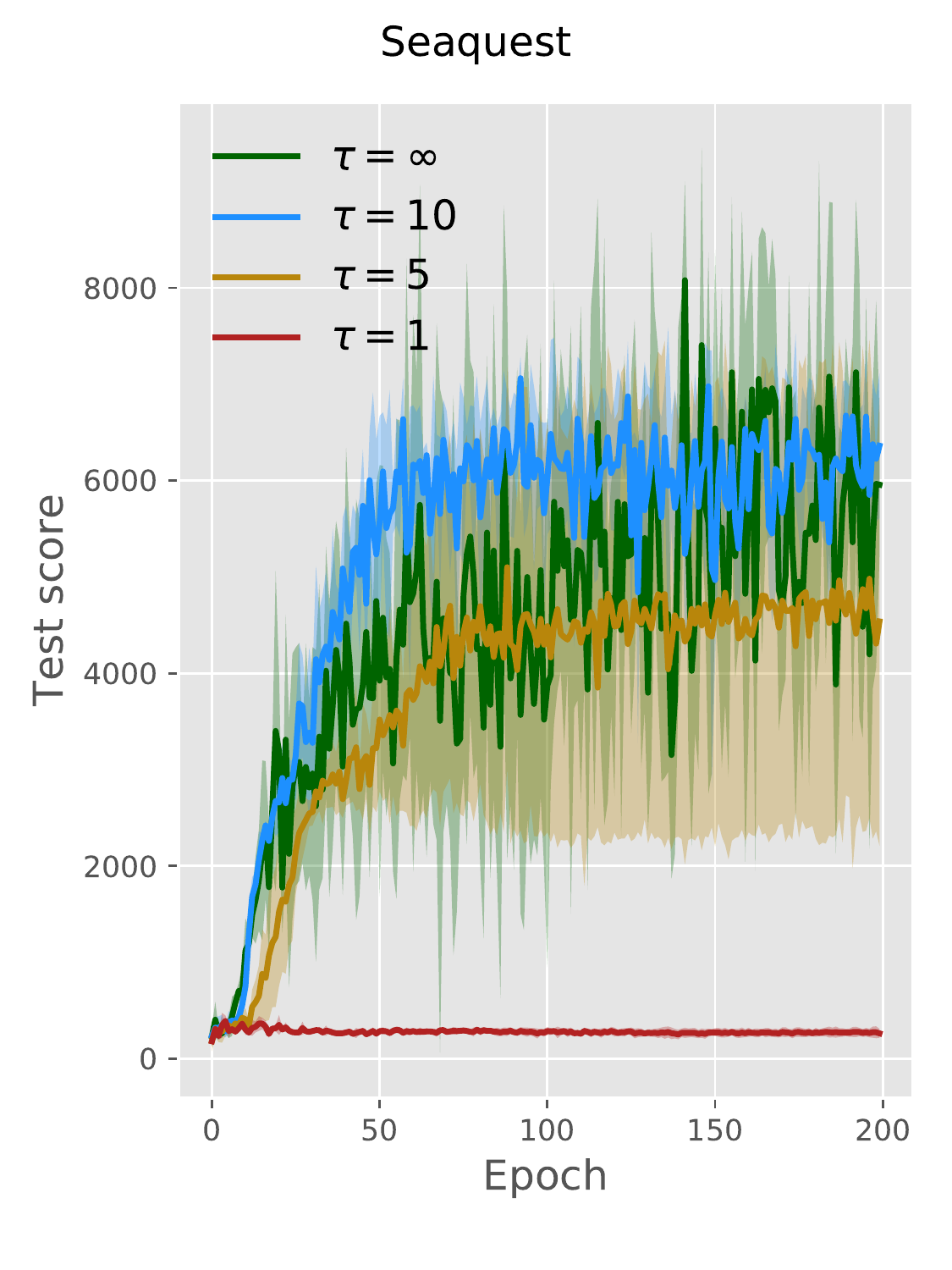}} 
	\caption{Mean and one standard deviation of test scores on the Atari games,
	for different values of $\tau$ in S-DDQN.}
	\label{fig:score_tau_supp}
\end{figure}

\begin{figure}[!t]
	\centering
	\subfigure{\includegraphics[width=0.495\linewidth]{./fig/QBert_qvalue_tau_add5.pdf}} 
	\subfigure{\includegraphics[width=0.495\linewidth]{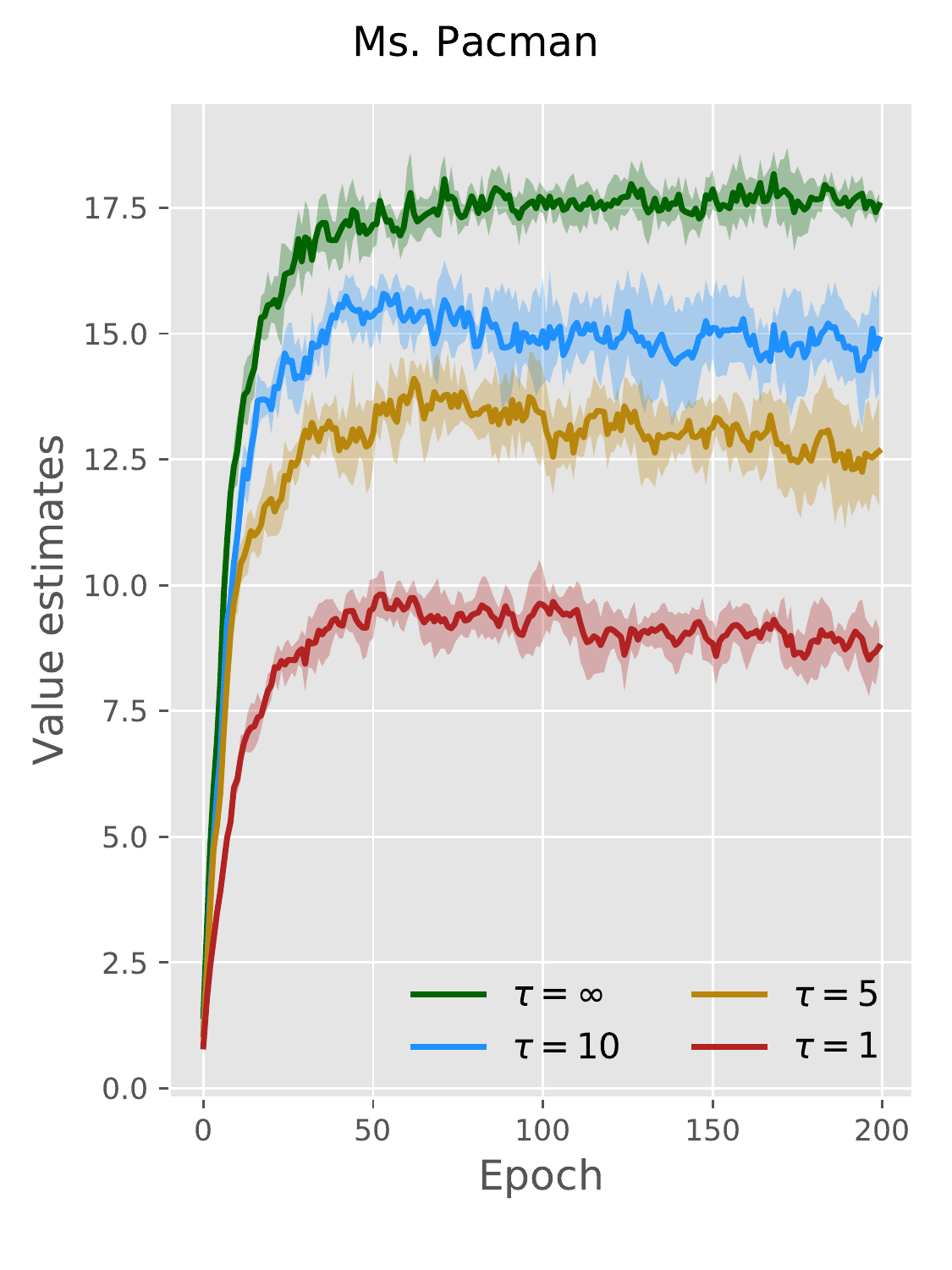}}
	\subfigure{\includegraphics[width=0.495\linewidth]{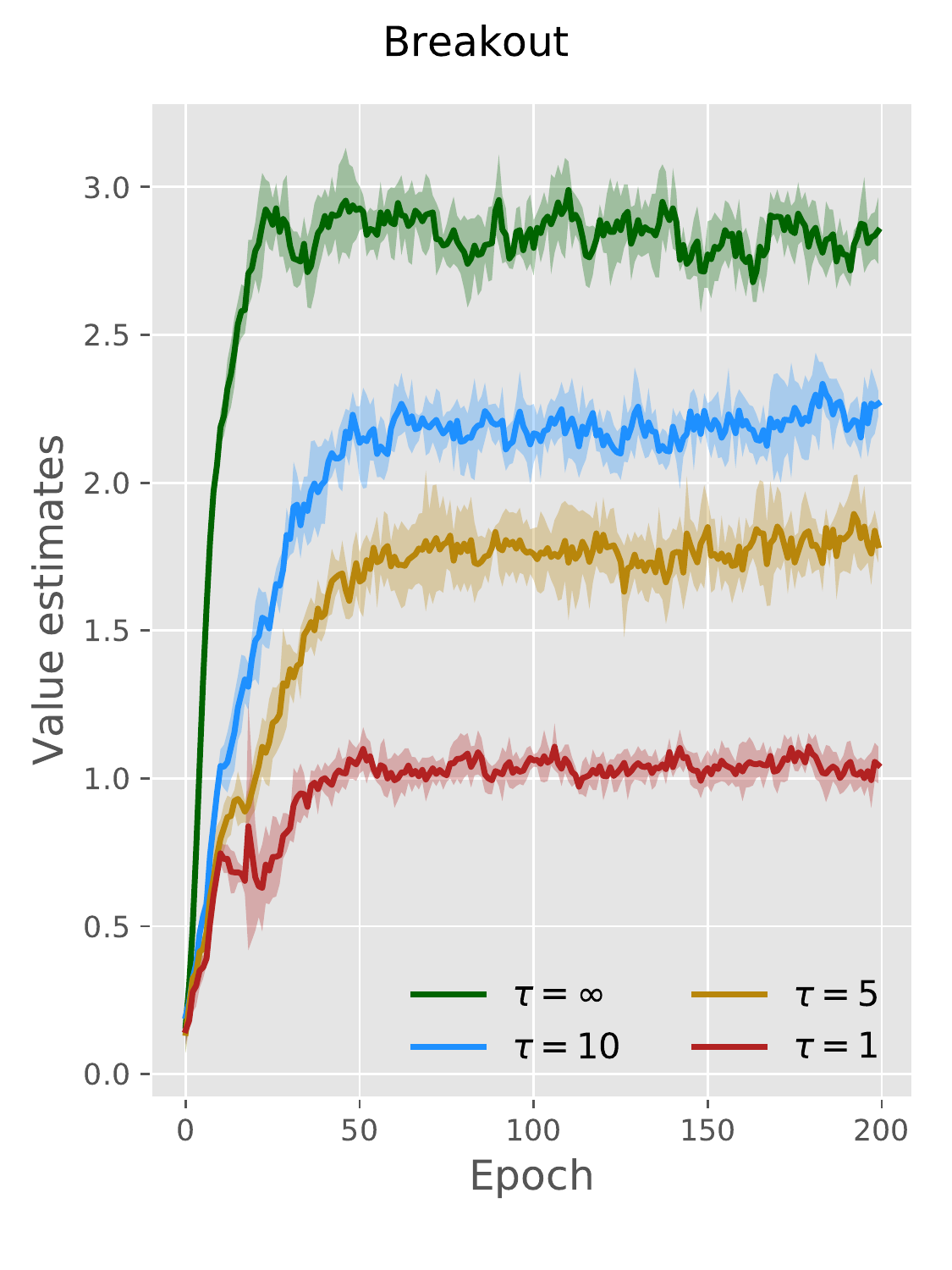}}
	\subfigure{\includegraphics[width=0.495\linewidth]{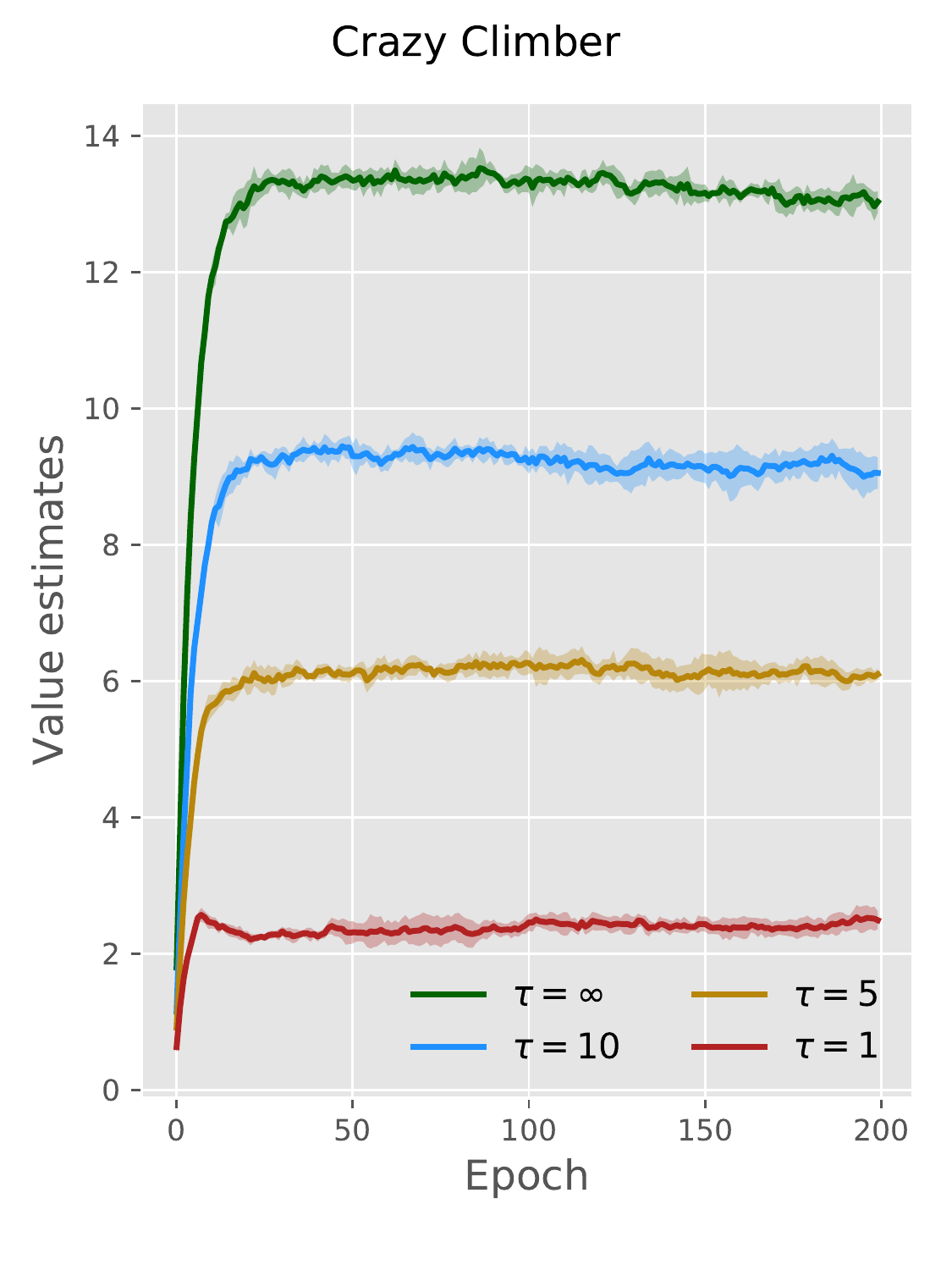}} 
	\subfigure{\includegraphics[width=0.495\linewidth]{./fig/Asterix_qvalue_tau_add5.pdf}} 
	\subfigure{\includegraphics[width=0.495\linewidth]{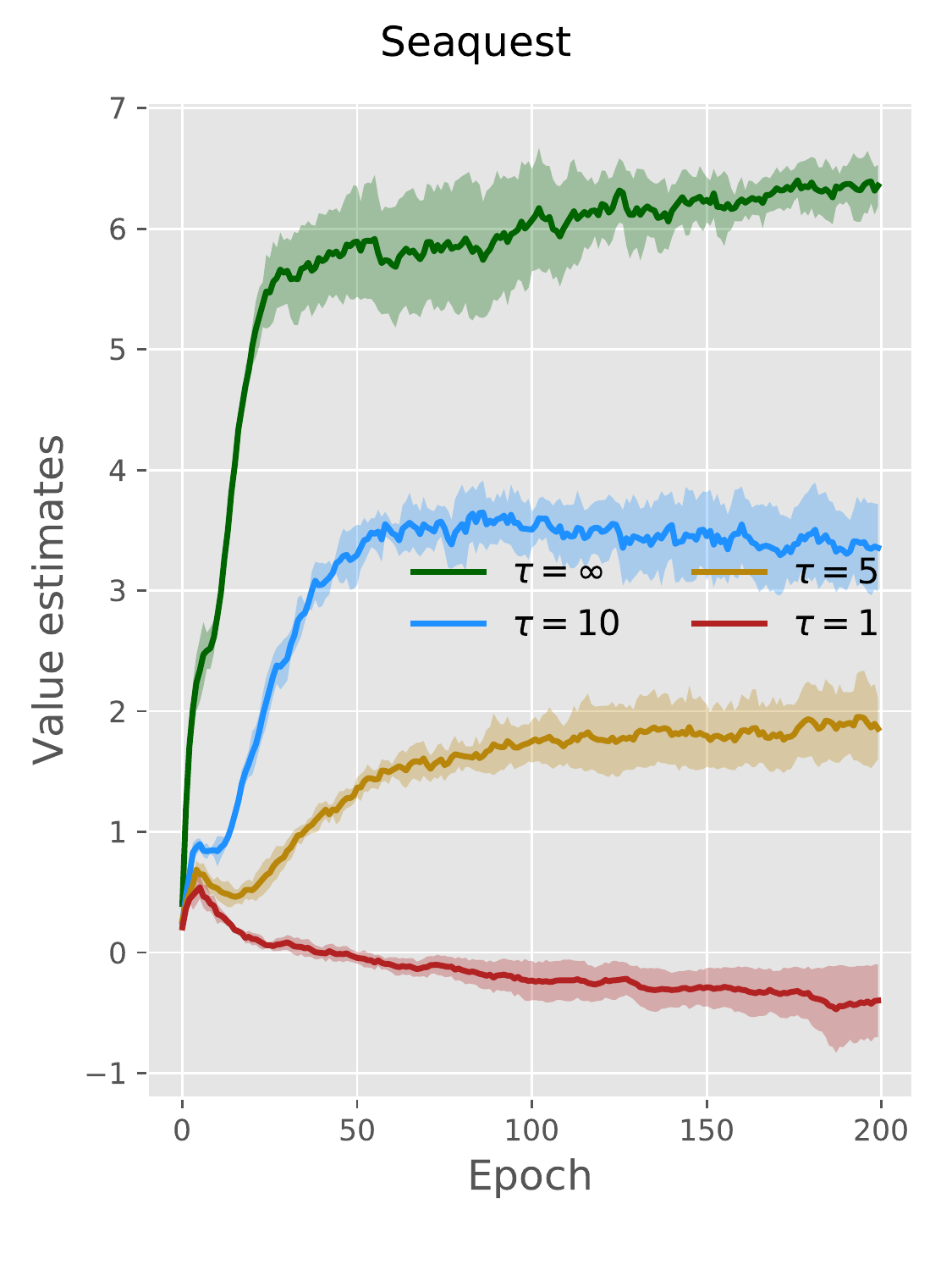}}  	
	\caption{Mean and one standard deviation of the estimated Q-values on the Atari games,
	for different values of $\tau$ in S-DDQN.}
	\label{fig:qvalue_tau_supp}
\end{figure}

\begin{figure}[!h]
	\centering
	\subfigure{\includegraphics[width=0.495\linewidth]{./fig/QBert_res_tau_grad_add5.pdf}} 
	\subfigure{\includegraphics[width=0.495\linewidth]{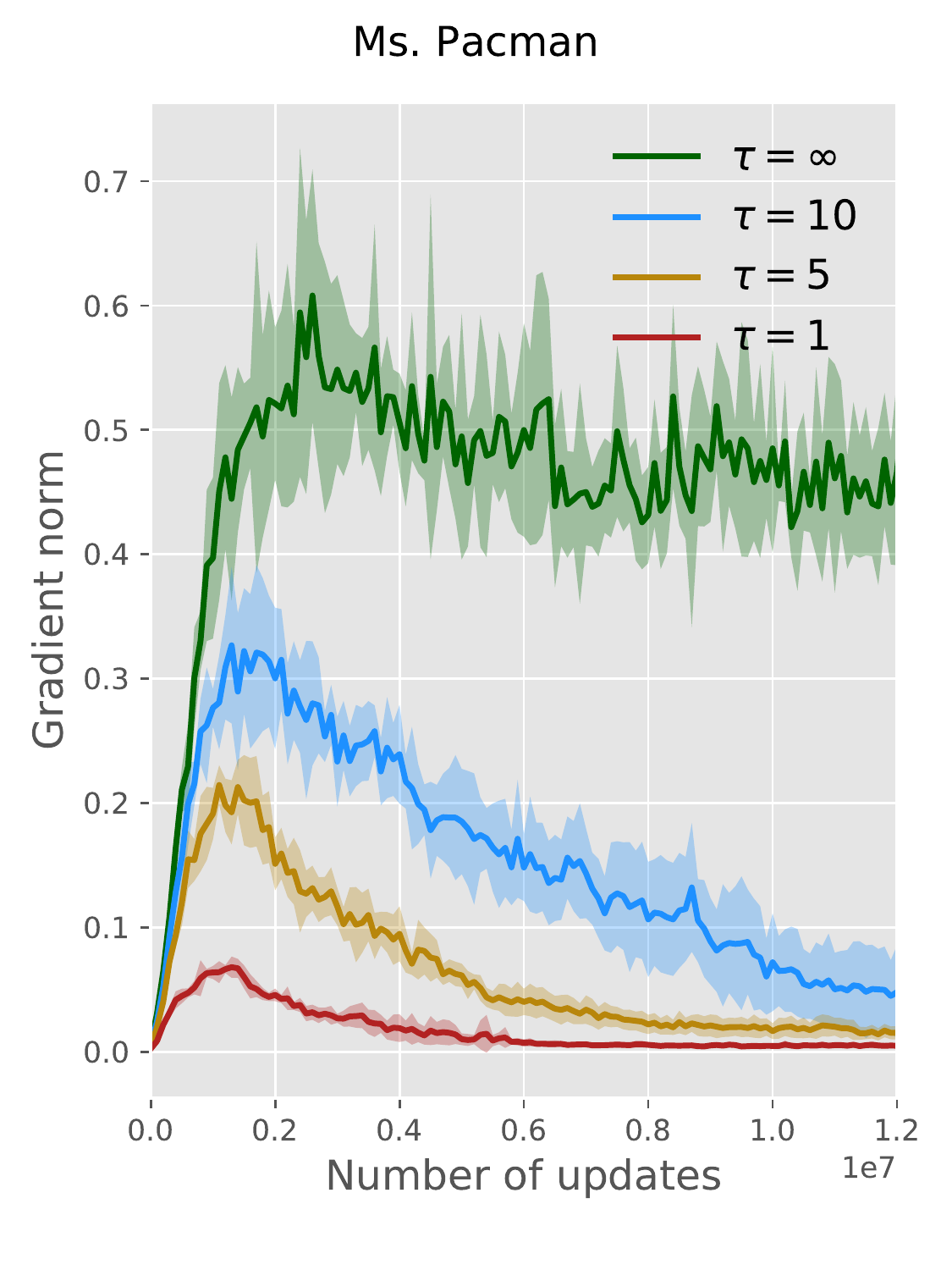}}
	\subfigure{\includegraphics[width=0.495\linewidth]{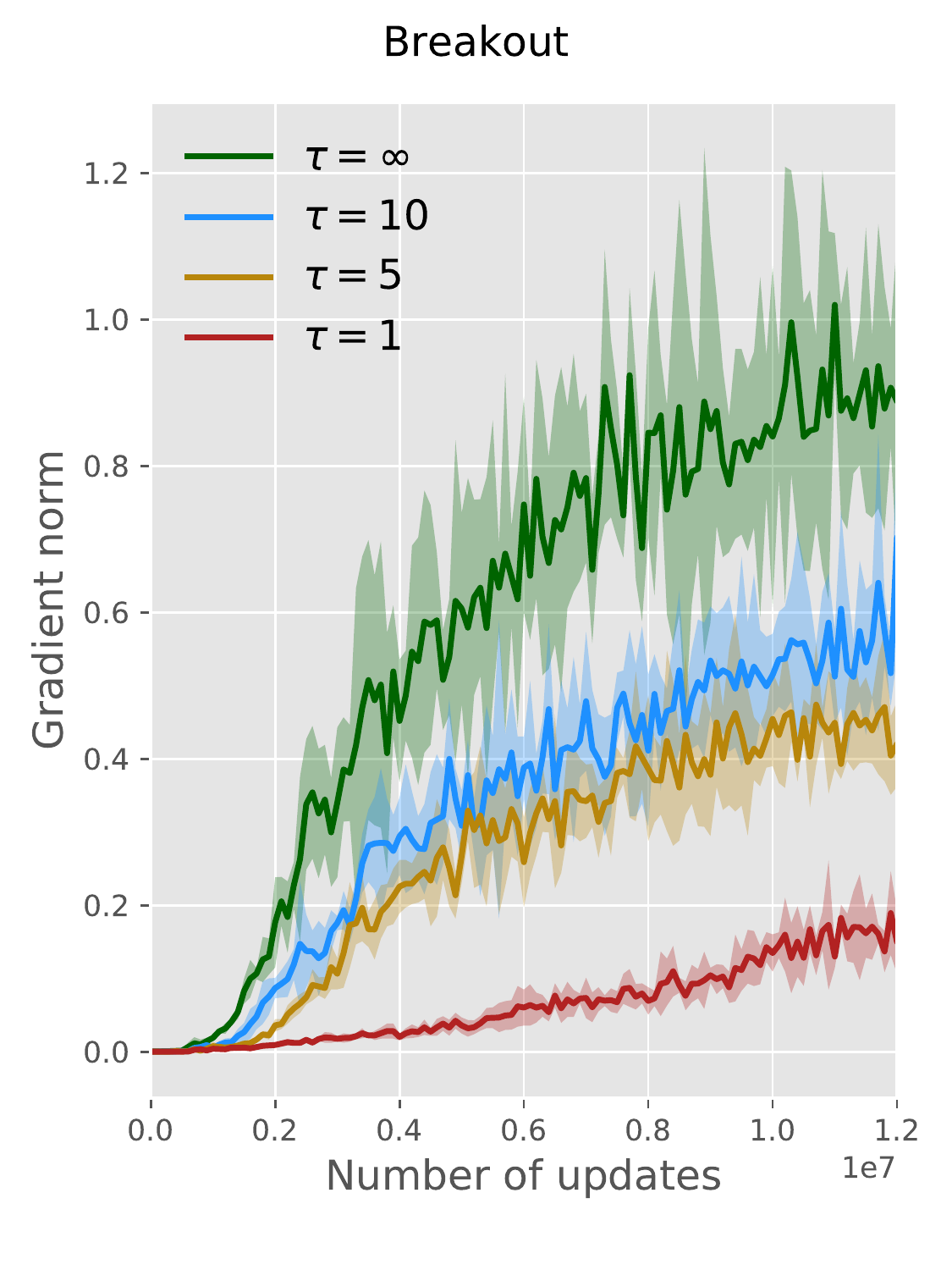}}
	\subfigure{\includegraphics[width=0.495\linewidth]{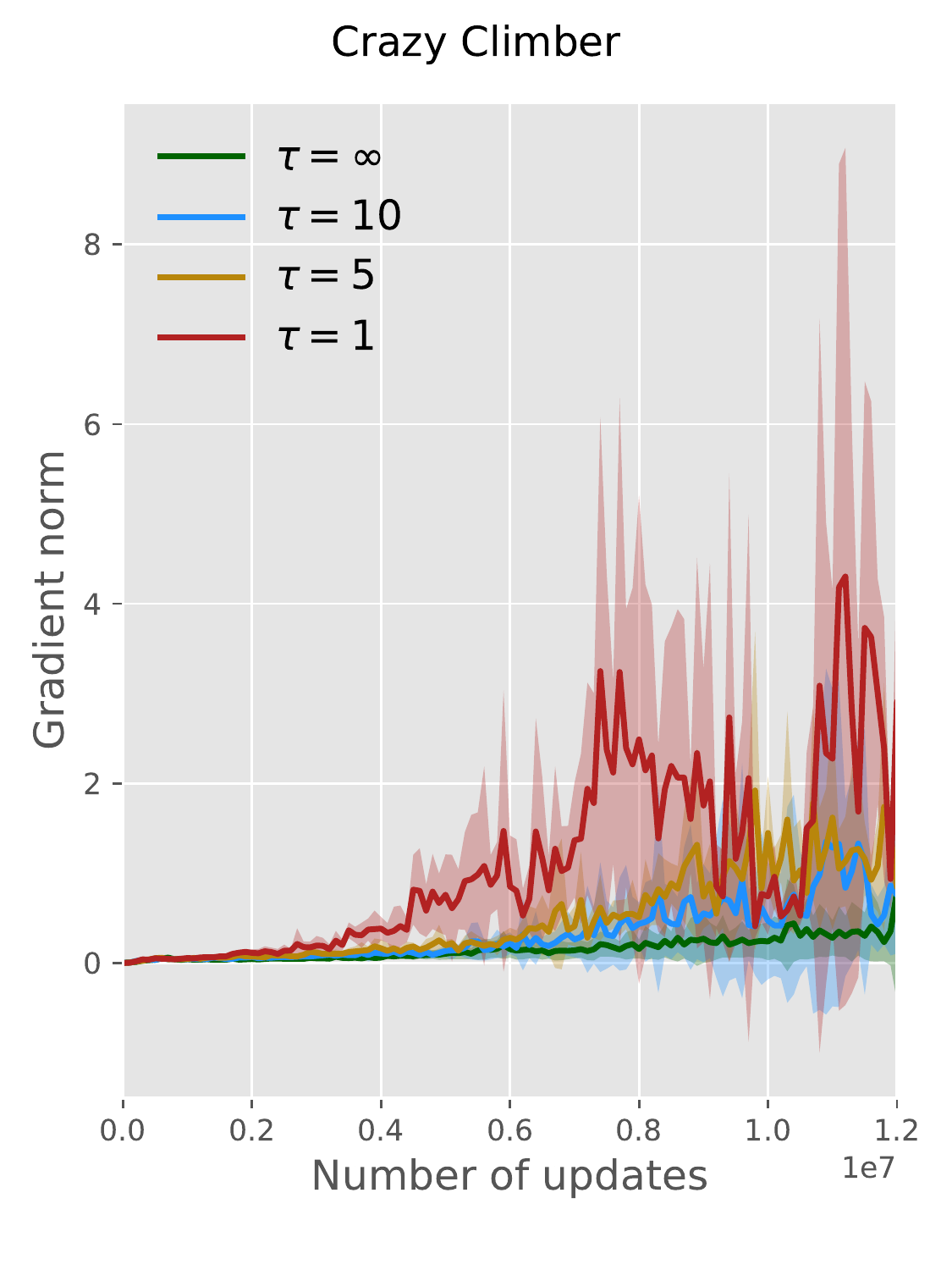}} 
	\subfigure{\includegraphics[width=0.495\linewidth]{./fig/Asterix_res_tau_grad_add5.pdf}} 
	\subfigure{\includegraphics[width=0.495\linewidth]{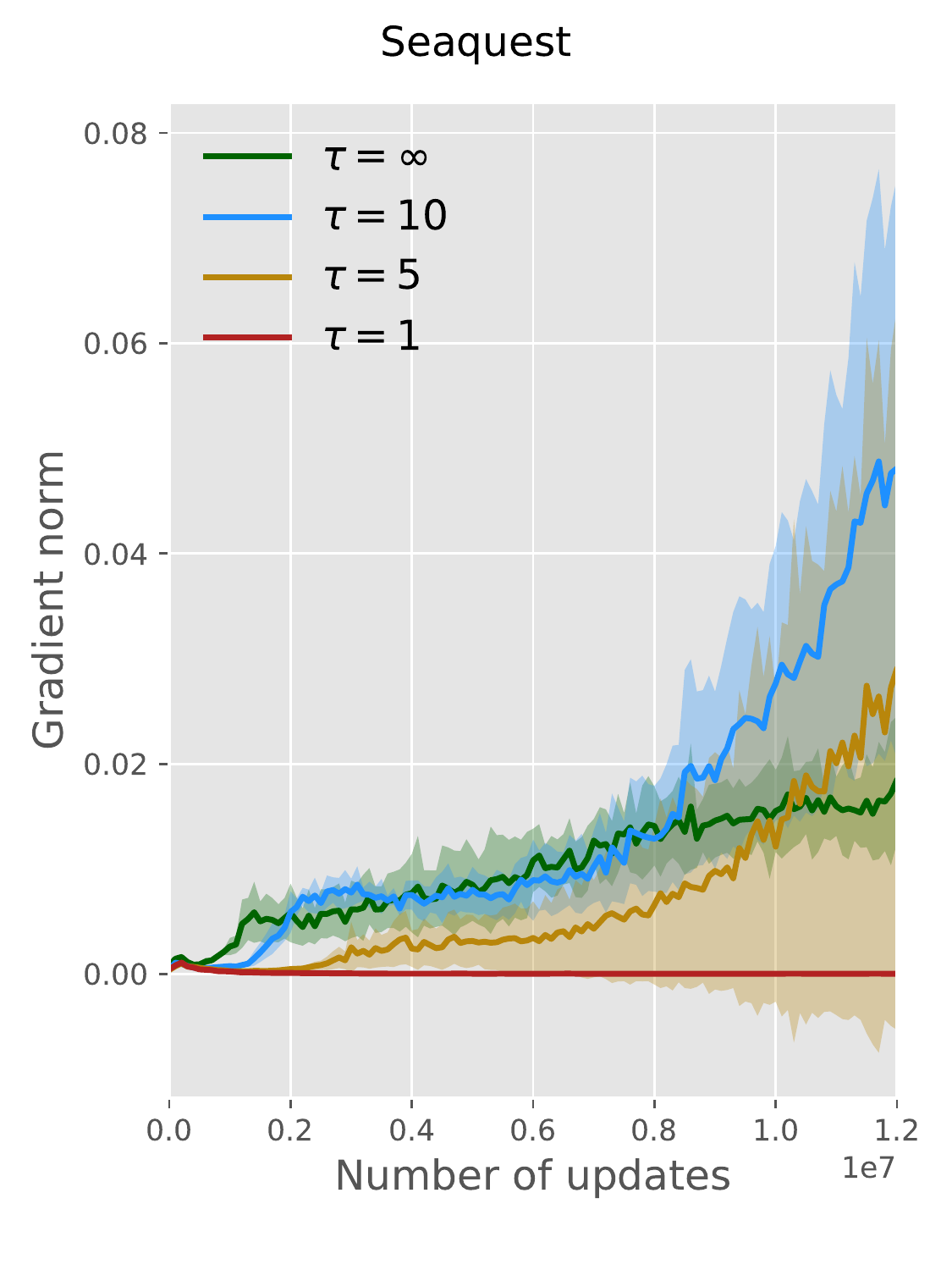}} 
	\caption{Mean and one standard deviation of the gradient norm on the Atari games,
	for different values of $\tau$ in S-DDQN.}
	\label{fig:grad_tau_supp}
\end{figure}

\end{document}